\documentclass{article} 
\usepackage{collas2023_conference,times}
\usepackage{easyReview}

\usepackage{amsmath}
\usepackage{graphicx}
\usepackage{dsfont}
\usepackage{subfigure}
\usepackage{stmaryrd}
\usepackage{float}
\usepackage{wrapfig}
\usepackage{gensymb}
\usepackage{graphicx}
\usepackage{comment}
\usepackage{amsmath,amssymb} 
\usepackage{color}
\usepackage{epsfig}
\usepackage{array}
\usepackage{multirow}
\usepackage{amsfonts}
\usepackage{caption}
\usepackage{wrapfig}
\usepackage{verbatim}
\usepackage{float}
\usepackage{enumitem}
\usepackage{booktabs}
\usepackage{mathtools}
\usepackage{lipsum}
\usepackage{subfigure}
\usepackage{wrapfig,tabulary,multirow}
\usepackage{mathtools}
\usepackage{amsmath, amsthm,amssymb, bbm, bm}
\usepackage{graphicx}
\usepackage{dsfont}
\usepackage{subfigure}
\usepackage{stmaryrd}
\usepackage{float}
\usepackage{mathtools}
\usepackage{graphicx}
\usepackage{dsfont}
\usepackage{subfigure}
\usepackage{stmaryrd}
\usepackage{float}
\usepackage{wrapfig}
\usepackage{gensymb}
\usepackage{diagbox}
\usepackage{graphicx}
\usepackage{comment}
\usepackage{amsmath,amssymb} 
\usepackage{color}
\usepackage{epsfig}
\usepackage{array}
\usepackage{multirow}
\usepackage{amsfonts}
\usepackage{caption}
\usepackage{wrapfig}
\usepackage{float}
\usepackage{enumitem}
\usepackage{booktabs}
\usepackage{tabulary,overpic,xcolor}
\usepackage{thmtools,thm-restate,amsthm}
\usepackage{algorithm}
\usepackage{algorithmic}

\usepackage{multirow}
\usepackage{amsfonts}
\usepackage{caption}
\usepackage{wrapfig}
\usepackage{float}
\newcommand{\ie}{i.e.}
\newcommand{\eg}{e.g.}
\newcommand{\mdp}{\mathcal{M}}

\newcommand{\Exp}{\mathop{\mathbb E}\displaylimits}
\newtheorem{assumption}{Assumption}
\newtheorem{theorem}{Theorem}
\newtheorem{lemma}{Lemma}
\newtheorem{definition}{Definition}%

\usepackage{amsmath,amsfonts,bm}

\def\eqref#1{Equation~\ref{#1}}

\def\1{\bm{1}}

\DeclareMathAlphabet{\mathsfit}{\encodingdefault}{\sfdefault}{m}{sl}
\SetMathAlphabet{\mathsfit}{bold}{\encodingdefault}{\sfdefault}{bx}{n}

\DeclareMathOperator*{\argmax}{arg\,max}
\DeclareMathOperator*{\argmin}{arg\,min}

\usepackage{hyperref}
\hypersetup{
    colorlinks=true,
    linkcolor=red,
    filecolor=magenta,
    urlcolor=blue,
    citecolor=purple,
    pdftitle={Overleaf Example},
    pdfpagemode=FullScreen,
    }

\title{Learning Meta Representations for Agents in Multi-Agent Reinforcement Learning}

\author{
  \qquad \qquad\,\,\, Shenao Zhang \qquad\qquad \qquad\qquad\qquad\qquad \qquad \qquad\quad Li Shen \\
  \qquad \qquad\quad Georgia Institute of Technology \qquad \qquad\qquad\qquad\qquad \, Tencent AI Lab \\
  \qquad \qquad\quad \texttt{shenao@gatech.edu} \qquad \qquad \qquad \quad\qquad\qquad\qquad \texttt{lshen.lsh@gmail.com} \\
  \AND
  \qquad \qquad\quad Lei Han \qquad \qquad \qquad \quad\qquad \qquad\qquad\qquad\qquad \qquad\,\,\, Li Shen \\
  \qquad \qquad\quad Tencent Robotics X \qquad \quad\qquad\qquad \qquad \qquad \qquad\qquad\, JD Explore Academy \\
  \qquad \qquad\quad \texttt{leihan.cs@gmail.com} \qquad \qquad \quad \qquad\qquad \qquad \,\,\,\,\,\texttt{mathshenli@gmail.com} \\
}

\preprintcopy 

\begin{document}

\maketitle

\begin{abstract}
In multi-agent reinforcement learning, the behaviors that agents learn in a single Markov Game (MG) are typically confined to the given agent number. Every single MG induced by varying the population may possess distinct optimal joint strategies and game-specific knowledge, which are modeled independently in modern multi-agent reinforcement learning algorithms. In this work, our focus is on creating agents that can generalize across population-varying MGs. Instead of learning a unimodal policy, each agent learns a policy set comprising effective strategies across a variety of games. To achieve this, we propose \textit{Meta Representations for Agents} (MRA) that explicitly models the game-common and game-specific strategic knowledge. By representing the policy sets with multi-modal latent policies, the game-common strategic knowledge and diverse strategic modes are discovered through an iterative optimization procedure. We prove that by approximately maximizing the resulting constrained mutual information objective, the policies can reach Nash Equilibrium in every evaluation MG when the latent space is sufficiently large. When deploying MRA in practical settings with limited latent space sizes, fast adaptation can be achieved by leveraging the first-order gradient information. Extensive experiments demonstrate the effectiveness of MRA in improving training performance and generalization ability in challenging evaluation games.
\end{abstract}

\section{Introduction}
\label{intro}
Behaviors of agents learned in a single Markov Game (MG) highly depend on the environmental settings, especially the number of agents, \ie, population \citep{suarez2019neural,long2020Evolutionary}. Many multi-agent reinforcement learning (MARL) algorithms \citep{sukhbaatar2016learning,foerster2016learning,lowe2017multi} are developed in games with a fixed population. However, the algorithms may suffer from generalization issues, \ie, the policies learned in a single MG are brittle to the change of the agent number \citep{suarez2019neural}. Recent works have experimentally shown the benefit of knowledge transfer between MGs with different populations \citep{agarwal2019learning,long2020Evolutionary}, which is required to perform between successive games. Unfortunately, the resulting agents are still confined to particular training games, with less ability for extrapolation.


In this work, we are concerned with learning multi-agent policies that generalize across Markov Games constructed by varying the population from the same underlying environment. The created agents are expected to behave well in both training MGs and novel (or unseen) evaluation MGs. However, for each agent, optimizing one unimodal policy even to maximize the performance on the entire \textit{training} MG set is still challenging \citep{teh2017distral}. Effective policies in population-varying games, such as ones that achieve Nash Equilibrium in each game, may behave dramatically different due to the game-specific strategic knowledge of themselves. Such discrepancy will hamper the performance in individual games \citep{brunskill2013sample}. In this regard, it is desirable to learn \textit{sets of policies} that are formed by the optimal strategies for each training MG, while transferring knowledge to \textit{unseen} MGs is still challenging nevertheless.


To cope with this generalization challenge, our solution involves modeling the \textit{game-specific} and \textit{game-common} strategic knowledge. In unseen games, although the optimal game-specific knowledge that leads to optimal policies is unobtainable, the game-common knowledge and various strategic modes can be captured during training by imposing \textit{knowledge variations}, \ie, the \textit{suboptimal} game-specific knowledge. Instead of only fitting the best response, learning to make decisions under multiple knowledge variations plays the role of augmenting the training games. As a result, the strategic knowledge is learned in an unsupervised manner and agents can effectively generalize to novel MGs.

For games induced by varying the population, the distinct optimal policies are determined by different strategic relationships between agents (see \eqref{rel_v} for a formal definition), which we characterize as game-specific knowledge. For example, in a Pac-Man game that is populationally dominated by ghost agents, the game-specific knowledge for Pac-Man is to focus on the ghost agents' positions for survival. However, as all Pac-Man agents ignore other Pac-Man agents' positions, they are unaware of collaborating and their ability to collect food dots in evaluation games is poor. One potential fix is to learn the representations which can serve as the game-common knowledge and generate a set of policies that output the best-effort actions when conditioned on different (suboptimal) strategic relationships (i.e., knowledge variations) between agents, e.g., by forcing the Pac-Man to pay more attention to other Pac-Man in spite of the game being ghost-dominated. Although these policies can be suboptimal in the training MG, they have the potential to perform well in evaluation games in a zero-shot manner. This motivates us to learn \textit{diverse} strategies so that even if not all knowledge variations are covered during training, the game-common knowledge can be learned. Better generalization is thus achieved since we only need to fit the best strategic relationship in the evaluation game.

To formalize this intuition, we propose \textit{Meta Representations for Agents} (MRA) to discover the underlying strategic structures. Specifically, by meta-representing the policy sets with multi-modal latent policies, the game-common knowledge and diverse strategic modes are captured through iterative diversity-driven optimization. We prove that by approximately maximizing a constrained mutual information objective, the latent policies can reach Nash Equilibrium in every evaluation game if with a sufficiently large latent space. When deployed in limited-size latent spaces, fast adaptation is achieved by leveraging the first-order gradient information. We further empirically validate the benefits of MRA, which is capable of boosting training performance and extrapolating over a variety of evaluation games.

\section{Related Work}
\label{related}
Multi-Agent Reinforcement Learning (MARL) extends RL to multi-agent systems. In this work, we follow the centralized training with decentralized execution (CTDE) setting. Recent MARL algorithms with CTDE setting \citep{lowe2017multi,jiang2018learning,iqbal2018actor} learn in MGs with fixed numbers of agents. However, the resulting agents are shown to be unable to play well in some situations. For instance, it is shown in \citep{suarez2019neural} that the random exploration bottleneck may result in both inferior and brittle policies, \eg, competitive agents trained in a small population lack adequate exploration compared with agents trained in a large population. On the other hand, the complexity of games grows exponentially with the population, which causes direct learning intractable \citep{yang2018mean}. For both the above difficulties, training in multiple MDPs or MGs is shown to be helpful \citep{teh2017distral,wang2020few,long2020Evolutionary}, which highlights the importance of learning transferable knowledge. 
In single-agent RL, knowledge transfer is proven useful to improve the performance in related MDPs \citep{taylor2009transfer}, \eg, knowledge from game-specific experts distilled to a policy \citep{rusu2015policy,parisotto2015actor,teh2017distral}. Recently, generalizable policy sets are learned in SMERL \citep{kumar2020one}, which share similarities with our work. However, SMERL is proposed to improve the single-agent policy robustness, with the motivation that \textit{remembering} diverse (suboptimal) policies in a \textit{single} MDP can directly lead to robust behaviors, with no need to perform explicit perturbations. Similar ideas to learn transferable skills also appear in recent works \citep{lim2021dynamics,xie2021deep}. Approaches that also adopt latent variable policies and mutual information objectives \citep{eysenbach2018diversity,Sharma2020Dynamics-Aware,mahajan2019maven,zheng2018structured} differ from ours since they either focus on the unsupervised skill discovery in single MDPs or learn in multi-task setups without generalization guarantees.

Meta-learning for RL \citep{vilalta2002perspective}, including Reptile \citep{nichol2018first} and $\text{RL}^2$ \citep{duan2016rl}, is related to our work, which extracts the \textit{prior knowledge} in related MDPs by \eg, recurrent models \citep{wang2016learning,duan2016rl} or feed-forward models \citep{brunskill2013sample}. Alternatively, the gradient information can also be leveraged to meta-learn \citep{finn2017model,nichol2018first}. However, in our work, the learning protocol is more like multi-task learning in the transfer learning literature. The common knowledge in MRA is different from the prior knowledge in meta-learning as the latter captures the high-level essence, while the former is the feature that directly transfers. Also, MRA by \textit{explicitly} modeling the common knowledge and specific knowledge is more suitable for specific problems and has more interpretability.

Recent MARL works \citep{agarwal2019learning,wang2020few,long2020Evolutionary} experimentally reveal the performance benefits of transferring knowledge between population-varying MGs, but with the ultimate goal of training in complex large-population MGs, achieved by curriculum learning. There are also works focusing on multi-task MARL, \eg, \citep{omidshafiei2017deep} with independent learners. Although transfer learning is more about the intra-agent transfer, \ie, between MGs, the inter-agent transfer is also addressed by parameter sharing between cooperative agents \citep{tan1993multi,terry2020parameter} or between homogeneous agents in role-symmetric games \citep{suarez2019neural,Muller2020A}. With the aim of learning dynamic team composition of \textit{heterogeneous} agents, COPA \citep{liu2021coach} introduced a coach-player framework in \textit{single} training games. Besides, the counterfactual reasoning in REFIL \citep{iqbal2021randomized} focused on the multi-task training setting, while we study the generalizability of MARL agents to unseen evaluation games with a theoretically justified objective. There are also works that aim to discover diverse strategic behaviors in MARL \citep{tang2021discovering,lupu2021trajectory}, with randomized policy gradient and zero-shot coordination, respectively. Notably, approaches that model the \textit{dynamical} interaction between agents \citep{wang2019influence,yang2021ciexplore} are orthogonal to the \textit{strategic} modeling in our work.

\section{Preliminaries}
\label{Preliminaries}
\textbf{Markov Game:} An N-agent Markov Game $m$ is defined by the state space $\mathcal{S}$, action sets $\{\mathcal{A}^1, \ldots, \mathcal{A}^N\}$, and observation sets $\{\mathcal{O}^1, \ldots, \mathcal{O}^N\}$. For any agent $i\in[1,N]$, $o^i\in\mathcal{O}^i$ is an observation of the global state $s\in\mathcal{S}$. The state transition and the reward function for agent $i$ are defined as $\mathcal{P}_m: \mathcal{S}\times\mathcal{A}^1\times\ldots\times\mathcal{A}^N\shortrightarrow\Delta(\mathcal{S})$ and $\mathcal{R}^i_m: \mathcal{S} \times \mathcal{A}^i \shortrightarrow [0,1]$, respectively, where $\Delta(\mathcal{S})$ denotes the set of discrete probability distributions over $\mathcal{S}$. The joint strategy is denoted as $\boldsymbol{\pi}=(\pi^1, \ldots, \pi^N)=(\pi^i,\boldsymbol{\pi^{\textrm{-}i}})$, where $\pi^{i}$ is the strategy of agent $i$ and  $\boldsymbol{\pi^{\textrm{-}i}}$ is the joint strategy excluding it. In the following sections, we study the Markov Games with discrete state and action spaces.

In this work, we consider role-symmetric MGs \citep{suarez2019neural,Muller2020A}, where homogeneous agents are with the same reward function and action space. The number of types of homogeneous agents is denoted as $h$, \eg, $h=2$ in an N-agent Pac-Man game representing Pac-Man and ghost agents. Homogeneous agents are symmetric in each game, \ie, changing the policy of an agent with another homogeneous agent will not affect the outcome \citep{terry2020parameter}. Population-varying MGs, including training and evaluation MGs, \eg, varying $N_1$ and $N_2$ in an $N_1$ Pac-Man, $N_2$ ghosts game, are with the same $h$ and with states from state set $S$, whereas described by different transition $\mathcal{P}$ and joint space of observation  $\mathcal{O}$, action $\mathcal{A}$, reward $\mathcal{R}$.

\textbf{Relational Representation:} As an opponent modeling framework, relational representation \citep{long2020Evolutionary,agarwal2019learning,iqbal2018actor,zhang2021structure} aims to capture the strategic relationship between agents and output an embedding $e$ for further policy and critic function learning. Specifically, consider the observation $o^i$ of agent $i$ with entities $o^i = \left[o_s^i, o_1^i, \ldots, o_j^i, \ldots, o_N^i\right]$, where $o_s^i$ is agent $i$'s self properties (\eg, its speed), $o_j^i$ is agent $i$'s observation on agent $j$ (\eg, distance from agent $j$), and the observed environment information (\eg, landmark locations) is concatenated to these entities. Then with self-attention \citep{vaswani2017attention} generating the pair-wise relation $g^{i,j}$, \ie, the $j$-th entity of agent $i$'s (egocentric) relational graph $g^i$, the representation embedding $e^i$ for agent $i$ is formulated as
\begin{equation}
\begin{aligned}
\label{rel_v}
    e^i=\sum_{j \neq i}g^{i,j}\mathcal{V}(o_j^i),
\text{where}\ g^{i,j}=\frac{\exp(\mathcal{Q}(o_s^i)^{\top} \mathcal{K}(o_j^i))}{\sum_{j \neq i}\exp(\mathcal{Q}(o_s^i)^{\top} K(o_j^i))},
\end{aligned}
\end{equation}
where we follow the Transformer architecture \citep{vaswani2017attention} and let $\mathcal{V}(\cdot)$, $\mathcal{Q}(\cdot)$ and $\mathcal{K}(\cdot)$ represent linear functions. The observation embedding with an arbitrary number of agents can thus be represented with a fixed length.

\textbf{Nash Equilibrium:} A core concept in game theory is Nash Equilibrium (NE). When every agent in the MG $m$ acts according to the joint strategy $\boldsymbol{\pi}$ at state $s$, the value of agent $i$, denoted by $v_{\boldsymbol{\pi}}^{i,m}(s)$, is the expectation of $i$'s $\gamma$-discounted cumulative reward. Formally, we define
\begin{align*}
    v_{\boldsymbol{\pi}}^{i,m}(s)= \Exp_{\boldsymbol{a}\sim \boldsymbol{\pi}, s_0=s, s_t\sim {\mathcal{P}_m}} \biggl[\sum_t \gamma ^t r^i_m(s_t,\boldsymbol{a_t})\biggr].
\end{align*}
In this work, the bold symbol is joint over all agents, and variables with superscript $i$ are of agent $i$. Denote the value of the best response for agent $i$ as $v_{\boldsymbol{\pi^{\textrm{-}i}}}^{*i,m}$, which is the best policy of agent $i$ when $\boldsymbol{\pi^{\textrm{-}i}}$ is executed, \ie, $v^{*i,m}_{\boldsymbol{\pi^{\textrm{-}i}}}=\max_{\pi^i}v^{i,m}_{\pi^i,\boldsymbol{\pi^{\textrm{-}i}}}$. Then the joint strategy $\boldsymbol{\pi}$ reaches NE if for any agent $i\in \{1,...,N\},
    v_{\boldsymbol{\pi}}^{i,m}(s)= v_{\boldsymbol{\pi^{\textrm{-}i}}}^{*i,m}(s).$

A common metric to measure the distance to a Nash Equilibrium is \textsc{NashConv}, which represents how much each player (or agent) gains by deviating from the best response (unilaterally) in total. And it can be approximately calculated in small games \citep{johanson2011accelerating,lanctot2017unified}. We denote the \textsc{NashConv} of $\boldsymbol{\pi}$ in the Markov Game $m$ as $\mathcal{D}_m(\boldsymbol{\pi})$, defined as
$$
\mathcal{D}_m(\boldsymbol{\pi})=\mathcal{D}_m(\pi^i,\boldsymbol{\pi}^{\textrm{-}i})=\left\lVert{\left\lVert{v^{*i,m}_{\boldsymbol{\pi}^{\textrm{-}i}} - v^{i,m}_{\boldsymbol{\pi}}}\right\rVert_{s,\infty}}\right\rVert_{i,1} =\sum_{1\leq i\leq N}\max_{s\in\mathcal{S}} \lvert v^{*i,m}_{\boldsymbol{\pi}^{\textrm{-}i}}(s) - v^{i,m}_{\boldsymbol{\pi}}(s)\rvert,
$$
where $\lVert{\cdot}\rVert_{s,\infty}$ is the $\mathcal{L}_{+\infty}$-norm over the state space $\mathcal{S}$ and $\lVert{\cdot}\rVert_{i,1}$ is the $\mathcal{L}_1$-norm over agent indexes. With this definition, the joint strategy $\boldsymbol{\pi}$ reaches NE in $m$ if and only if $\mathcal{D}_m(\boldsymbol{\pi})=0$.

\section{Learning Meta Representations for Agents}
\label{analysis}
\subsection{Problem Statement}
In a single Markov Game, achieving Nash Equilibrium gives reasonable solutions and is of great importance \citep{hu2003nash,yang2018mean,perolat2017learning}. To enable generalization in different MGs, the most straightforward way is to learn a joint strategy set $\boldsymbol{\Pi}$ that contains effective joint strategies for every MG, \eg, the ones that achieve NE. We denote the set of all training MGs as $\mdp$ and the set of evaluation MGs as $\mdp'$. Then the goal is to learn an optimal joint strategy set $\boldsymbol{\Pi}^*$ that satisfies
\begin{equation}
\begin{aligned}
\label{primal}
    \forall m' \in \mdp', \exists \boldsymbol{\pi}\in \boldsymbol{\Pi}^*, \text{~s.t.~} \mathcal{D}_{m'}(\boldsymbol{\pi})=0.
\end{aligned}
\end{equation}

Consider the problem of optimizing $\boldsymbol{\Pi}$ so that the optimal $\boldsymbol{\Pi}^*$ satisfies \eqref{primal}. We first need its size $\lvert\boldsymbol{\Pi}\rvert$ to be sufficiently large to comprise at least one effective strategy for every $m'\in\mdp'$. Then $\boldsymbol{\Pi}$ should be improved with respect to the worst-performing $m'$, \ie, the game with no effective strategy contained in $\boldsymbol{\Pi}$, to achieve low regret $\mathcal{D}_{m'}$ for all $m'$. In other words, $\boldsymbol{\Pi}$ is updated to include the joint strategy $\boldsymbol{\pi}$ that minimizes $\mathcal{D}_{m'}(\boldsymbol{\pi})$. Formally, 
\begin{equation}
\begin{aligned}
\label{mm}
\boldsymbol{\Pi}^* = \argmin_{\boldsymbol{\Pi}}\mathcal{L}(\boldsymbol{\Pi}), \text{~where~} \mathcal{L}(\boldsymbol{\Pi}) = \min_{\boldsymbol{\pi}\sim\boldsymbol{\Pi}}\max_{m'\in \mdp'}\mathcal{D}_{m'}(\boldsymbol{\pi}).
\end{aligned}
\end{equation}
However, minimizing $\mathcal{L}(\boldsymbol{\Pi})$ over the unseen evaluation games in $\mdp'$ is impractical in general. In the following sections, we cope with this intractability by introducing a heuristic algorithm and showing that the resulting objective is indeed equivalent to  \eqref{mm}, optimizing which can lead to the optimal strategy set that satisfies  \eqref{primal}.

\subsection{Relational Representation with Latent Variable Policies}
\label{sec_rr}
Instead of learning independent unimodal policies to form the set $\boldsymbol{\Pi}$, we adopt hierarchical latent variable policies to represent the multimodality, with the game-common and game-specific strategic knowledge explicitly modeled by relational representation. Specifically, for population-varying MGs, we regard the relational graph $g$ as the game-specific knowledge since (1) agents optimally behave in each game by learning the \textit{per-game optimal} relational graph; and (2) agents take different actions when incorporating different strategic relationships so that multiple strategic modes are obtained with varied $g$. For this reason, we treat $g$ as a high-level latent variable that is dynamically generated by $g=\phi(o,z)$, where $z$ is a low-level latent sampled from a learned distribution $p(z\mid m;\psi)$. An agent takes action $a\sim \pi(\cdot\mid o,g;\theta)$, where $g=\phi(o,z)$ and $z\sim p(z\mid m;\psi)$. Then the common knowledge that determines how agents can optimally behave conditioned on different $g$ is distilled into the policy parameter $\theta$, which includes the transformation $\mathcal{V}$ in \eqref{rel_v} and the successive policy network parameters. 

Our implementation is illustrated in Figure \ref{fig_ill}. Specifically, different relational graphs are controlled by the lower-level latent variable $z$, which selects the attention heads inspired by the \textit{option} \citep{sutton1999between} structure. We apply the multi-head self-attention architecture \citep{vaswani2017attention} and set the head number as the dimension of the categorical distribution $p(z)$. Then with a sampled latent, say $z_1$, the relational graph $g_1$ corresponds to the output at the $z_1$-th head. The same relational graph is used for both the decentralized actor and the centralized critic.

\begin{figure*}[h]
\centering
\includegraphics[width=0.88\textwidth]{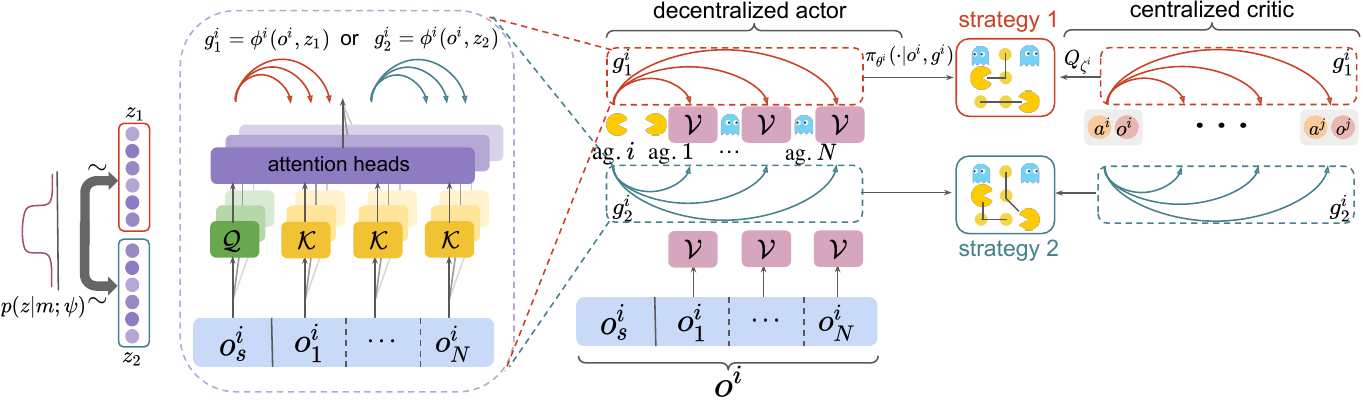}
\caption{The illustration of the implemented architecture of our MRA algorithm. The observation $o^i$ of agent (ag.) $i$ contains the agent's self-information $o_s^i$ and the observed information of other agents, i.e., $o_1^i, \ldots, o_N^i$. After sampling the latent $z$ and computing the (egocentric) relational graph $g^i$, the action and estimated value of agent $i$ can be generated. Here, we show how different strategy modes can be generated, which correspond to $z_1, g_1^i$ and $z_2, g_2^i$.}
\label{fig_ill}
\end{figure*}

Unfortunately, although the above design fits well into the multi-task MARL setup, where each task differs in the number of agents, there is no guarantee that the resulting agents can generalize to novel evaluation MGs. To address this issue, we present the insights and objectives of our \textit{Meta Representations for Agents} (MRA) algorithm as follows.

\subsection{Generalization by Strategic Knowledge Discovery}
Our core idea to enable generalization is to discover the underlying strategic structures in the underlying games. Although the effective policies in the evaluation games are never known during training, agents can still learn the common strategic knowledge and different behavioral modes solely in the \textit{training} games in an unsupervised manner with the imposed \textit{suboptimal} game-specific knowledge. In particular for population varying games, the agents are assigned different strategic relationships, \ie, each agent pays additional attention to some agents while ignoring others. Instead of learning only one optimal joint policy, training with multiple strategic relationships enables the unsupervised discovery of behavior modes, some of which offer appreciable returns in evaluation MGs. Thus, when evaluating in novel games, the desired policy behaviors can be quickly generated by adaptation. In the extreme case that sufficiently many strategic modes are captured with an extremely large latent space, the desired policy for evaluation games can be directly found.

Specifically, agents optimally behave in each game $m\in\mdp$ with the optimal policy parameter $\theta^*$ and the (per-game) optimal relational graph $g^*$. By imposing \textit{knowledge} (or \textit{relation}) \textit{variations} in $m$, \ie, \textit{multiple suboptimal} $g$ at a certain observation, agents learn how the best decisions to accomplish the task are made, \ie, learn $\theta^*$ that achieves the highest average return. With the discovery of distinct strategic modes, the game-common knowledge contained in $\theta^*$ is obtained. Thus, when agents are in the novel MGs $m'\in\mdp'$, their policies can effectively adapt by learning the optimal relational graph in $m'$, or achieve zero-shot transfer (without adaptation) if the latent space is large. This gives the objective of $\theta^i$ that maximizes the average return of all knowledge variations and all training MGs:
\begin{equation}
\begin{aligned}
\label{objective_theta}
\max_{\theta^i}\mathcal{L}(\theta^i) = \max_{\theta^i} \Exp_{m\sim\mdp,g^i}\left[v^{i,m}_{\pi^i(\cdot\mid\cdot,g^i;\theta^{i}), \boldsymbol{\pi^{\textrm{-}i}}}\right].
\end{aligned}
\end{equation}
Notably,  \eqref{objective_theta} differs from the objective of multi-task learning where the average training return is maximized by learning the optimal $\theta^*$ and a \textit{single optimal} relational graph in each game.

In order to perform well in all $m'\in\mdp'$, the strategic modes captured during training should cover as many behaviors as possible. This requires a large latent space size $\lvert Z\rvert$ and \textit{diverse} actions. Therefore, we introduce a diversity-driven objective that encourages high mutual information between $g$ and $a$ for behavior diversity, as well as between $m$ and $g$ to extract game-specific knowledge, defined as follows:
\begin{equation}
\begin{aligned}
\label{objective_phi}
\max_{\psi^i,\phi^i}\mathcal{L}(\psi^i,\phi^i) = \max_{\psi^i,\phi^i} \mathcal{I}(g^i;a^i\mid o^i) + \mathcal{I}(m;g^i\mid o^i),
\end{aligned}
\end{equation}
where $\mathcal{I}(g;a\mid o) = \mathcal{H}(a\mid o) - \mathcal{H}(a\mid o,g)$ is the mutual information between $g$ and $a$ conditioned on observation $o$. With a slight abuse of notation, the variable $m$ denotes the basic information of game $m$, such as the numbers of agents of each set of homogeneous agents. 

\subsection{Fast Adaptation with Limited Latent Space Size}
\label{fast_adaptation_latent_space}
Despite the diversity-inducing objective, we also need a large enough latent space $\lvert Z\rvert$. Specifically, denote by $\boldsymbol{\Pi}_{\Theta}$ the joint strategy set parameterized by $\Theta=\{\psi,\phi,\theta\}$. Then achieving zero-shot adaptation requires $\lvert Z\rvert=\lvert\boldsymbol{\Pi}_{\Theta}\rvert\geq\lvert\boldsymbol{\Pi}^*\rvert$. However, without further assumptions on the Markov Games in the evaluation set $\mdp'$, the size $\lvert\boldsymbol{\Pi}^*\rvert$ will be unbounded. Therefore, with practical limited-size latent models, we adopt the techniques from Reptile \citep{nichol2018first} to fast adapt to evaluation games, achieved by repeatedly selecting a game and moving the parameter towards the trained weights on this game to find the point near all games' solution manifolds.


Specifically, we optimize $\theta$ by performing $K$ policy gradient steps on each individual MG, instead of trying to maximize the average return over all training games in a joint way. Formally, after selecting a training MG $m$, the objective for $\theta$ in the $k$-th mini-batch changes from \eqref{objective_theta} to $\mathcal{L}^k_{m}(\theta^i) = \Exp_{g^i}[v^{i,m}_{\pi^i(\cdot\mid\cdot,g^i;\theta^{i}), \boldsymbol{\pi^{\textrm{-}i}}}]$. Let $U_m^K(\theta)$ denote the policy parameter after $K$ gradient steps\footnote{The gradient can be calculated by leveraging a centralized value estimation as described in \eqref{grad_pi}.} with learning rate $\beta$. Then $\theta$ is updated by $\theta\shortleftarrow\theta +\alpha \Delta\theta$, where $\alpha$ is a hyperparameter and $\Delta\theta = U_m^K(\theta) - \theta$. By doing so, the first-order gradient information can be leveraged to update $\theta$ towards the instance-specific adapted policy parameter. This can be seen by writing the expected update $\mathbb{E}[\Delta \theta]$ over mini-batches in $m$ as
\begin{equation}
    \label{update_expected}
    \begin{aligned}
    \mathbb{E}[\Delta \theta] \!= (K-1) \mathbb{E} \left[\nabla\mathcal{L}_m^k(\theta)\right] + \frac{(K-1)(K-2)\beta}{2} \mathbb{E}\left[\nabla{\Big(\nabla{\mathcal{L}_m^k(\theta)} \nabla{\mathcal{L}_m^j(\theta)}}\Big)\right],
    \end{aligned}
\end{equation}
where $\nabla{\mathcal{L}}_m^k(\theta)$ is the gradient at the initial $\theta$. The derivation is in Appendix \ref{fast_adap}. Notably, the second term on the RHS in  \eqref{update_expected} is with the direction that increases the inner product between gradients of different mini-batches $j,k$. That is, $\theta$ is optimized not only to maximize the return under all relation variations, but also towards the place where gradients of different variations point to the same direction, \ie, the place that is easy to optimize from. With this property, when the optimal strategic modes of evaluation MGs are not discovered during training, the learned $\theta$ can still fast adapt to effective policies.

\begin{wrapfigure}{R}{0.55\textwidth}
\begin{minipage}{0.55\textwidth}
        \vspace{-1.3cm}
  \begin{algorithm}[H]
    \caption{MRA: Training in the MG set $\mdp$}
    \begin{algorithmic}[1]
\WHILE {not converged}
    \FOR {$\text{MG } m \in \mdp$ with population $N_m$}
        \STATE {Every agent samples $z^i\sim p_{\psi^i}(\cdot|m)$ and  $g^i = \phi^i(o^i,z^i)$}\hspace{-0.1cm}
            \STATE {Execute $a^1, \ldots, a^{N_m}$ in $m$, where $a^i\!\sim\!\pi_{\theta^i}(\cdot|o^i,g^i)$}
            \FOR {agent $i = 1, \ldots, N_m$}
            \STATE {Update $\theta^i$ by $\theta^i\shortleftarrow\theta^i + \alpha(U_m^K(\theta^i) - \theta^i)$}\label{line1}\\
            \STATE {Update $\psi^i$ and $\phi^i$ by \eqref{objective_phi}}\label{line2}
    \ENDFOR
    \ENDFOR
\ENDWHILE
    \end{algorithmic}
    \label{alg_high}
  \end{algorithm}
  \vspace{-40pt}
\end{minipage}
\end{wrapfigure}
The pseudocode of \textit{Meta Representations for Agents} (MRA) is presented in Algorithm \ref{alg_high}, where iterative optimization of \eqref{objective_theta} and  \eqref{objective_phi} is performed to update the policy parameter $\theta$ and the parameters $\psi, \phi$ in the relational representation.
It is worth noting that the right pseudocode is a high-level abstract of our method. We will discuss the optimization procedure and the algorithm implementation in Section \ref{sec_opt} and Appendix \ref{pcode}.

\section{Analysis}
\label{analysis_the}
In this section, we provide a theoretical analysis of MRA. Specifically, we show that the optimal parameters resulting from the MRA objective in \eqref{objective_theta} and  \eqref{objective_phi} ensure generalizability under certain conditions.
To begin with, we introduce the Markov state transition operator $\mathcal{P}_m^{\pi^i,\boldsymbol{\pi^{\textrm{-}i}}}$ in MG $m$, defined as
$$
\left(\mathcal{P}_m^{\pi^i,\boldsymbol{\pi^{\textrm{-}i}}} x\right)(s)= \int_{s'\sim \mathcal{S}}x(s')\Exp_{\pi^i,\boldsymbol{\pi^{\textrm{-}i}}}\left[\mathcal{P}_m(d s'\mid s,a^i,\boldsymbol{a^{\textrm{-}i}})\right].
$$
Here, $x\colon \mathcal{S}\shortrightarrow\mathbb{R}$ is an $L_1$ Lebesgue integrable function. The norm of the operator is defined as $\lVert{\Lambda}\rVert_{op} := \sup\{\lVert{\Lambda x}\rVert_{s,1}\colon \lVert{x}\rVert_{s,1}\leq 1 \}$.

Then we make the assumption of Lipschitz Game.
\begin{assumption}
\label{lip_assumption}
(Lipschitz Game). For any Markov Game $m\in(\mdp\cup\mdp')$, there exists a Lipschitz coefficient $\iota_m>0$ such that for all agent in $m$ and $s\in\mathcal{S}$:
\begin{align*}
\left\lVert{\mathcal{P}_m^{\pi^{*i},\boldsymbol{\pi^{\textrm{-}i}}} - \mathcal{P}_m^{\pi^{i},\boldsymbol{\pi^{\textrm{-}i}}}}\right\rVert_{op} \leq \iota_m \left\lVert{\left\lVert{\pi^{*i}(a\mid s)-\pi^i(a\mid s)}\right\rVert_{a,1}}\right\rVert_{s,\infty},
\end{align*}
where $\lVert{\cdot}\rVert_{a,1}$ is the $\mathcal{L}_1$-norm over the action space $\mathcal{A}$. 
\end{assumption}

The Lipschitz assumption has been made in a plethora of preceding studies \citep{liu2021policy,zhang2019policy,zhang2022conservative}. We note that Assumption \ref{lip_assumption} is reasonable since the Lipschitz coefficient $\iota_m$ can be interpreted as the \textit{influence} of agents \citep{radanovic2019learning,dimitrakakis2019multi}, which measures how much the policy changing of an agent can affect the game environment.

Then we define a distance metric that measures the discrepancy between $\mdp$ and $\mdp'$ by comparing and computing the distance to NE in the games of the two sets. Let $N_m$ denote the total number of agents in game $m$, and $h_{i,m}$ denote the homogeneous agent set of agent $i$ in $m$. 
\begin{definition}
\label{def_distance_sets}
For two sets of MGs $\mdp$ and $\mdp'$, we define the distance $\varsigma$ between $\mdp$ and $\mdp'$ by
\begin{align*}
\varsigma = \max_{\substack{m'\in \mdp'\\i\in\{1,\ldots,N_{m'}\}}}\min_{ \substack{m\in\mdp,i'\in h_{i,m}\\ \boldsymbol{\pi}\in\{\boldsymbol{\pi}\mid\mathcal{D}_{m}(\boldsymbol{\pi})=0\}\\\boldsymbol{\pi'}\in\{\boldsymbol{\pi'}\mid\mathcal{D}_{m'}(\boldsymbol{\pi'})=0\}}}\mathcal{D}_{m'}(\pi^{i'},\boldsymbol{\pi'^{\textrm{-}i}}).
\end{align*}
\end{definition}

We also define the $\epsilon$-range joint strategy set $\boldsymbol{\hat{\Pi}}$ to guide the policy learning of agents during training.
\begin{definition}
\label{def4}
For the training MG set $\mdp$ and $\epsilon > 0$, the $\epsilon$-range joint strategy set $\boldsymbol{\hat{\Pi}}$ is defined as:
\begin{align*}
\boldsymbol{\hat{\Pi}} = \bigcup_{m\in\mdp}\boldsymbol{\hat{\Pi}_m}, 
\text{~~where~~} \; \boldsymbol{\hat{\Pi}_m} = \{\boldsymbol{\pi} \mid\mathcal{D}_m(\boldsymbol{\pi}) \leq \epsilon\}.
\end{align*}
\end{definition}
By bounding $\epsilon$ that characterizes a large set $\boldsymbol{\hat{\Pi}}$, we show by the following theorem that \eqref{mm} can be solved from the constrained mutual information maximization objective. 
\begin{theorem}
\label{theorem1}
If $\lvert\boldsymbol{\Pi}_{\Theta}\rvert \geq \lvert\boldsymbol{\hat{\Pi}}\rvert$ and $\epsilon$ satisfies $\epsilon\geq \varsigma - \min_{\iota_m,\iota_{m'}}\frac{\varsigma\gamma\left(\iota_{m'}-\iota_{m}\right)}{\gamma \iota_{m'}+1-\gamma}$, then with the optimal parameters $\Theta^*=\{\psi^*,\phi^*,\theta^*\}$ given by
\begin{equation}
\begin{aligned}
\label{objective}
    \psi^*, \phi^* = \argmax_{\psi,\phi} \mathcal{I}(g;a\mid o) + \mathcal{I}(m;g\mid o)
    \text{~~s.t.~~}  \boldsymbol{\pi}_{\theta^*} \in \boldsymbol{\hat{\Pi}},
\end{aligned}
\end{equation}
for every evaluation Markov Game $m'\in \mdp'$, there exists a joint strategy $\boldsymbol{\pi}\in \boldsymbol{\Pi}_{\Theta^*}$ that reaches Nash Equilibrium (\ie, $\boldsymbol{\Pi}_{\Theta^*}=\boldsymbol{\Pi}^*$ satisfies \eqref{primal}). Here, the variables and parameters are per-agent, \eg, $\boldsymbol{\pi}_{\theta}$ is joint over $\pi_{\theta^i}$, and the superscript is omitted for clarity.\vspace{-0.3cm}
\end{theorem}

\begin{proof}
See Appendix \ref{proof_label} for a full proof and Appendix \ref{proof_sketch} for a proof sketch.\vspace{-0.3cm}
\end{proof}

Theorem \ref{theorem1} suggests a general paradigm of diversity-driven learning that is effective when $\boldsymbol{\hat{\Pi}}$ satisfies certain properties. In practical MGs, however, the unknown Lipschitz coefficient and the hardness of calculating $\varsigma$ pose challenges to computing the satisfying $\epsilon$. An approximation to the optimal parameters in  \eqref{objective} is to perform iterative optimization following  \eqref{objective_theta} and \eqref{objective_phi}. 

With fixed $\phi$ and $\psi$, the objective of $\theta$ in  \eqref{objective_theta} (greedily) maximizes the expected value over variations in order to minimize the distances to Nash of different policy modes. In other words, the distance $\mathcal{D}_m(\boldsymbol{\pi})$ to the equilibrium of the corresponding joint strategies is minimized to satisfy $\mathcal{D}_m(\boldsymbol{\pi})\leq \epsilon$ in the long run. Then the optimization of $\phi$ and $\psi$ follows to maximize $\mathcal{I}(g;a\mid o) + \mathcal{I}(m;g\mid o)$. By iteratively improving the mutual information and updating $\theta$ towards the $\epsilon$-range $\boldsymbol{\hat{\Pi}}$, the obtained solutions are close to the optimal parameters in  \eqref{objective}. Besides, the condition $\lvert\boldsymbol{\Pi}_{\Theta}\rvert \geq \lvert\boldsymbol{\hat{\Pi}}\rvert$ in the theorem supports the intuition that a sufficiently large policy set (or latent space) is required for zero-shot transfer. However, as an approximation to the theorem, MRA also has some limitations, which we discuss and provide potential improvements in Section \ref{conc}.

\section{Objective Optimization}
\label{sec_opt}
We have shown that the iterative optimization of MRA arises from a theoretically justified objective. In this section, we present the optimization procedures in Line \ref{line1} and Line \ref{line2} of Algorithm \ref{alg_high} that correspond to the policy gradient and the mutual information maximization, respectively.
\label{methods}
\subsection{Multi-Agent Actor-Critic Policy Gradient}
The policy parameter $\theta^i$ in objective  \eqref{objective_theta} is optimized by introducing a \textit{centralized} critic $Q_{\zeta^i}$ for each agent $i$ \citep{lowe2017multi}. Denote the target network with delayed policy and critic parameters as $\bar{\theta}$, $\bar{\zeta}$, and replay buffer as $D$. The parameterized critic $Q_{\zeta^i}$ is optimized to minimize
\begin{equation}
\begin{aligned}
\label{loss_q}
\mathcal{L}(\zeta^i) = \Exp_{(\boldsymbol{o},\boldsymbol{a},\boldsymbol{o'},\boldsymbol{r}) \sim D}\left[\left(Q_{\zeta^i}(\boldsymbol{o},\boldsymbol{a}) - y^i \right)^2\right], 
\text{where}\ y^i = r^i + \gamma  \Exp_{\boldsymbol{a'} \sim \boldsymbol{\pi}_{\bar{\theta}}}\left[Q_{\bar{\zeta}^i}(\boldsymbol{o'},\boldsymbol{a'}) \right]
\end{aligned}
\end{equation}

Then the gradient of the policy parameter $\theta^i$ of agent $i$ during training is given by
\begin{equation}
\begin{aligned}
\label{grad_pi}
\nabla_{\theta^i}\mathcal{L}(\boldsymbol{\pi}) &=  \Exp_{(\boldsymbol{o},\boldsymbol{g}) \sim D, \boldsymbol{a}\sim\boldsymbol{\pi}}\left[\nabla_{\theta^i}\log\pi_{\theta^i}(a^i\mid o^i,g^i)Q_{\zeta^i}(\boldsymbol{o},\boldsymbol{a})\right],
\end{aligned}
\end{equation}
where $Q_{\zeta^i}(\boldsymbol{o},\boldsymbol{a})$ can also be replaced by the advantage function $A_{\zeta^i}(\boldsymbol{o},\boldsymbol{a}) := Q_{\zeta^i}(\boldsymbol{o},\boldsymbol{a}) - \Exp_{\boldsymbol{a}}[Q_{\zeta^i}(\boldsymbol{o},\boldsymbol{a})]$.


When evaluation in a novel MG, $\theta^i$ and $\phi^i$ is fine-tuned to greedily maximize agents' individual rewards. Denote $\omega^i = \{\theta^i,\phi^i\}$. Then the gradient of $\omega^i$ is given by
\begin{equation}
\begin{aligned}
\label{grad_pi_eval}
\nabla_{\omega^i}\mathcal{L}(\boldsymbol{\pi}) &=  \Exp_{\substack{\boldsymbol{o}\sim D, \boldsymbol{a}\sim\boldsymbol{\pi}}}\left[\nabla_{\omega^i}\log\pi_{\theta^i}(a^i\mid o^i,\phi^i(o^i,z^i))Q_{\zeta^i}(\boldsymbol{o},\boldsymbol{a})\right].
\end{aligned}
\end{equation}

In role-symmetric games, the parameters $\theta,\phi,\zeta$ are shared by homogeneous agents. And $\phi$ can be implemented as the option architecture \citep{sutton1999between}, \ie, $g$ corresponds to the $z$-th option sampled from the categorical distribution. Implementation details and the variants can be found in Appendix \ref{add_exp}.

\subsection{Mutual Information Maximization}
\label{o_s}
In an iteration, several update steps of actor and critic are followed by mutual information $\mathcal{I}(g^i;a^i\mid o^i) + \mathcal{I}(m;g^i\mid o^i)$ maximization. According to the definition of mutual information and the non-negativeness of KL divergence, the following bound holds. And $\phi$ is optimized to optimize  \eqref{mi1} by gradient ascent. All the derivations in this section are provided in Appendix \ref{mi_der}.
\begin{equation}
\begin{aligned}
\label{mi1}
\mathcal{I}(g^i;a^i\mid o^i) \!\geq\! \Exp_{g^i, o^i\sim D, a^i \!\sim\! \pi_{{\theta^i}}(\cdot\mid o^i,g^i)}\left[\log\frac{\pi_{\bar{\theta}^i}(a^i\mid o^i,g^i)}{p(a^i\mid o^i)}\right],
\end{aligned}
\end{equation}
where $p(a^i\mid o^i) = \Exp_{z'\sim p(\cdot\mid m), g'=\phi^i(o^i,z')}\left[\pi_{\bar{\theta}}(a^i\mid o^i,g')\right]$.

The second mutual information term $\mathcal{I}(m;g^i\mid o^i)$ can be simplified by
\begin{equation}
\begin{aligned}
\label{mimi2}
    \mathcal{I}(m;g^i\mid o^i) = \mathbb{E}_{m,o^i\sim D}\left[\log p(m\mid o^i,g^i)\right] + \log \lvert\mdp\rvert,
\end{aligned}
\end{equation}
where $\lvert\mdp\rvert$ is the number of training MGs. To calculate the RHS of  \eqref{mimi2}, we introduce an auxiliary inference network $\xi$. Denote the one-hot index of MG as $x$. Then the auxiliary network outputs the index prediction $\hat{x}$, \ie, $p(\hat{x}\mid o,g;\xi)$. The cross-entropy objective is given as follows:
\begin{equation}
\begin{aligned}
\label{mi2}
    \min_{\psi,\xi} \mathbb{E}_{z \sim p(\cdot\mid m;\psi), o^i\sim D}\left[-x \log\left(p\left(\hat{x}\mid o^i,\phi^i(o^i,z);\xi\right)\right)\right].
\end{aligned}
\end{equation}
By minimizing \eqref{mi2}, $\psi$ and $\xi$ are simultaneously optimized. 

\section{Experiments}
\label{exp}
Experiments are conducted in three environments built upon the particle-world framework \citep{lowe2017multi} that cover both competitive and mixed games, including the treasure collection environment, resource occupation environment, and the Pacman-like world. We provide the experiment settings and task descriptions to Appendix \ref{sec_app_env_set}. 

\subsection{Benefits of game-common Strategic Knowledge}
\label{benefits_single}
We first conduct experiments to show the benefits of the proposed method by demonstrating that when training in various Markov Games, the game-common strategic knowledge can benefit individual training games.

In Figure \ref{single_fig}, we compare MRA and the following baseline methods: (1) the MADDPG algorithm \citep{lowe2017multi} with relational representations (\textbf{MADDPG}); (2) the Reptile algorithm \citep{nichol2018first} (\textbf{Reptile}); (3) baseline with the same network architectures as MRA, but agents learn their policies only in a \textit{single} MG (\textbf{baseline}).
\vspace{-0.45cm}

\begin{figure*}[htbp]
\centering
\subfigure[Treasure collection.]{
\begin{minipage}[t]{0.23\linewidth}
\centering
\includegraphics[width=1.7in]{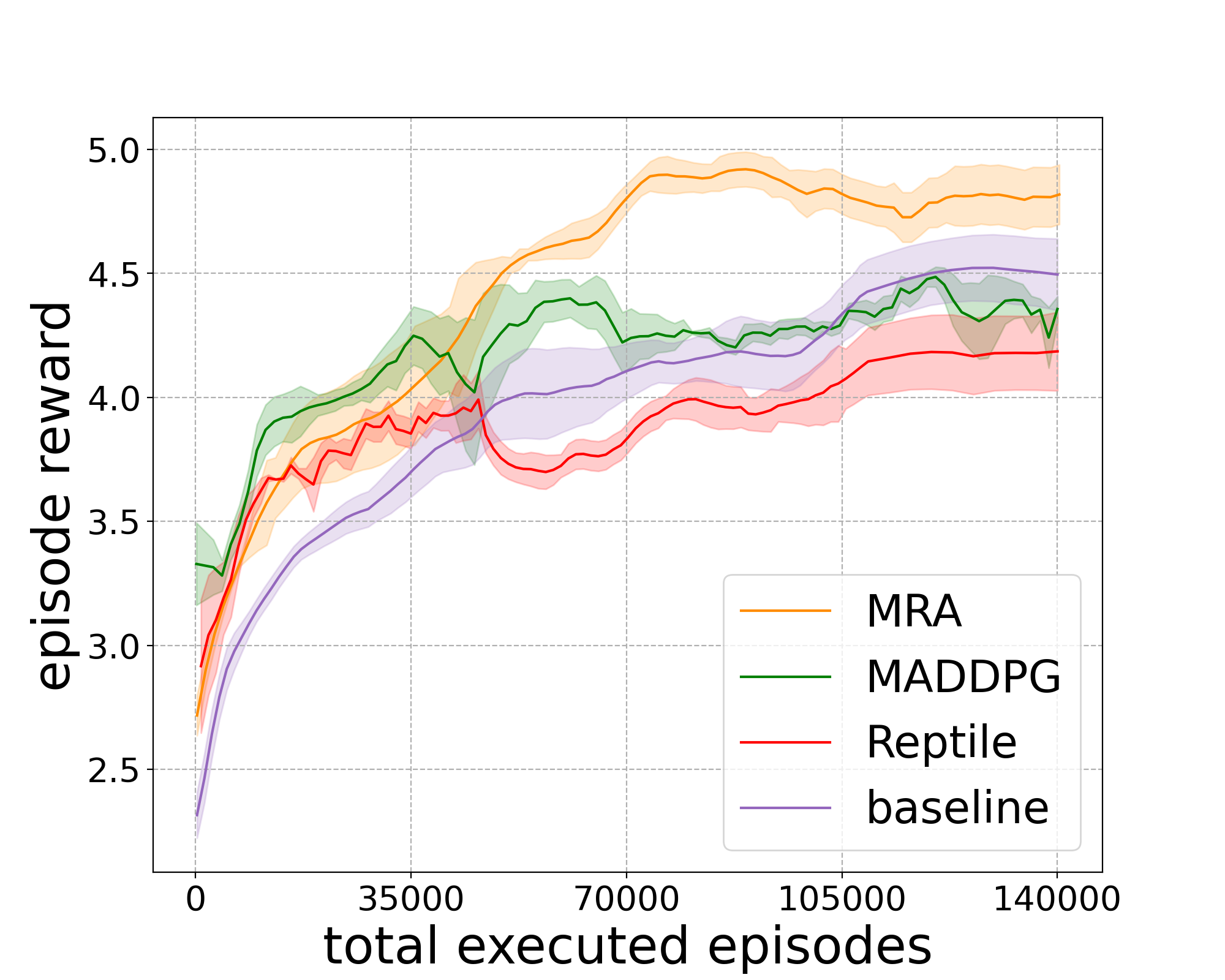}
\end{minipage}%
}%
\subfigure[Resource  occupation.]{
\begin{minipage}[t]{0.23\linewidth}
\centering
\includegraphics[width=1.7in]{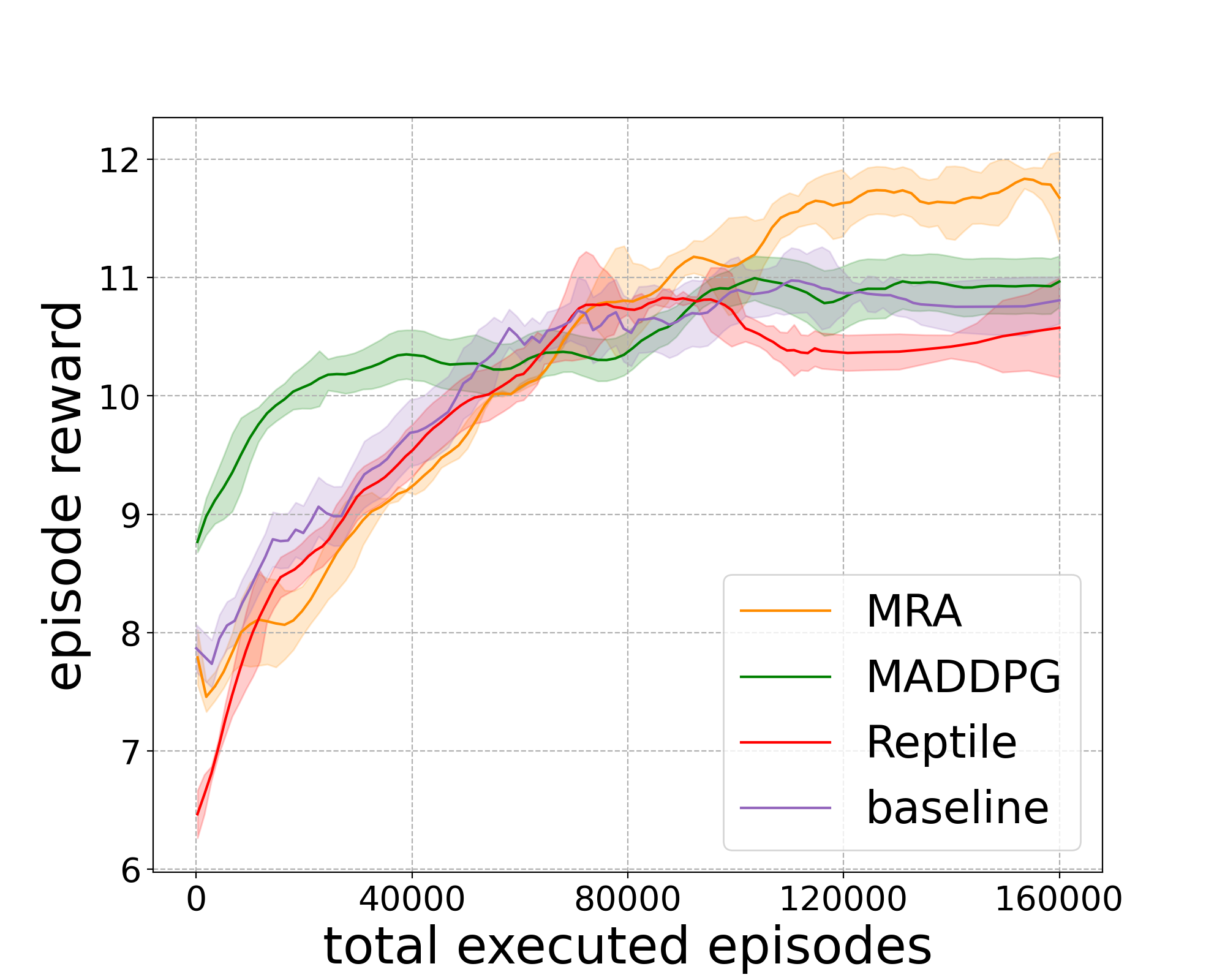}
\end{minipage}%
}%
\subfigure[Pacman-like world: \textbf{Left:} Pac-Man. \textbf{Right:} Ghost.]{
\begin{minipage}[t]{0.53\linewidth}
\centering
\includegraphics[width=1.7in]{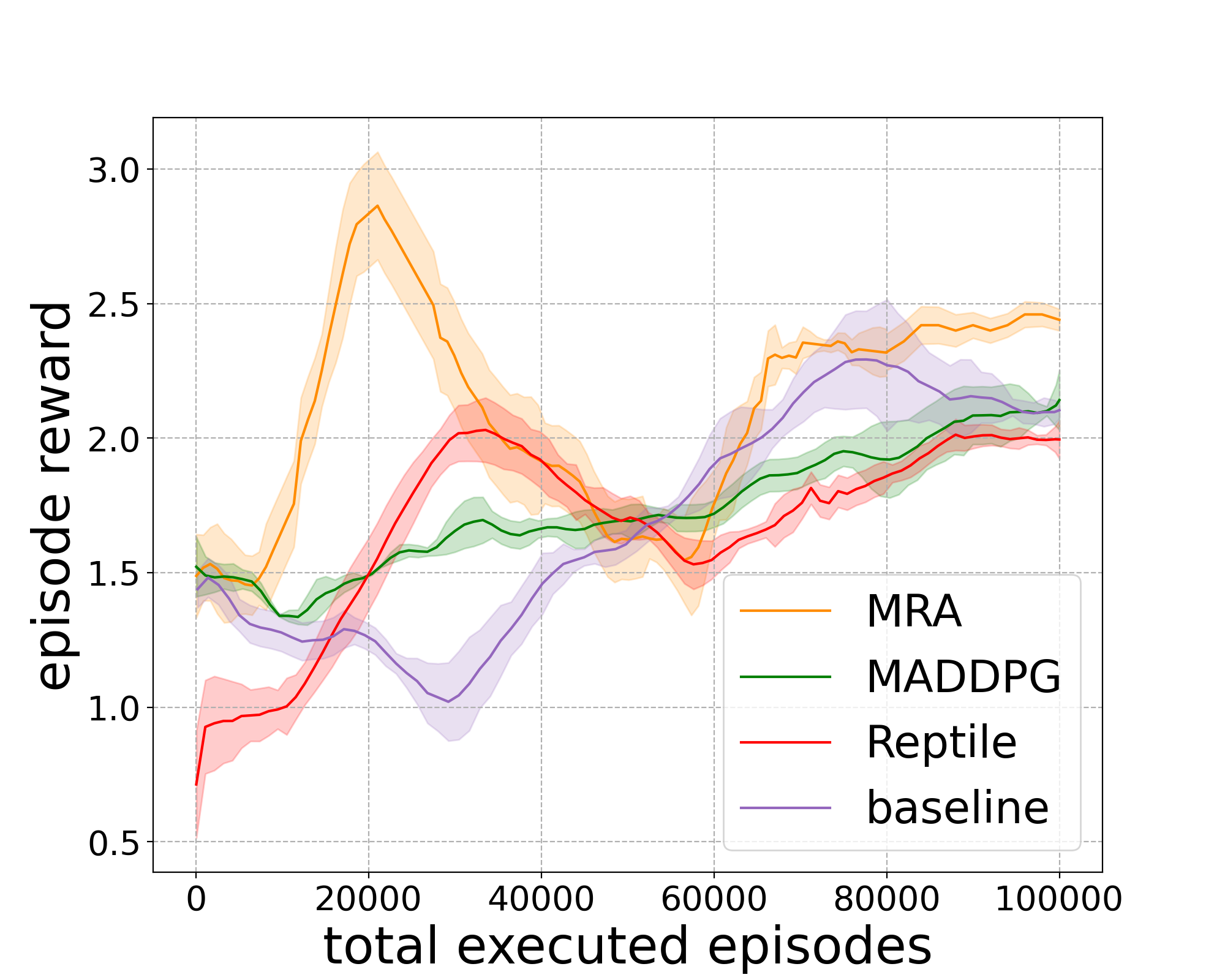}
\hspace{-0.17in}
\includegraphics[width=1.7in]{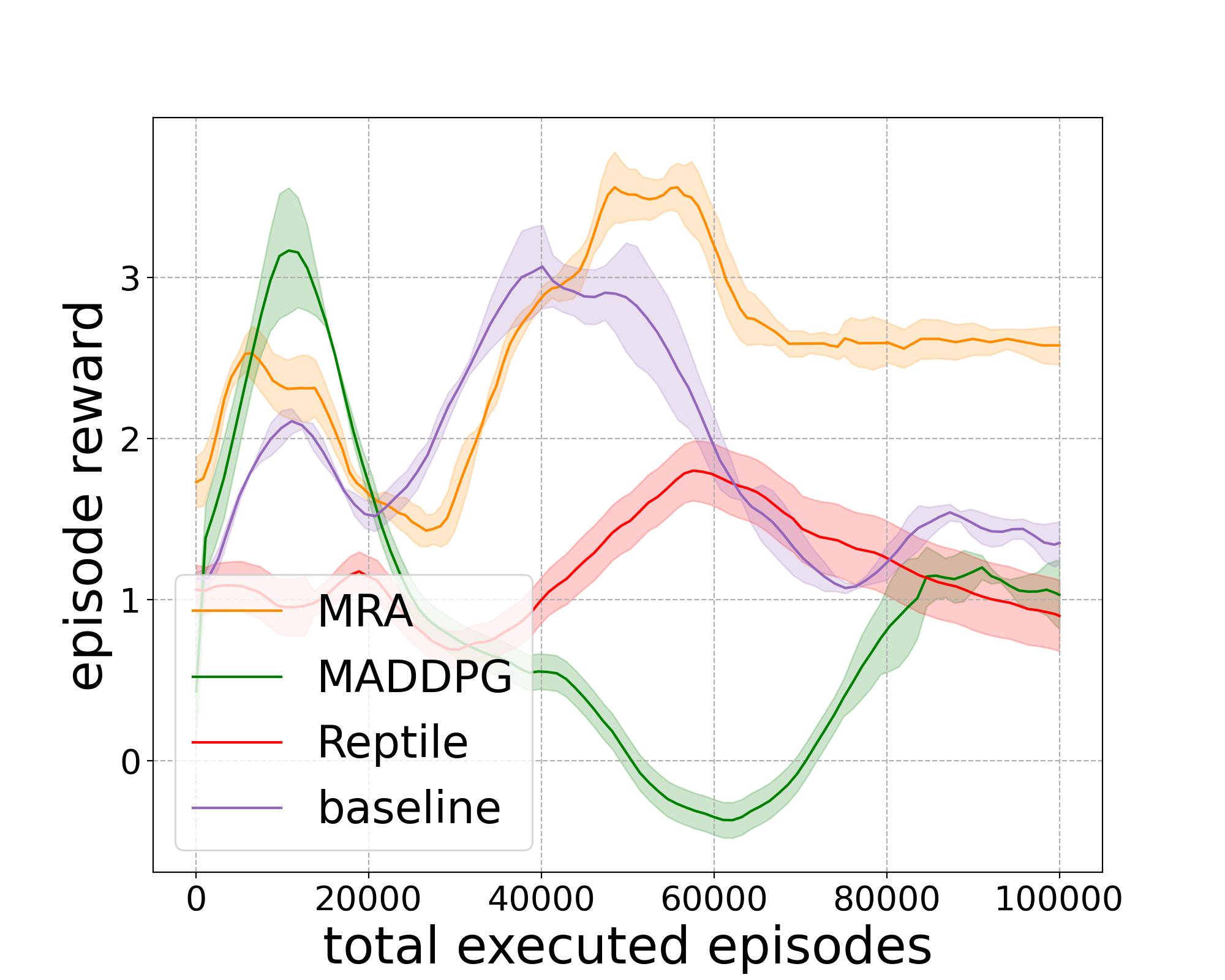}
\end{minipage}
}%
\vspace{-0.2cm}
\caption{Benefits of meta-representations in the three environments. \textbf{(a):} $6$ collector agents and $20$ treasure dots;  \textbf{(b):} $12$ agents in the $6$-resources environment;  \textbf{(c):} $8$ Pac-Man agents, $4$ ghost agents and $20$ food dots.\vspace{-0.2cm}}
\label{single_fig}
\end{figure*}

For MRA and Reptile, the size of training MG set $\lvert\mdp\rvert$ for the three environments is $4,4,3$, respectively. The settings of the training Markov Games in these three environments are as follows. For the treasure collection task, there are $4$ training MGs. The numbers of agents are $3$, $6$, $12$, and $24$, respectively, which we denote as $\{3,6,12,24\}$. The setting of the $4$ training MGs in resource occupation task is $\{6,9,12,15\}$. For the PacMan task, there are $3$ training MGs in total: $\{(4,2),(6,3),(8,4)\}$, where $(4,2)$ in the first MG denotes that there are $4$ PacMan agents and $2$ ghost agents. 

We note that the \textit{actual} executed episodes of MRA in one MG are $\lvert\mdp\rvert$ times \textit{smaller} than that for baseline and MADDPG, which reveals the efficiency of the proposed meta-representation. Comparisons between MRA and Reptile validate the effectiveness of MRA to meta-represent effective policies in various MGs. Due to the existence of game-specific knowledge, the optimal policies in different games are distinct. Thus, in Reptile, using a unimodal policy to represent them all negatively affect the performance. By comparing MRA with baseline and MADDPG, the benefit of the game-common strategic knowledge for individual games is revealed. Since the baseline agents are trained only in a single game, they cannot benefit from such common knowledge, even with more episodes executed. The common knowledge helps MRA outperform MADDPG. We also evaluate MADDPG with homogeneous agents sharing parameters, which works poorly. Although the Pacman-like world is not a zero-sum game, we still provide cross-comparison results in Appendix \ref{sec_app_cc} for completeness. 

\subsection{Performance Comparison in Multiple Games}
\label{performance_multi}
In this part, we compare the performance of MRA with multi-task and meta-learning methods, including \textbf{EPC} \citep{long2020Evolutionary} and \textbf{$\text{RL}^2$} \citep{duan2016rl}. Specifically, EPC learns policies from multiple training MGs with relational representations and evolutionary algorithms. However, EPC \textit{refits} each training MG after obtaining effective policies in the previous training MG. On the contrary, MRA and $\text{RL}^2$ learn the generalizable strategic knowledge and essence-capturing prior knowledge, respectively.

The EPC algorithm, one of the curriculum learning approaches, is implemented by initializing $3$ parallel sets of agents and mix-and-match the top $2$ sets to the successive MG. For the meta-RL algorithm $\text{RL}^2$, each trial contains a cycle of all the $\lvert\mdp\rvert$ MGs. We also compare another MRA variant, \textbf{uni-MRA}, that samples $z$ from a uniform distribution. In the treasure collection and resource occupation tasks, the results are shown in Figure \ref{fig:add}. 

\begin{figure*}[htbp]
\centering
\vspace{-0.3cm}
\subfigure{
    \begin{minipage}[t]{0.25\linewidth}
        \centering
        \includegraphics[width=1.7in]{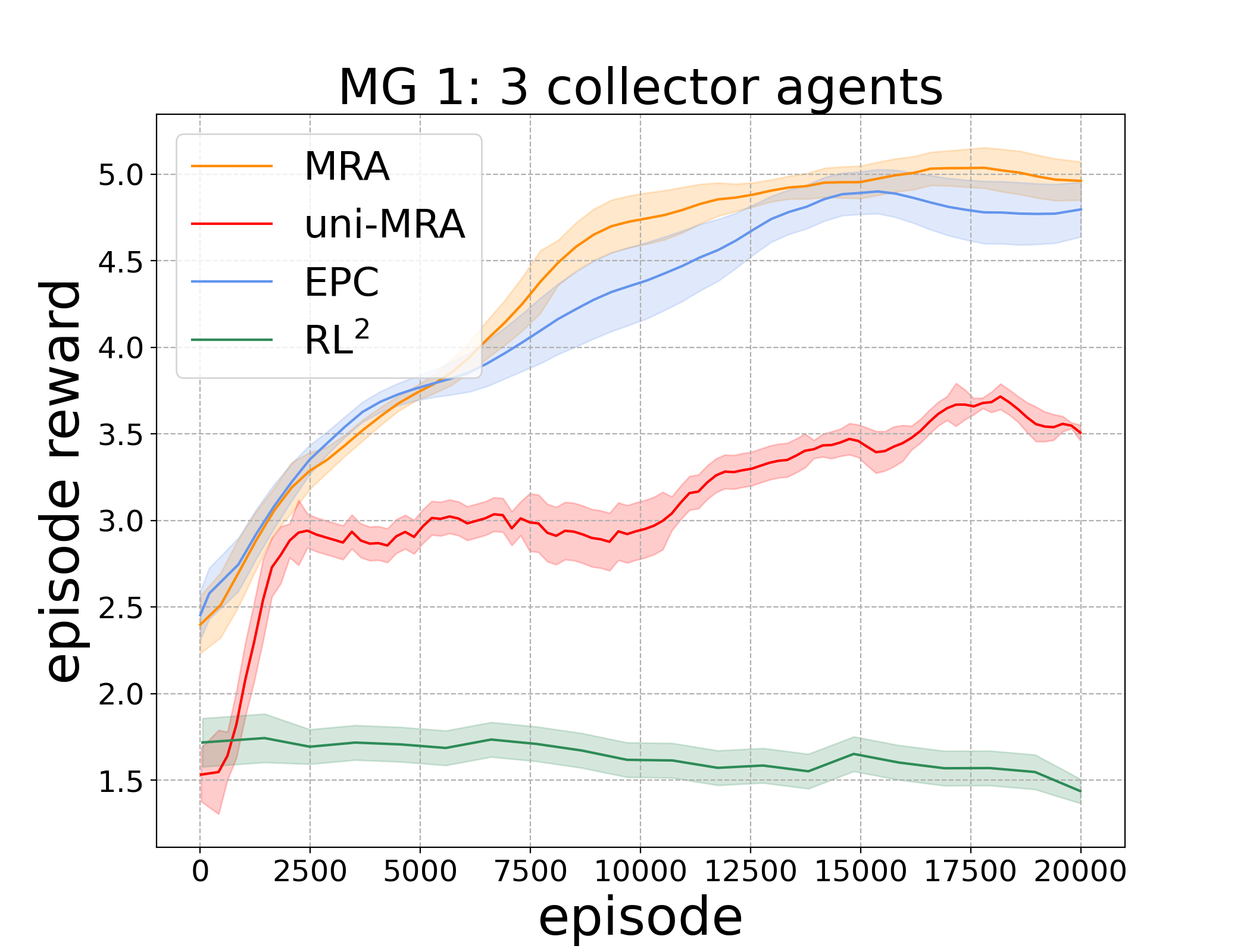}\\
    \end{minipage}%
}%
\subfigure{
    \begin{minipage}[t]{0.25\linewidth}
        \centering
        \includegraphics[width=1.7in]{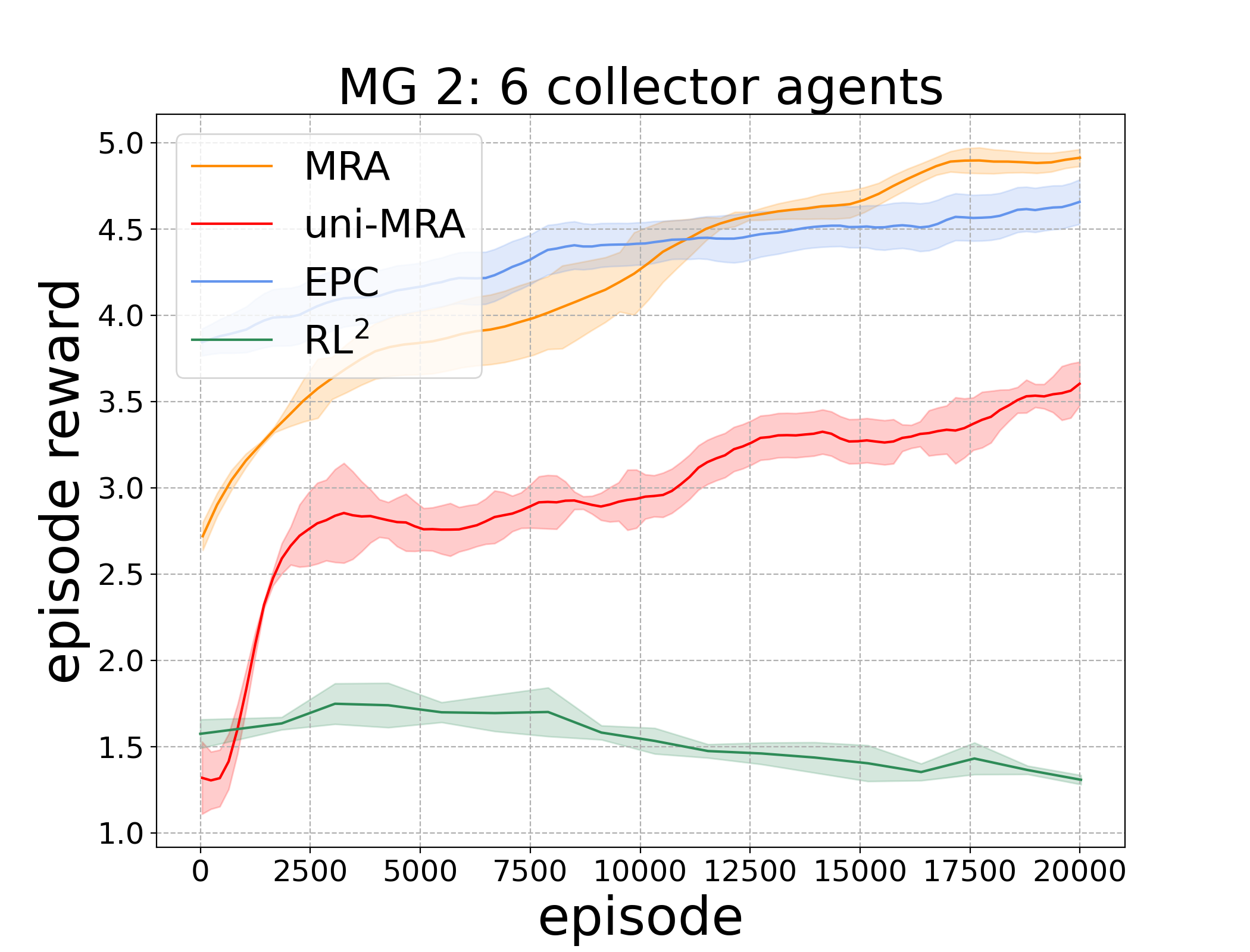}\\
    \end{minipage}%
}%
 \subfigure{
    \begin{minipage}[t]{0.25\linewidth}
        \centering
        \includegraphics[width=1.7in]{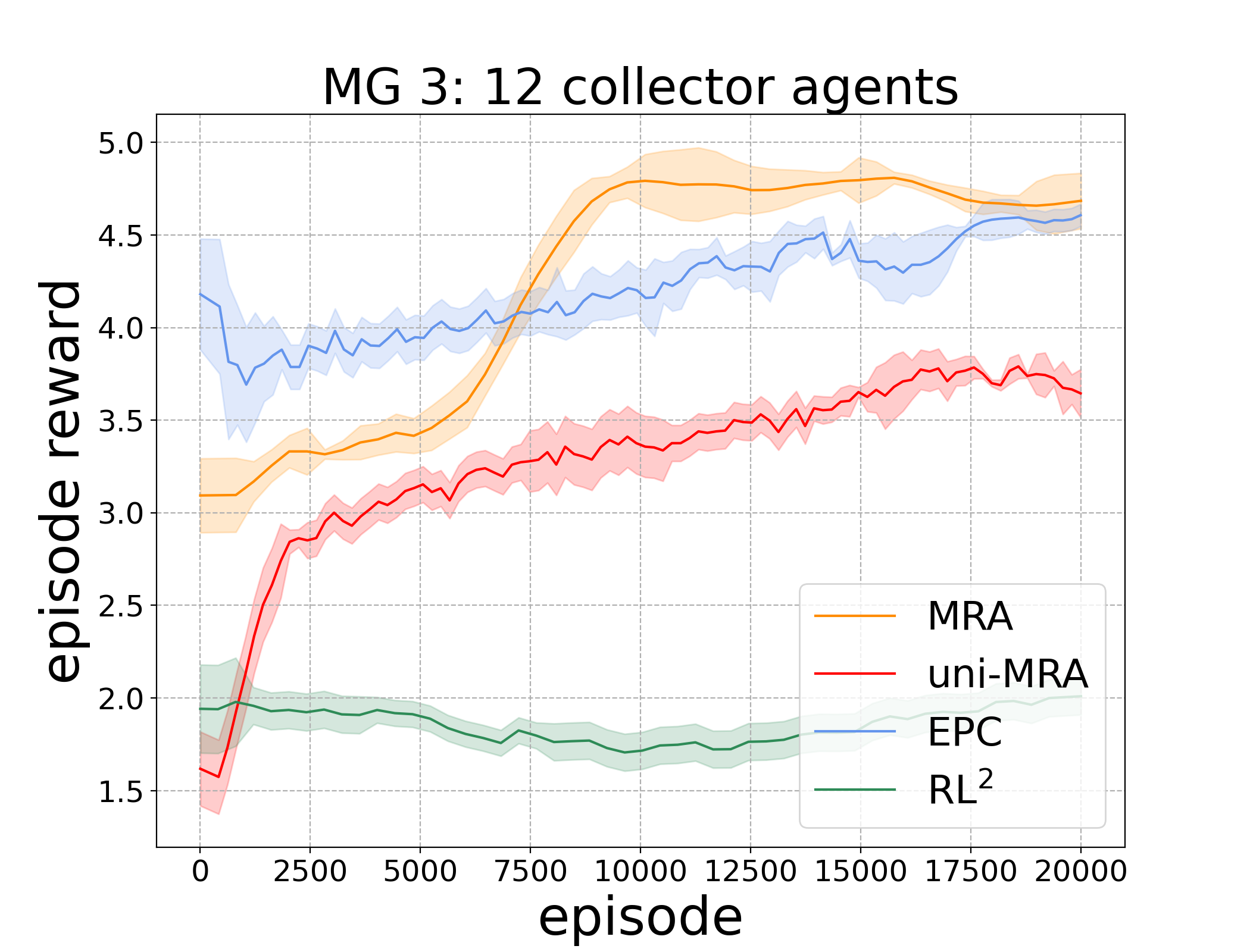}\\
    \end{minipage}%
}%
 \subfigure{
    \begin{minipage}[t]{0.25\linewidth}
        \centering
        \includegraphics[width=1.7in]{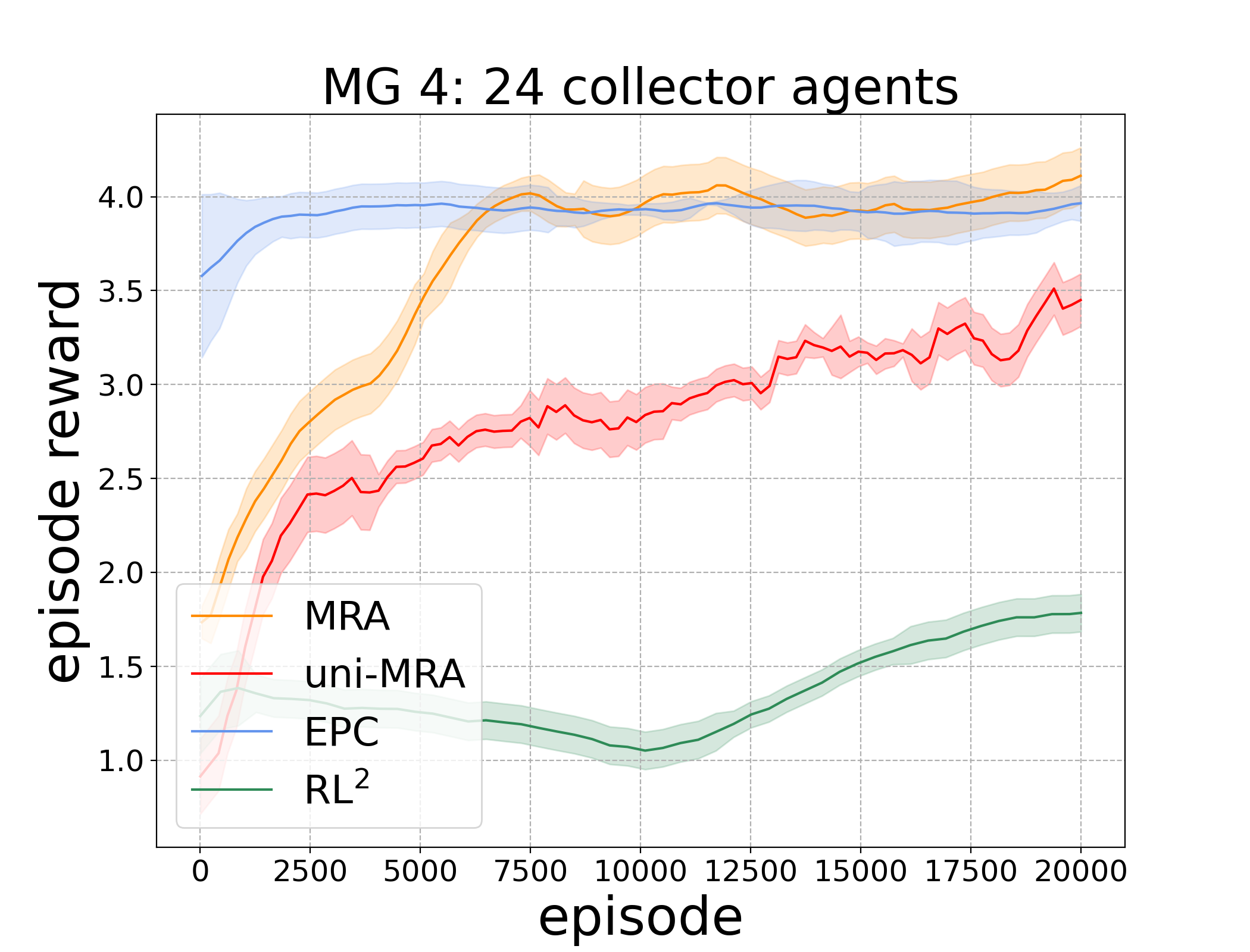}\\
    \end{minipage}%
}\\
\vspace{-0.25cm}
\centering
\subfigure{
    \begin{minipage}[t]{0.25\linewidth}
        \centering
        \includegraphics[width=1.7in]{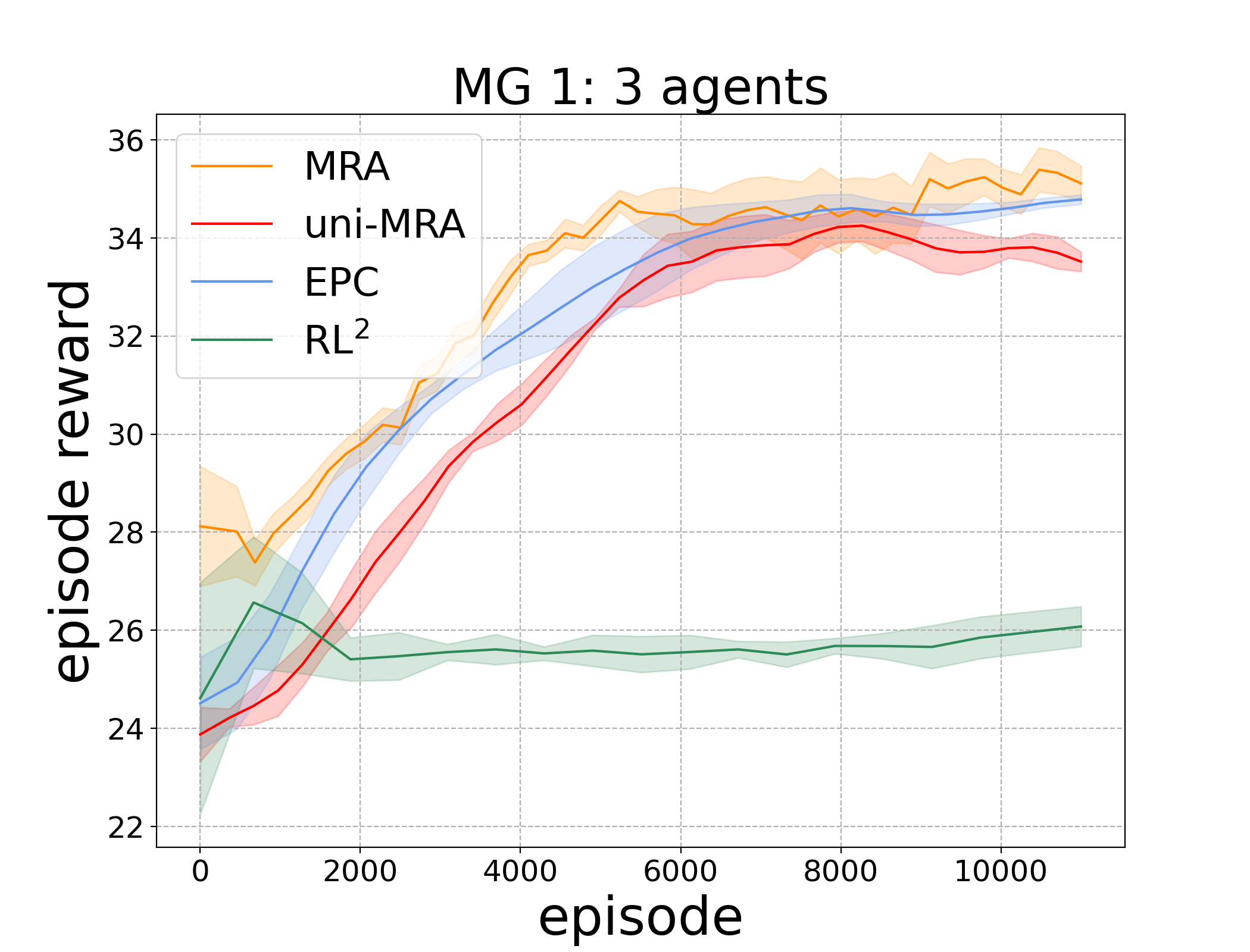}\\
    \end{minipage}%
}%
\subfigure{
    \begin{minipage}[t]{0.25\linewidth}
        \centering
        \includegraphics[width=1.7in]{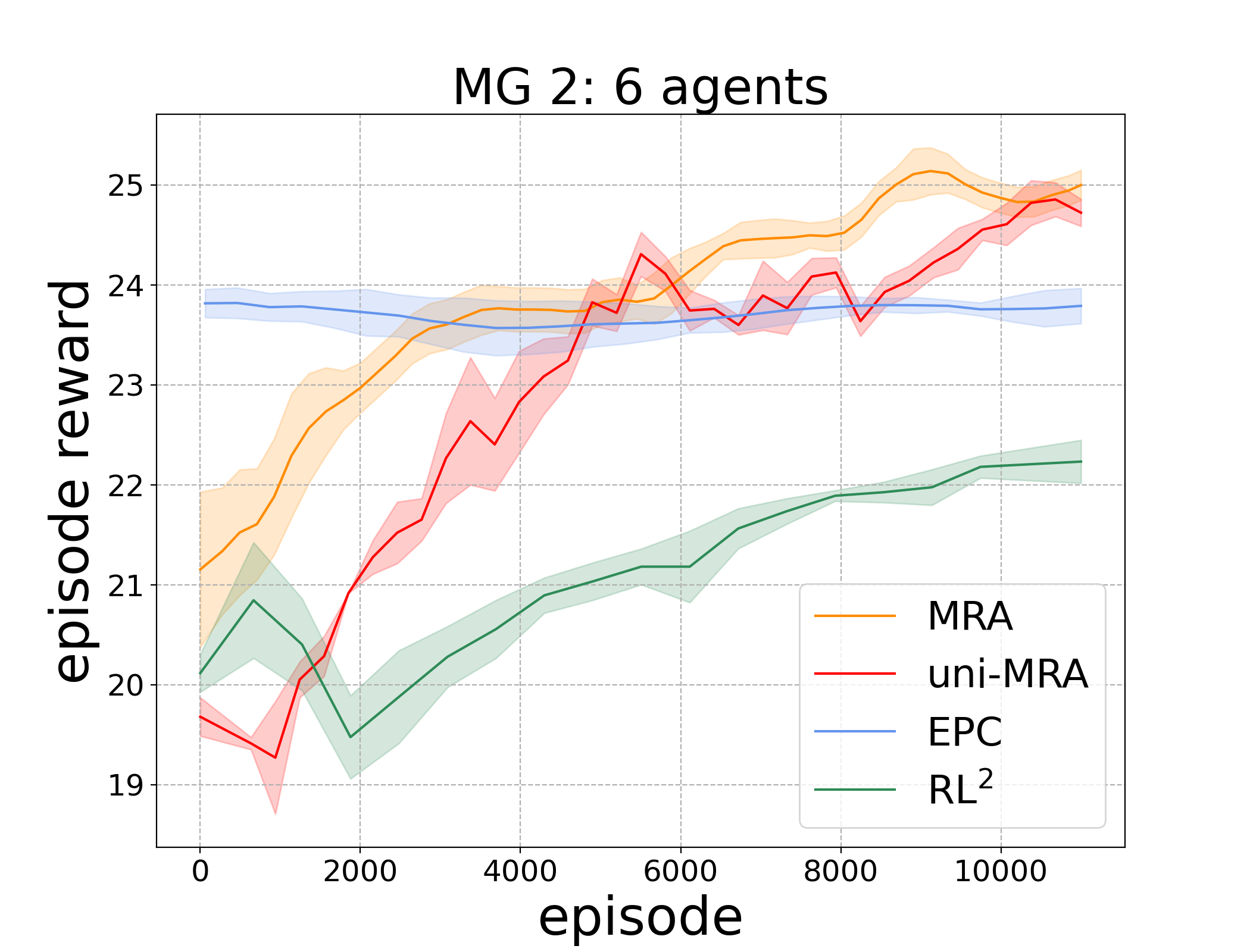}\\
    \end{minipage}%
}%
 \subfigure{
    \begin{minipage}[t]{0.25\linewidth}
        \centering
        \includegraphics[width=1.7in]{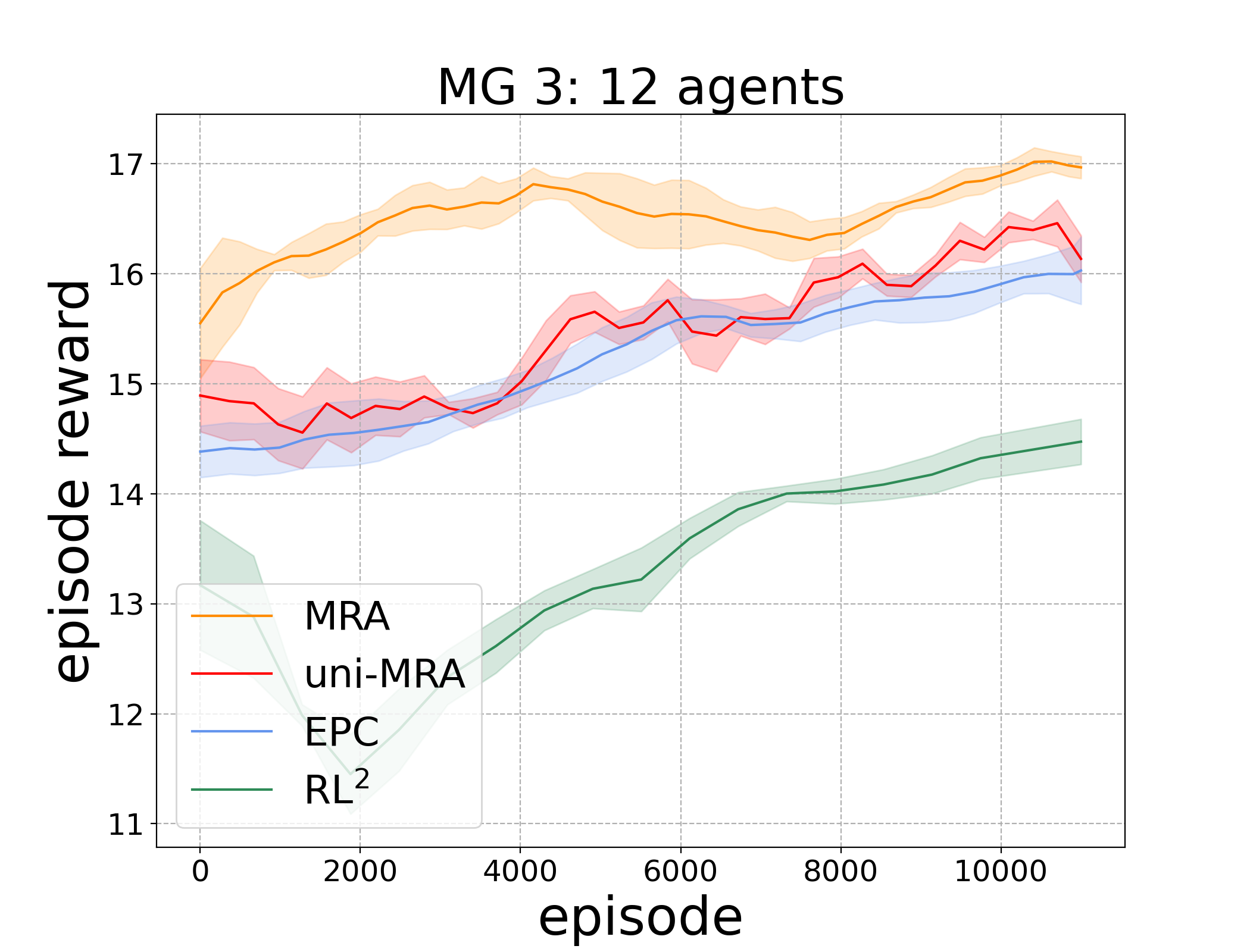}\\
    \end{minipage}%
}%
 \subfigure{
    \begin{minipage}[t]{0.25\linewidth}
        \centering
        \includegraphics[width=1.7in]{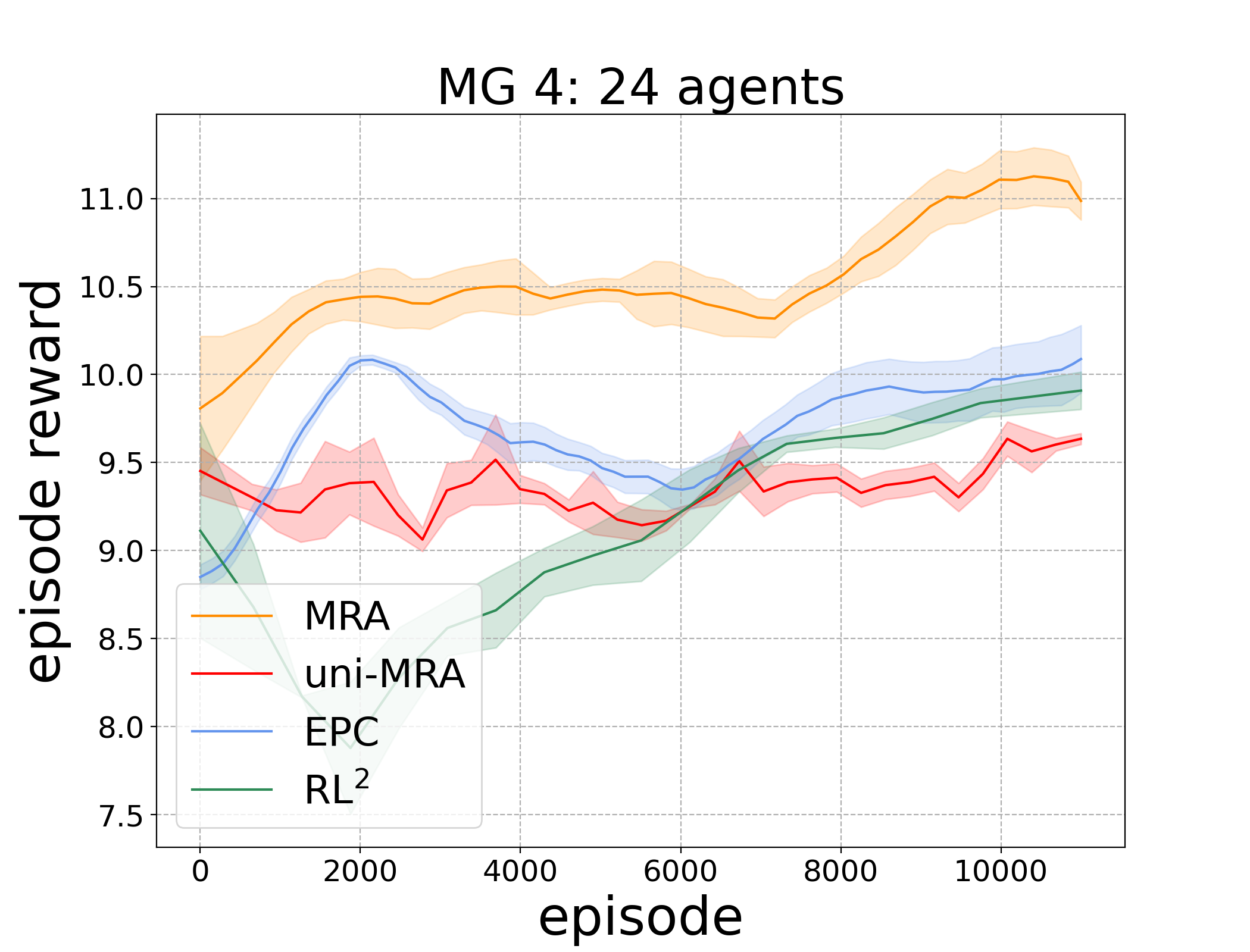}\\
    \end{minipage}%
}%
\vspace{-0.25cm}
\caption{The training curves in the Markov Games constructed by
varying the number of agents from the underlying environments. \textbf{First row:} Treasure collection environment; \textbf{Second row:} Resource occupation environment. In both environments, the total numbers of training MGs are $4$, where the numbers of agents are $3$, $6$, $12$, and $24$, respectively.\vspace{-0.35cm}}
\label{fig:add}
\end{figure*}

Due to the discrepancy between effective policies in different MGs, the game-common strategic knowledge is not well exploited by EPC. $\text{RL}^2$ agents are also observed to perform poorly in some MGs, which verifies the benefits of \textit{explicitly} modeling the game-common knowledge and game-specific knowledge when the number of agents varies. Since no game-specific information is conditioned in uni-MRA, we observe that some MGs dominate others.

\subsection{Generalization Evaluation}
If agents with the learned policies can both (1) adapt better and faster; and (2) perform well in novel (or unseen) evaluation games in a zero-shot manner, then the algorithm is considered to generalize well. In the following parts, we test the generalizability of MRA and other baseline methods based on these two metrics.

\textbf{Adaptation Ability:} We first show that MRA adapts faster and better compared with \textbf{EPC}, \textbf{$\text{RL}^2$}, \textbf{MADDPG}, and \textbf{MAAC} \citep{iqbal2018actor}. Here, MADDPG and MAAC are trained from scratch, while MRA, EPC, and $\text{RL}^2$ are fine-tuned from the parameters trained in multiple MGs as described in Section \ref{performance_multi}. The results are shown in Figure \ref{transfer}.

\begin{figure*}[htbp]
\centering
\subfigure[Sparse rewards.]{
\centering
\label{sparse}
\includegraphics[width=1.67in]{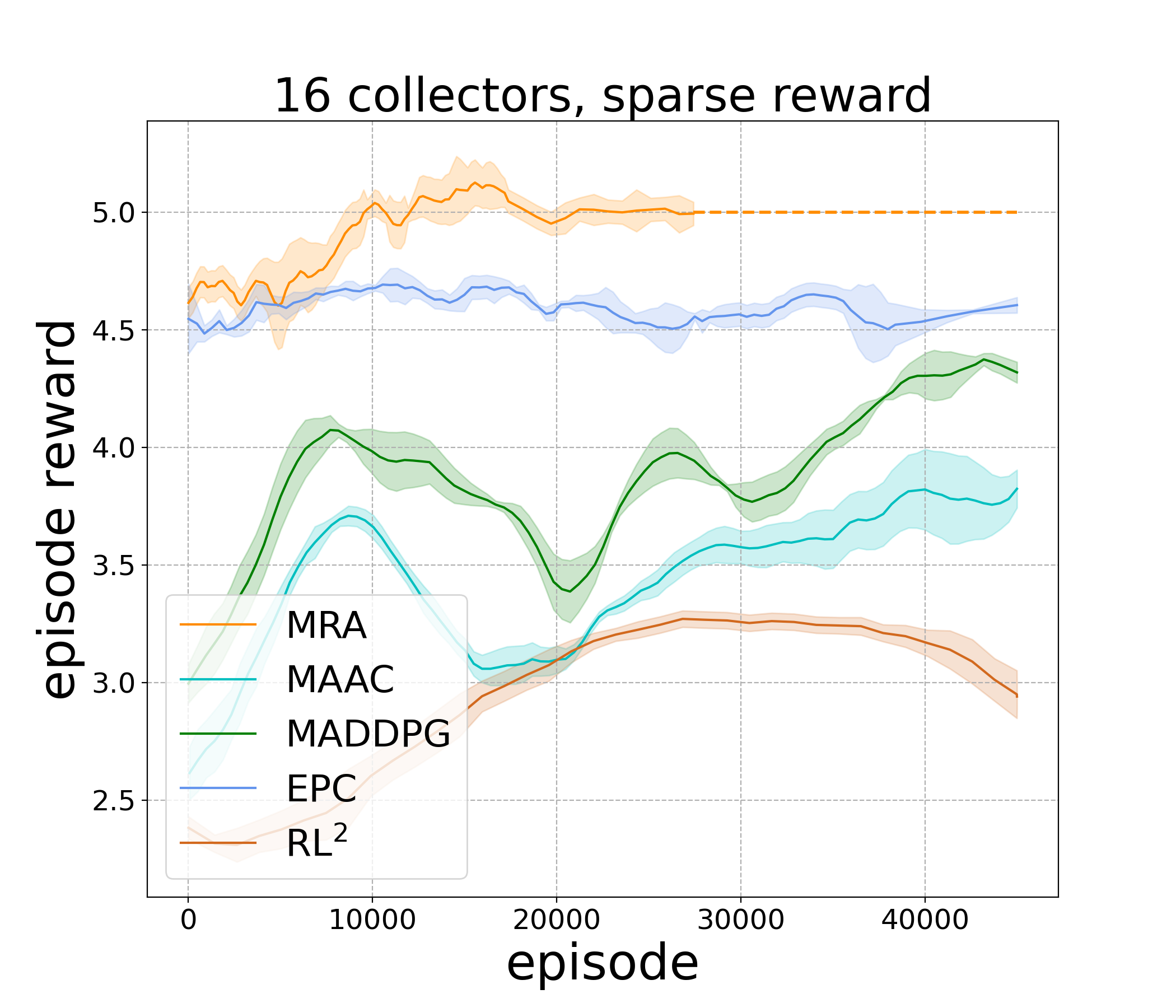}
}
\hspace{-0.2in}
\subfigure[Large population.]{
\centering
\label{complex}
\includegraphics[width=1.67in]{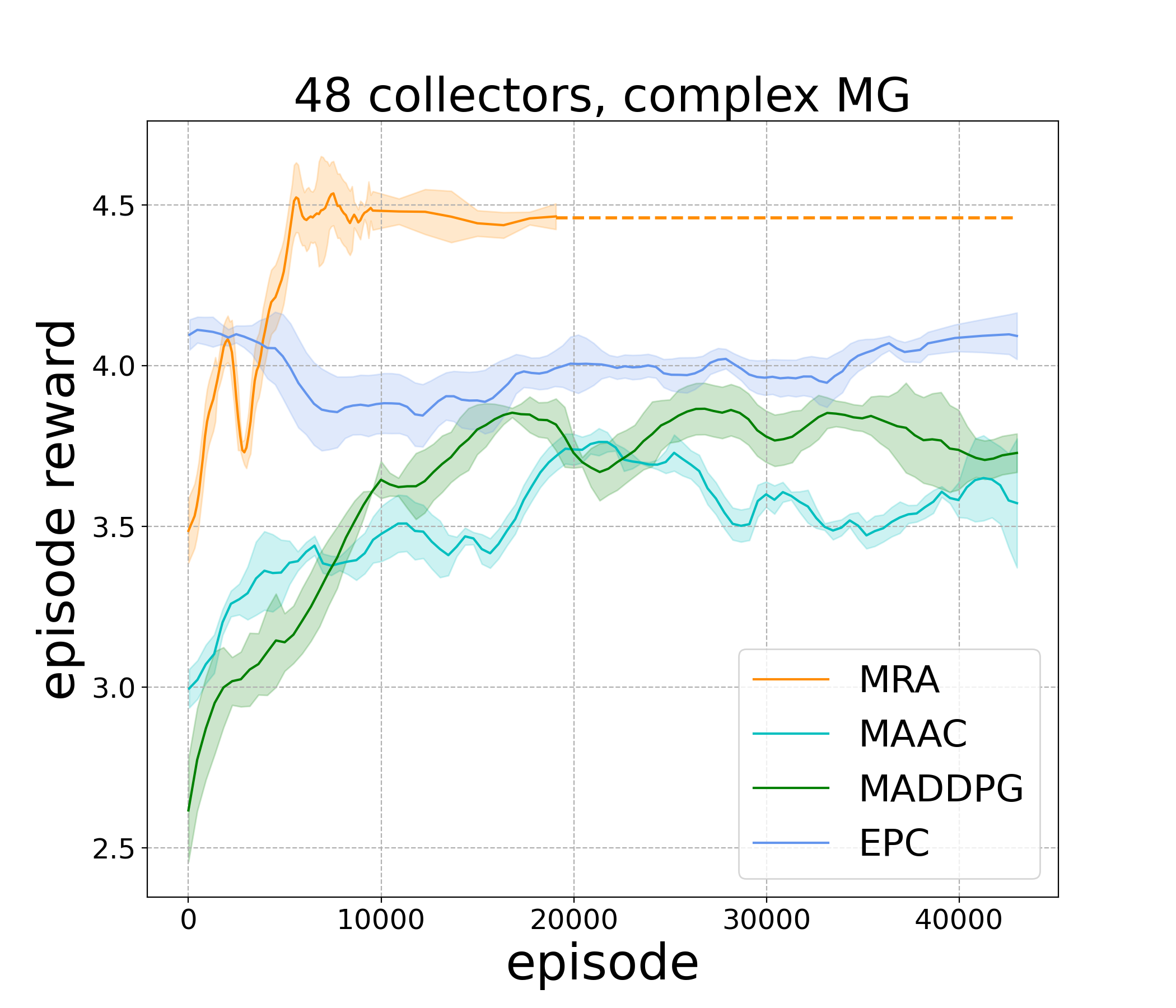}
}
\hspace{-0.2in}
\subfigure[Imbalanced competitive agents.]{
\label{transfer_pac}
\includegraphics[width=1.67in]{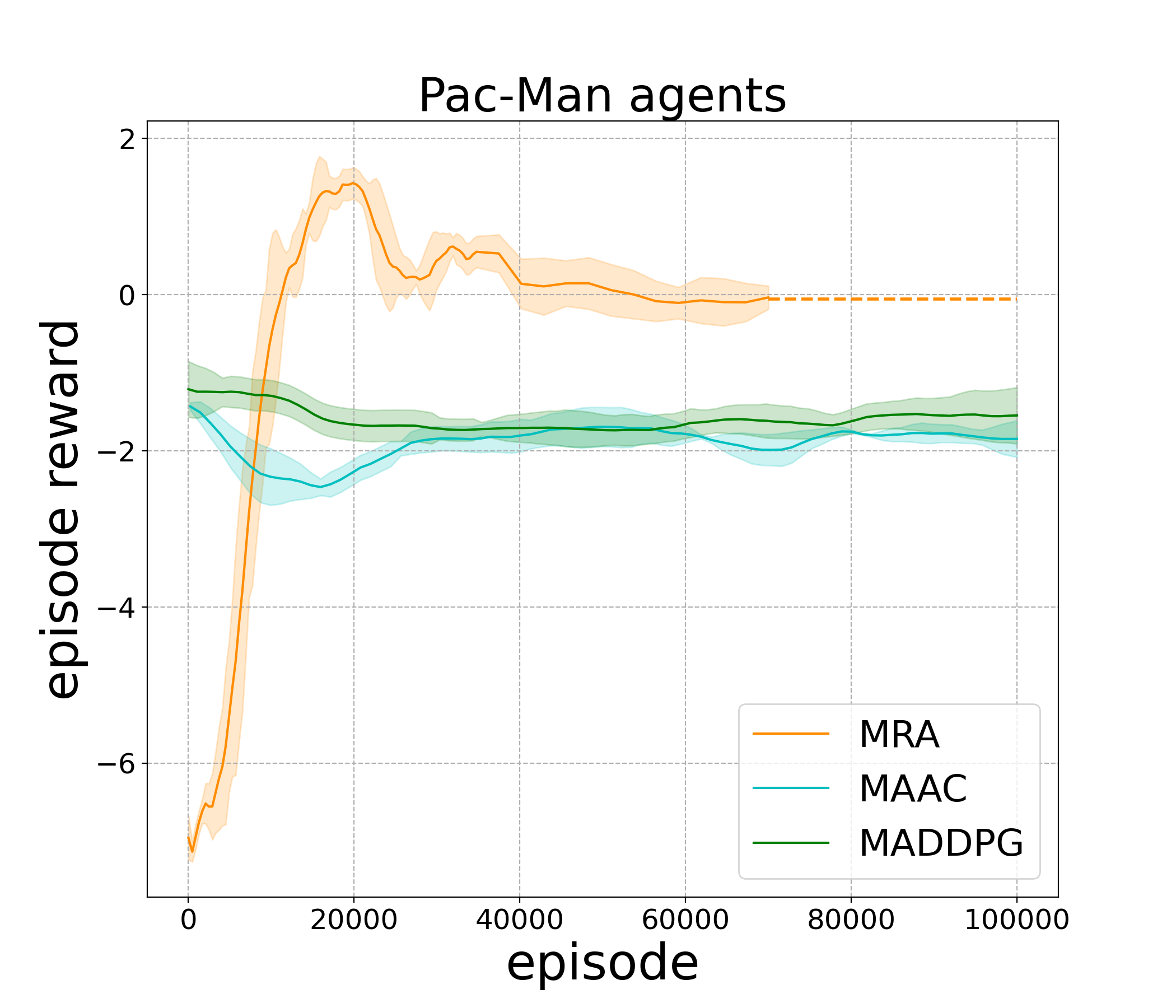}
\hspace{-0.2in}
\includegraphics[width=1.67in]{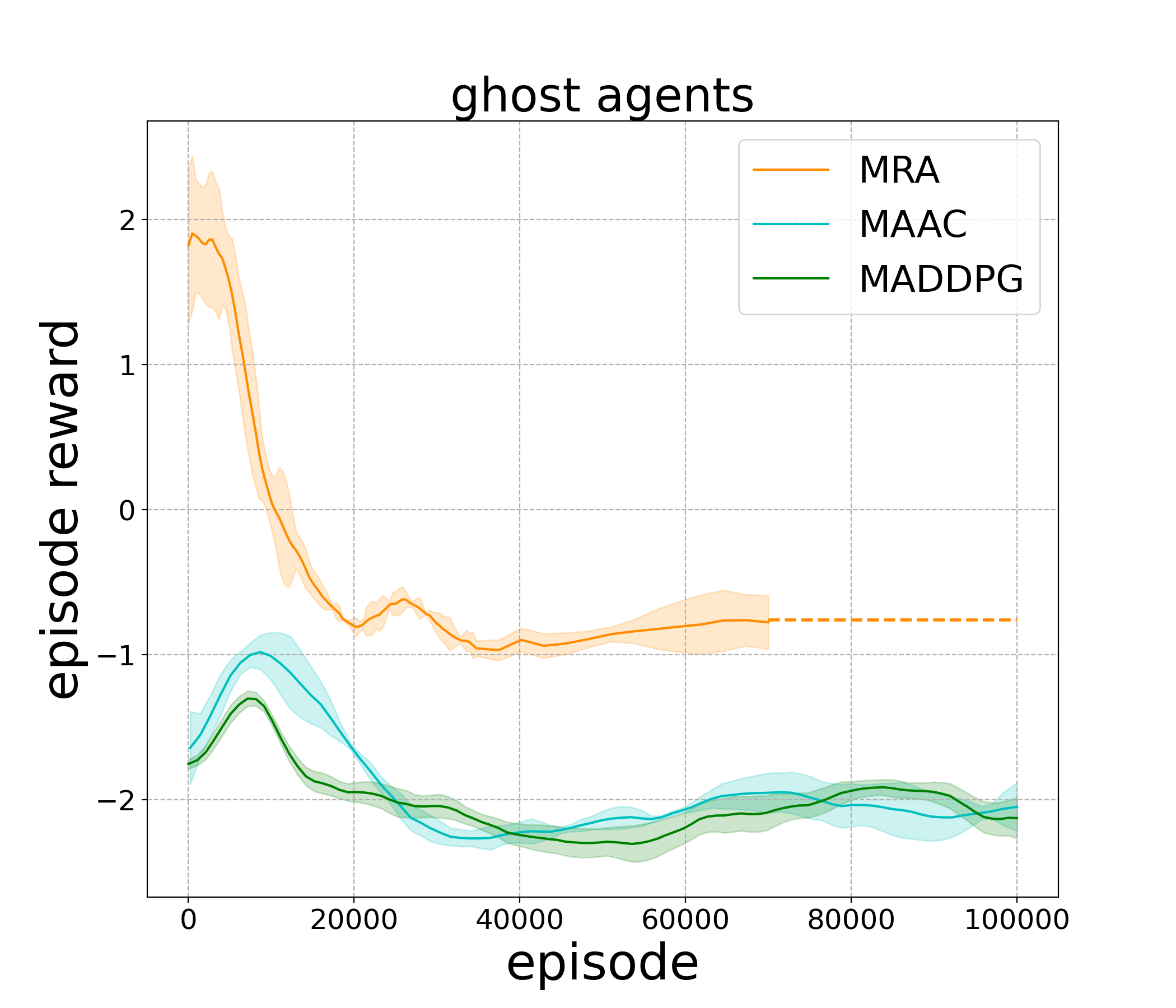}
}
\vspace{-0.29cm}
\setlength{\belowcaptionskip}{-10pt}
\caption{Adaptation performance comparison. \textbf{(a):} Sparse reward MG in the treasure collection environment;  \textbf{(b):} Complex MG with a large number of agents in the treasure collection environment;  \textbf{(c):} Imbalanced Pac-Man and ghost agents: $2$ Pac-Man vs $8$ ghosts, where the random exploration bottleneck prevents the agents to learn useful strategies due to the quickly terminated episodes and the lack of informative training signals. The dashed lines represent the asymptotic performance at convergence.\vspace{-0.1cm}}
\label{transfer}
\end{figure*}

Our first observation is that the game-common knowledge can provide agents with a good policy initialization and thus overcome the random exploration bottleneck. For example, random exploration can happen in the Pacman-like world where the ghost agents populationally dominate the game. When this is the case, the episode will quickly terminate since the Pac-Man agents will soon be killed by the ghost agents. As a result of the lack of informative training signals, both the Pac-Man and ghost agents will end up with almost random behaviors. This is evidenced by the poor performance of the MAAC and MADDPG algorithms in Figure \ref{transfer_pac}, where $8$ ghost agents are chasing $2$ Pac-Man agents. In the contrast, with the extracted common knowledge guiding the Pac-Man to take ghost-avoidance actions, the MRA agents can easily learn to accomplish the task.

Besides, in Figure \ref{sparse} when reward shaping is removed, i.e., in the sparse reward setting, MRA agents are able to effectively adapt to policies that have higher asymptotic performance in fewer episodes, compared to other baselines. We also show in Figure \ref{complex} that the complexity brought by the large population, such as $48$, can be successfully handled by our MRA algorithm. These results verify the generalization advantage of the meta-representation in MRA compared with EPC, which only transfers knowledge in the \textit{training} Markov Games.

\begin{wrapfigure}{r}{0.5\textwidth}
    \centering
        \vspace{-0.45cm}
    \includegraphics[width=1.0\linewidth]{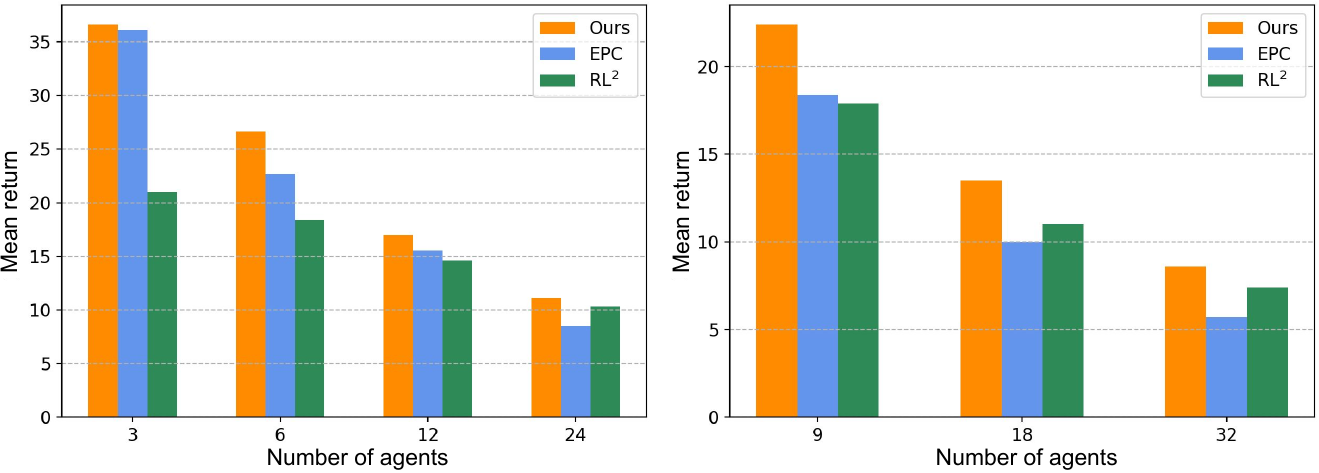}\vspace{-0.18cm}
\caption{\textbf{Left:} Evaluation in the training MGs. \textbf{Right:} Zero-shot transfer to unseen evaluation MGs.}
    \label{direct}
    \vspace{-0.5cm}
\end{wrapfigure}
\textbf{Zero-Shot Transfer:} MRA also has better generalizability compared to EPC and $\text{RL}^2$. In Figure \ref{direct}, we report the ability of MRA to represent different policies in multiple training games and zero-shot transfer to evaluation games in the resource occupation environment. The performance of MRA is calculated by taking expectations over the latent $z$. We note that the return of MRA will become higher if enumerated trial-and-error is performed, \ie, choose the best policy mode among various latent samples $z$. Besides, we also provide additional experimental results and ablation studies in Appendix \ref{add_exp}.

\section{Conclusion \& Discussions}
\label{conc}
In this paper, we propose meta representations for agents (MRA) that can generalize in Markov Games with varying populations. With latent variable policies and relational representations, the diverse strategic modes are captured. As an approximation to a theoretically justified objective, MRA effectively discovers the underlying strategic structures in the games that facilitate generalizable knowledge learning. Experimental results also verify the benefits of MRA.

Our work also opens several new problems. Theorem \ref{theorem1} requires the computation of $\varsigma$ and $\epsilon$ as well as an extremely large latent space, both of which are impractical. Although approximations that MRA makes are reasonable, obtaining optimal $\boldsymbol{\Pi}^*$ will not always be guaranteed. Possible future investigations include bounding $\varsigma$ by imposing restrictions on the evaluation MG set, or enlarging the latent space size by \eg, adopting continuous latent variables.

With role-symmetric game settings, MRA has benefits in many research problems, including dealing with population complexity, overcoming the multi-agent random exploration bottleneck, and adapting faster with the meta-represented agents. A fruitful avenue for future work is to augment MRA by \eg, adapting roles \citep{wang2020multi}, to apply to other game settings. Besides, achieving NE may not indicate the global optimality in general-sum MGs. Therefore, we would like to explore different metrics such as social optimum and correlated equilibrium as future work, which may require introducing the definition of distances established for these solution concepts.

For population-varying MGs, we model game-specific knowledge as strategic relationship. Although it may lose the universality in broader scopes compared with general meta-RL algorithms, we hope the idea of explicit strategic knowledge modeling can inspire algorithms that adjust with the task.


\bibliography{collas2023_conference}
\bibliographystyle{collas2023_conference}

\newpage
\appendix
\section{Proofs}
\label{proof_label}
\subsection{Proof Sketch of Theorem \ref{theorem1}}
\label{proof_sketch}
Theorem \ref{theorem1} can be proved by establishing the following lemmas.
\begin{lemma}
\label{l1}
Define the distance $\kappa(\pi^i)$ between policy $\pi^i$ and the best response $\pi^{*i}$ in the action space as:
$$\kappa(\pi^i) = \left\lVert{\left\lVert{\pi^{*i}(a\mid s)-\pi^i(a\mid s)}\right\rVert_{a,1}}\right\rVert_{s,\infty}.$$
For any joint strategy $\boldsymbol{\pi}$, the \textsc{NashConv} $\mathcal{D}_m(\boldsymbol{\pi})$ is bounded by: \newline
\begin{align*}
\mathcal{D}_m(\boldsymbol{\pi})\leq 
\left(\frac{\gamma \iota_m}{(1-\gamma)^2}+\frac{1}{1-\gamma}\right)\left\lVert{\kappa(\pi^i)}\right\rVert_{i,1}.
\end{align*}
\end{lemma}

Lemma \ref{l1} builds a bridge between the action-space distance and the value-space distance $\mathcal{D}_m(\boldsymbol{\pi})$. Then the distance $\varsigma$ between $\mdp$ and $\mdp'$ can be represented in form of $\kappa(\pi^i)$ for some specific index $i$. Intuitively, if $\epsilon$ is larger than a certain threshold, the $\epsilon$-range joint strategy set $\boldsymbol{\hat{\Pi}}$ is also large enough to contain all the strategies that achieve NE in evaluation games. Formally, we provide Lemma \ref{prop1}.

\begin{lemma}\label{prop1}
For the training MG set $\mdp$ and the evaluation MG set $\mdp'$, if $\epsilon$ satisfies 
\begin{equation}
\label{eps_con}
\begin{aligned}
\epsilon\geq \varsigma - \min_{\iota_m,\iota_{m'}}\frac{\varsigma\gamma\left(\iota_{m'}-\iota_{m}\right)}{\gamma \iota_{m'}+1-\gamma},
\end{aligned}
\end{equation}
then for every evaluation Markov Game $m'\in \mdp'$, the joint strategy that achieves Nash Equilibrium in $m'$ is guaranteed to be contained in the $\epsilon$-range joint strategy set $\boldsymbol{\hat{\Pi}}$.
\end{lemma}

We then show how Lemma \ref{prop1} can be utilized to obtain the objective in  \eqref{objective}. The first step is to find an equivalence with optimization objective as formally stated in Lemma \ref{prop_in}.

\begin{lemma}\label{prop_in}
If $\lvert\boldsymbol{\Pi}_{\Theta}\rvert \geq \lvert\boldsymbol{\hat{\Pi}}\rvert$ and $\epsilon$ satisfies $\epsilon\geq \varsigma - \min_{\iota_m,\iota_{m'}}\frac{\varsigma\gamma\left(\iota_{m'}-\iota_{m}\right)}{\gamma \iota_{m'}+1-\gamma}$, then the optimal $\boldsymbol{\Pi}$ that maximizes the objective:
\begin{equation}
\begin{aligned}
\label{maxmin}
\mathcal{L}(\boldsymbol{\Pi}) =\max_{\boldsymbol{\pi}\sim\boldsymbol{\Pi}}\min_{\boldsymbol{\hat{\pi}}\sim\boldsymbol{\hat{\Pi}}}\mathbb{E}_{\substack{\boldsymbol{a}\sim \boldsymbol{\hat{\pi}},\boldsymbol{o}}}\left [\log \boldsymbol{\pi}(\boldsymbol{a}\mid\boldsymbol{o}) \right ],
\end{aligned}
\end{equation}
for every evaluation Markov Game $m'\in \mdp'$, there exists a joint strategy $\boldsymbol{\pi}\in \boldsymbol{\Pi}_{\Theta^*}$ that reaches Nash Equilibrium (\ie, $\boldsymbol{\Pi}_{\Theta^*}=\boldsymbol{\Pi}^*$ satisfies equation \eqref{primal}).
\end{lemma}

The next step is to build the relationship between  \eqref{maxmin} and the mutual information objective in  \eqref{objective}. If $\boldsymbol{\pi} \in \boldsymbol{\hat{\Pi}}$, we notice that $\max\limits_{\boldsymbol{\pi}}\mathbb{E}_{\substack{\boldsymbol{\hat{\pi}}}}\left [\log \boldsymbol{\pi} \right ] = - \mathcal{H}(\boldsymbol{\hat{\pi}})$. Thus, Theorem \ref{theorem1} can be proved by leveraging the close connection between entropy and mutual information.

In the sequel, we provide the proofs of the above three lemmas in \ref{A1}, \ref{A2}, and \ref{A3}, respectively. The proof of Theorem \ref{theorem1} in \ref{A4} is built upon these lemmas.

\subsection{Proof of Lemma \ref{l1}}
\label{A1}

Before giving the proof of Lemma \ref{l1}, we first provide a useful lemma of the Markov state transition operator $\mathcal{P}_m^{\pi^{i},\boldsymbol{\pi^{\textrm{-}i}}}$. The following Lemma \ref{liulemma} is modified from \citep{liu2021policy} Lemma E.1, which is originally presented for fictitious self-play. We modify the lemma and the proof to be suitable for the concerned multi-agent general-sum Markov Game setting with our notations.
\begin{lemma}
\label{liulemma}
(Liu et al., 2021). The Markov state transition operator satisfies:\newline
$$
\left\lVert{\sum_t \gamma^t \left[ \left(\mathcal{P}_m^{\pi^{*i},\boldsymbol{\pi^{\textrm{-}i}}}\right)^t - \left(\mathcal{P}_m^{\pi^{i},\boldsymbol{\pi^{\textrm{-}i}}}\right)^t \right]}\right\rVert_{op} \leq \frac{\gamma}{\left(1 - \gamma\right)^2}\left\lVert{ \mathcal{P}_m^{\pi^{*i},\boldsymbol{\pi^{\textrm{-}i}}} - \mathcal{P}_m^{\pi^{i},\boldsymbol{\pi^{\textrm{-}i}}}}\right\rVert_{op}.
$$
\end{lemma}
\begin{proof}

As a first step, we study the operator $\mathcal{P}_m^{\pi^{i},\boldsymbol{\pi^{\textrm{-}i}}}$ and obtain the following results:
\begin{equation}
\label{eq_first_rec}
    \begin{aligned}
    &\left\lVert{ \left(\mathcal{P}_m^{\pi^{*i},\boldsymbol{\pi^{\textrm{-}i}}}\right)^t - \left(\mathcal{P}_m^{\pi^{i},\boldsymbol{\pi^{\textrm{-}i}}}\right)^t }\right\rVert_{op}\\
    &= \left\lVert{ \left(\mathcal{P}_m^{\pi^{*i},\boldsymbol{\pi^{\textrm{-}i}}}\right)^{t-1}  \left(\mathcal{P}_m^{\pi^{*i},\boldsymbol{\pi^{\textrm{-}i}}} - \mathcal{P}_m^{\pi^{i},\boldsymbol{\pi^{\textrm{-}i}}} \right) + \left( \left(\mathcal{P}_m^{\pi^{*i},\boldsymbol{\pi^{\textrm{-}i}}}\right)^{t-1} - \left(\mathcal{P}_m^{\pi^{i},\boldsymbol{\pi^{\textrm{-}i}}}\right)^{t-1} \right)  \mathcal{P}_m^{\pi^{i},\boldsymbol{\pi^{\textrm{-}i}}}}\right\rVert_{op}\\
    &\leq \left\lVert{ \left(\mathcal{P}_m^{\pi^{*i},\boldsymbol{\pi^{\textrm{-}i}}}\right)^{t-1}  \left(\mathcal{P}^{\pi^{*i},\boldsymbol{\pi^{\textrm{-}i}}} - \mathcal{P}_m^{\pi^{i},\boldsymbol{\pi^{\textrm{-}i}}} \right)}\right\rVert_{op} + \left\lVert{\left( \left(\mathcal{P}_m^{\pi^{*i},\boldsymbol{\pi^{\textrm{-}i}}}\right)^{t-1} - \left(\mathcal{P}_m^{\pi^{i},\boldsymbol{\pi^{\textrm{-}i}}}\right)^{t-1} \right)  \mathcal{P}_m^{\pi^{i},\boldsymbol{\pi^{\textrm{-}i}}}}\right\rVert_{op}\\
    &\leq \left\lVert{ \mathcal{P}_m^{\pi^{*i},\boldsymbol{\pi^{\textrm{-}i}}} - \mathcal{P}_m^{\pi^{i},\boldsymbol{\pi^{\textrm{-}i}}}}\right\rVert_{op} + \left\lVert{ \left(\mathcal{P}_m^{\pi^{*i},\boldsymbol{\pi^{\textrm{-}i}}}\right)^{t-1} - \left(\mathcal{P}_m^{\pi^{i},\boldsymbol{\pi^{\textrm{-}i}}}\right)^{t-1} }\right\rVert_{op},
        \end{aligned}
\end{equation}
where the equality follows from basic algebra, the first inequality follows from the triangle inequality, and the last inequality holds since $\mathcal{P}_m^{\pi^{*i},\boldsymbol{\pi^{\textrm{-}i}}} \leq 1$ \citep{lasota1998chaos}.

Recursively applying \eqref{eq_first_rec}, we obtain
\begin{align*}
    &\left\lVert{ \left(\mathcal{P}_m^{\pi^{*i},\boldsymbol{\pi^{\textrm{-}i}}}\right)^t - \left(\mathcal{P}_m^{\pi^{i},\boldsymbol{\pi^{\textrm{-}i}}}\right)^t }\right\rVert_{op}\\
    & \leq \left\lVert{ \mathcal{P}_m^{\pi^{*i},\boldsymbol{\pi^{\textrm{-}i}}} - \mathcal{P}_m^{\pi^{i},\boldsymbol{\pi^{\textrm{-}i}}}}\right\rVert_{op} +  \left\lVert{ \mathcal{P}_m^{\pi^{*i},\boldsymbol{\pi^{\textrm{-}i}}} - \mathcal{P}_m^{\pi^{i},\boldsymbol{\pi^{\textrm{-}i}}}}\right\rVert_{op} + \left\lVert{ \left(\mathcal{P}_m^{\pi^{*i},\boldsymbol{\pi^{\textrm{-}i}}}\right)^{t-2} - \left(\mathcal{P}_m^{\pi^{i},\boldsymbol{\pi^{\textrm{-}i}}}\right)^{t-2} }\right\rVert_{op}\\
    & \leq t \cdot \left\lVert{ \mathcal{P}_m^{\pi^{*i},\boldsymbol{\pi^{\textrm{-}i}}} - \mathcal{P}_m^{\pi^{i},\boldsymbol{\pi^{\textrm{-}i}}}}\right\rVert_{op}.
\end{align*}

Then we have by summation that
\begin{align*}
    \left\lVert{\sum_t \gamma^t \left[ \left(\mathcal{P}_m^{\pi^{*i},\boldsymbol{\pi^{\textrm{-}i}}}\right)^t - \left(\mathcal{P}_m^{\pi^{i},\boldsymbol{\pi^{\textrm{-}i}}}\right)^t \right]}\right\rVert_{op} &\leq \sum_t \gamma^t \left\lVert{ \left(\mathcal{P}_m^{\pi^{*i},\boldsymbol{\pi^{\textrm{-}i}}}\right)^t - \left(\mathcal{P}_m^{\pi^{i},\boldsymbol{\pi^{\textrm{-}i}}}\right)^t }\right\rVert_{op}\\
    &\leq \left(\sum_t t \gamma^t \right)\left\lVert{ \mathcal{P}_m^{\pi^{*i},\boldsymbol{\pi^{\textrm{-}i}}} - \mathcal{P}_m^{\pi^{i},\boldsymbol{\pi^{\textrm{-}i}}}}\right\rVert_{op}\\
    &= \frac{\gamma}{\left(1 - \gamma\right)^2}\left\lVert{ \mathcal{P}_m^{\pi^{*i},\boldsymbol{\pi^{\textrm{-}i}}} - \mathcal{P}_m^{\pi^{i},\boldsymbol{\pi^{\textrm{-}i}}}}\right\rVert_{op}.
\end{align*}

This concludes the proof of Lemma \ref{liulemma}.
\end{proof}


Now we are ready to prove Lemma \ref{l1}. 

\begin{proof}
To begin, we denote by $\rho_{s,m}^{\pi^{*i},\boldsymbol{\pi^{\textrm{-}i}}}$ the state visitation measure of the joint strategy $(\pi^{*i},\boldsymbol{\pi^{\textrm{-}i}})$ in the Markov Game $m$, which is defined as follows:
\begin{align*}
    \rho_{s,m}^{\pi^{*i},\boldsymbol{\pi^{\textrm{-}i}}} &= \left(\mathcal{I} - \gamma \mathcal{P}_m^{\pi^i,\boldsymbol{\pi^{\textrm{-}i}}} \right)^{-1}  \delta_s\\
    &= \left( \sum_t \gamma^t \left(\mathcal{P}_m^{\pi^i,\boldsymbol{\pi^{\textrm{-}i}}}\right)^t \right)  \delta_s ,
\end{align*}
where $\delta_s$ is a Dirac delta function.

By converting the value of strategy to the integration of
reward over state measure, it holds that 
\begin{equation}
\label{init}
\begin{aligned}
&v^{*i,m}_{\boldsymbol{\pi}^{\textrm{-}i}}(s) - v^{i,m}_{\boldsymbol{\pi}}(s)\\
&= \mathbb{E}_{s'\sim \rho_{s,m}^{\pi^{*i},\boldsymbol{\pi^{\textrm{-}i}}}}\left[ \mathbb{E}_{\pi^{*i},\boldsymbol{\pi^{\textrm{-}i}}} r^i_m(s',\boldsymbol{a}) \right] - \mathbb{E}_{s'\sim \rho_{s,m}^{\pi^{i},\boldsymbol{\pi^{\textrm{-}i}}}}\left[ \mathbb{E}_{\pi^{i},\boldsymbol{\pi^{\textrm{-}i}}} r^i_m(s',\boldsymbol{a}) \right]\\
&= \mathbb{E}_{s'\sim \rho_{s,m}^{\pi^{*i},\boldsymbol{\pi^{\textrm{-}i}}}}\left[ \mathbb{E}_{\pi^{i},\boldsymbol{\pi^{\textrm{-}i}}} r^i_m(s',\boldsymbol{a}) \right] - \mathbb{E}_{s'\sim \rho_{s,m}^{\pi^{i},\boldsymbol{\pi^{\textrm{-}i}}}}\left[ \mathbb{E}_{\pi^{i},\boldsymbol{\pi^{\textrm{-}i}}} r^i_m(s',\boldsymbol{a}) \right] \\
&\qquad+ \mathbb{E}_{s'\sim \rho_{s,m}^{\pi^{*i},\boldsymbol{\pi^{\textrm{-}i}}}}\left[ \mathbb{E}_{\pi^{*i},\boldsymbol{\pi^{\textrm{-}i}}} r^i_m(s',\boldsymbol{a}) - \mathbb{E}_{\pi^{i},\boldsymbol{\pi^{\textrm{-}i}}} r^i_m(s',\boldsymbol{a}) \right] \\
&\leq \left\lVert{\rho_{s,m}^{\pi^{*i},\boldsymbol{\pi^{\textrm{-}i}}} - \rho_{s,m}^{\pi^{i},\boldsymbol{\pi^{\textrm{-}i}}}}\right\rVert_{s,1} + \mathbb{E}_{s'\sim \rho_{s,m}^{\pi^{*i},\boldsymbol{\pi^{\textrm{-}i}}}}\left[ \mathbb{E}_{\pi^{*i},\boldsymbol{\pi^{\textrm{-}i}}} r^i_m(s',\boldsymbol{a}) - \mathbb{E}_{\pi^{i},\boldsymbol{\pi^{\textrm{-}i}}} r^i_m(s',\boldsymbol{a}) \right],
\end{aligned}
\end{equation}
where the inequality follows from the definition of the $\mathcal{L}_1$-norm over the state space.

We then bound the resulting two terms separately. For the first term, we have from Lemma \ref{liulemma} that:
$$
\left\lVert{\sum_t \gamma^t \left[ \left(\mathcal{P}_m^{\pi^{*i},\boldsymbol{\pi^{\textrm{-}i}}}\right)^t - \left(\mathcal{P}_m^{\pi^{i},\boldsymbol{\pi^{\textrm{-}i}}}\right)^t \right]}\right\rVert_{op} \leq \frac{\gamma}{\left(1 - \gamma\right)^2}\left\lVert{ \mathcal{P}_m^{\pi^{*i},\boldsymbol{\pi^{\textrm{-}i}}} - \mathcal{P}_m^{\pi^{i},\boldsymbol{\pi^{\textrm{-}i}}}}\right\rVert_{op}
$$

Following from the definition of $\iota_m$, we have
\begin{align*}
    \left\lVert{\sum_t \gamma^t \left[ \left(\mathcal{P}_m^{\pi^{*i},\boldsymbol{\pi^{\textrm{-}i}}}\right)^t - \left(\mathcal{P}_m^{\pi^{i},\boldsymbol{\pi^{\textrm{-}i}}}\right)^t \right]}\right\rVert_{op}&\leq \frac{\gamma}{\left(1 - \gamma\right)^2}\left\lVert{ \mathcal{P}_m^{\pi^{*i},\boldsymbol{\pi^{\textrm{-}i}}} - \mathcal{P}_m^{\pi^{i},\boldsymbol{\pi^{\textrm{-}i}}}}\right\rVert_{op}\\
    &\leq \frac{\gamma \iota_m}{\left(1 - \gamma\right)^2}\kappa(\pi^i)
\end{align*}

Thus, it holds that
\begin{equation}
\label{first_bound}
    \begin{aligned}
    \left\lVert{\rho_{s,m}^{\pi^{*i},\boldsymbol{\pi^{\textrm{-}i}}} - \rho_{s,m}^{\pi^{i},\boldsymbol{\pi^{\textrm{-}i}}}}\right\rVert_{s,1}
    &= \left\lVert{\left( \sum_t \gamma^t \left(\mathcal{P}_m^{\pi^{*i},\boldsymbol{\pi^{\textrm{-}i}}}\right)^t \right)  \delta_s   -   \left( \sum_t \gamma^t \left(\mathcal{P}_m^{\pi^i,\boldsymbol{\pi^{\textrm{-}i}}}\right)^t \right)  \delta_s}\right\rVert_{s,1}\\
&\leq  \left\lVert{\sum_t \gamma^t \left[ \left(\mathcal{P}_m^{\pi^{*i},\boldsymbol{\pi^{\textrm{-}i}}}\right)^t - \left(\mathcal{P}_m^{\pi^{i},\boldsymbol{\pi^{\textrm{-}i}}}\right)^t \right]}\right\rVert_{op} \cdot \left\lVert{\delta_s}\right\rVert_{s,1}\\
&\leq \frac{\gamma \iota_m}{\left(1 - \gamma\right)^2}\kappa(\pi^i),
    \end{aligned}
\end{equation}
where the last inequality holds due to Assumption \ref{lip_assumption}.

Besides, the second term on the right-hand side of \eqref{init} satisfies
\begin{equation}
\label{second_bound}
    \begin{aligned}
    &\left\lVert{\mathbb{E}_{s'\sim \rho_{s,m}^{\pi^{*i},\boldsymbol{\pi^{\textrm{-}i}}}}\left[ \mathbb{E}_{\pi^{*i},\boldsymbol{\pi^{\textrm{-}i}}} r^i_m(s',\boldsymbol{a}) - \mathbb{E}_{\pi^{i},\boldsymbol{\pi^{\textrm{-}i}}} r^i_m(s',\boldsymbol{a}) \right]}\right\rVert_{s,\infty}\\
    &\leq \left\lVert{\mathbb{E}_{s'\sim \rho_{s,m}^{\pi^{*i},\boldsymbol{\pi^{\textrm{-}i}}}}\left[   \left\lVert{\pi^{*i}(\cdot\mid s)-\pi^i(\cdot\mid s)}\right\rVert_{1}\right]}\right\rVert_{s,\infty}\\
    &\leq \kappa(\pi^i)\left(\sum_t\gamma^t \right)= \frac{\kappa(\pi^i)}{1 - \gamma},
    \end{aligned}
\end{equation}
where the second inequality follows from the definition of $\kappa(\pi^i)$.

Finally, combining \eqref{first_bound} and \eqref{second_bound}, we obtain the stated result for $\mathcal{D}_m(\boldsymbol{\pi})$ as follows:
\begin{align*}
\mathcal{D}_m(\boldsymbol{\pi})&=\left\lVert{\left\lVert{v^{*i,m}_{\boldsymbol{\pi}^{\textrm{-}i}}(s) - v^{i,m}_{\boldsymbol{\pi}}(s)}\right\rVert_{s,\infty}}\right\rVert_{i,1}\\
&\leq \left\lVert{\left\lVert{\left\lVert{\rho_{s,m}^{\pi^{*i},\boldsymbol{\pi^{\textrm{-}i}}} - \rho_{s,m}^{\pi^{i},\boldsymbol{\pi^{\textrm{-}i}}}}\right\rVert_{s,1} + \mathbb{E}_{s'\sim \rho_{s,m}^{\pi^{*i},\boldsymbol{\pi^{\textrm{-}i}}}}\left[ \mathbb{E}_{\pi^{*i},\boldsymbol{\pi^{\textrm{-}i}}} r^i_m(s',\boldsymbol{a}) - \mathbb{E}_{\pi^{i},\boldsymbol{\pi^{\textrm{-}i}}} r^i_m(s',\boldsymbol{a}) \right]}\right\rVert_{s,\infty}}\right\rVert_{i,1}\\
&\leq \left\lVert{\frac{\gamma \iota_m}{\left(1 - \gamma\right)^2}\kappa(\pi^i) + \frac{\kappa(\pi^i)}{1 - \gamma}}\right\rVert_{i,1}\\
&=\left(\frac{\gamma \iota_m}{(1-\gamma)^2}+\frac{1}{1-\gamma}\right)\left\lVert{\kappa(\pi^i)}\right\rVert_{i,1}
\end{align*}
\end{proof}

\subsection{Proof of Lemma \ref{prop1}}
\label{A2}

\begin{proof}
Definition \ref{def_distance_sets} describes the distance $\varsigma$ between the training MG set $\mdp$ and the evaluation MG set $\mdp'$. In the following, we provide an equivalent logic statement:

\begin{align*}
\forall m' \in \mdp', i\in\{1,\ldots,N_{m'}\}, \exists m\in\mdp, \boldsymbol{\pi}, \boldsymbol{\pi'}, i'\in h_{i,m},\\
\text{~s.t.~}  \mathcal{D}_{m}(\boldsymbol{\pi})=0, 
\mathcal{D}_{m'}(\boldsymbol{\pi'})=0, \forall \pi^i\in\boldsymbol{\pi}, \mathcal{D}_{m'}(\pi^{i'},\boldsymbol{\pi'^{\textrm{-}i}})\leq\varsigma.
\end{align*}

In an evaluation MG $\Tilde{m}'\in\mdp'$, let $i$ and $i'$ be the agent index defined as follows:

\begin{equation}
\label{i_d}
    \begin{aligned}
    i = \argmax_{i\in\{1,\ldots,N_{\Tilde{m}'}\}} \left[\min_{i'\in h_{i,m}} \mathcal{D}_{\Tilde{m}'}(\pi^{i'},\boldsymbol{\pi'^{\textrm{-}i}})\right], \text{~s.t.~}  \mathcal{D}_{m}(\boldsymbol{\pi})=0, 
\mathcal{D}_{\Tilde{m}'}(\boldsymbol{\pi'})=0
    \end{aligned}
\end{equation}

Intuitively, the above agent $i$ is the agent in $\Tilde{m}'$ that being replaced by the trained policy $\pi^{i'}$ in the policy set leads to the largest distance to the Nash Equilibrium. And agent $i'$ is the corresponding agent that $\pi^{i'}$ achieves an NE in a particular training MG.

With $i$ and $i'$ denied in \eqref{i_d}, the bound in Lemma \ref{l1} can be specified as follows:
\begin{align*}
\mathcal{D}_{\Tilde{m}'}(\pi^{i'},\boldsymbol{\pi'^{\textrm{-}i}})&=\left\lVert{\left\lVert{v^{*i,\Tilde{m}'}_{\boldsymbol{\pi'^{\textrm{-}i}}} - v^{i,\Tilde{m}'}_{\pi^{i'},\boldsymbol{\pi'^{\textrm{-}i}}}}\right\rVert_{s,\infty}}\right\rVert_{i,1}\\    &=\left\lVert{v^{i,\Tilde{m}'}_{\boldsymbol{\pi'}} - v^{i,\Tilde{m}'}_{\pi^{i'},\boldsymbol{\pi'^{\textrm{-}i}}}}\right\rVert_{s,\infty}\\
    &= \left\lVert{v^{i,\Tilde{m}'}_{\pi'^{i},\boldsymbol{\pi'^{\textrm{-}i}}} - v^{i,\Tilde{m}'}_{\pi^{i'},\boldsymbol{\pi'^{\textrm{-}i}}}}\right\rVert_{s,\infty}\\
    &\leq \frac{\gamma \iota_{\Tilde{m}'}}{\left(1 - \gamma\right)^2}\kappa(\pi^{i'}) + \frac{\kappa(\pi^{i'})}{1 - \gamma},
\end{align*}
where the second equality holds since $\mathcal{D}_{\Tilde{m}'}(\boldsymbol{\pi'})=0$, and the distance from the joint strategy $(\pi^{i'},\boldsymbol{\pi'^{\textrm{-}i}})$ to a Nash Equilibrium is equal to the distance to $\boldsymbol{\pi'}$. The last inequality holds due to the bound in Lemma \ref{l1} and the definition of $i$ and $i'$.

This implies that for any MG $\Tilde{m}'$, it holds that
\begin{equation}
\label{sin_u}
\max_{\substack{i\in\{1,\ldots,N_{\Tilde{m}'}\}}}\min_{ \substack{m\in\mdp,i'\in h_i\\ \boldsymbol{\pi}\in\{\boldsymbol{\pi}\mid \mathcal{D}_{m}(\boldsymbol{\pi})=0\}\\\boldsymbol{\pi'}\in\{\boldsymbol{\pi'}\mid \mathcal{D}_{\Tilde{m}'}(\boldsymbol{\pi'})=0\}}}\mathcal{D}_{\Tilde{m}'}(\pi^{i'},\boldsymbol{\pi'^{\textrm{-}i}}) \leq \frac{\gamma \iota_{\Tilde{m}'}}{\left(1 - \gamma\right)^2}\kappa(\pi^{i'}) + \frac{\kappa(\pi^{i'})}{1 - \gamma}.
\end{equation}

With the maximum influence $\iota_{m'}$ over the evaluation MG $m'\in\mdp'$, we obtain:
\begin{equation}
\label{eq_varsig}
\begin{aligned}
        \varsigma = \max_{\iota_{m'}}\frac{\gamma \iota_{m'}}{\left(1 - \gamma\right)^2}\kappa(\pi^{i'}) + \frac{\kappa(\pi^{i'})}{1 - \gamma}.
\end{aligned}
\end{equation}

Since $\mathcal{D}_{m}(\boldsymbol{\pi})=0$ and $\mathcal{D}_{m'}(\boldsymbol{\pi'})=0$, the best response in the Markov Game $m$ with other agents' policies fixed as $\boldsymbol{\pi^{\textrm{-}i'}}$ is $\pi^{i'}$. Similarly, in game $m'$, the best response with other agent's policies fixed as $\boldsymbol{\pi'^{\textrm{-}i}}$ is $\pi'^i$. Therefore, we obtain
\begin{align*}
    \kappa(\pi^{i'}) = \kappa(\pi'^{i})= \left\lVert{\left\lVert{\pi^{i'}(a) - \pi'^{i}(a) }\right\rVert_{a,1}}\right\rVert_{s,\infty}.
\end{align*}

For $\epsilon$ that satisfies \eqref{eps_con}, we obtain:
\begin{equation}
\label{eq_in_epsilon}
    \begin{aligned}
    \epsilon &\geq \max_{\iota_{m'},\iota_m}\varsigma - \frac{\varsigma\gamma\left(\iota_{m'}-\iota_{m}\right)}{\gamma \iota_{m'}+1-\gamma}\\
    &= \max_{\iota_{m'},\iota_m} \left(\frac{\gamma \iota_{m}+1-\gamma}{\gamma \iota_{m'}+1-\gamma}\right)\cdot\left(\frac{\gamma \iota_{m'}}{\left(1 - \gamma\right)^2}\kappa(\pi^{i'}) + \frac{\kappa(\pi^{i'})}{1 - \gamma}\right)\\
    &\geq \max_{\iota_m} \left(\frac{\gamma \iota_{m}+1-\gamma}{\gamma \iota_{\Tilde{m}'}+1-\gamma}\right)\cdot\left(\frac{\gamma \iota_{\Tilde{m}'}}{\left(1 - \gamma\right)^2}\kappa(\pi'^{i}) + \frac{\kappa(\pi'^{i})}{1 - \gamma}\right)\\
    &\geq \max_{\iota_m}\frac{\gamma \iota_{m}}{\left(1 - \gamma\right)^2}\kappa(\pi'^{i}) + \frac{\kappa(\pi'^{i})}{1 - \gamma},
    \end{aligned}
\end{equation}
where the equality follows from \eqref{eq_varsig}, and the last two inequalities follow from basic algebra.

Thus, we obtain
\begin{align*}
\mathcal{D}_{m}\left((\pi^{i'},\boldsymbol{\pi^{{\textrm{-}i'}}})\right)&=\left\lVert{\left\lVert{v^{*i,m}_{\boldsymbol{\pi^{{\textrm{-}i'}}}} - v^{i,m}_{\pi'^{i},\boldsymbol{\pi^{{\textrm{-}i'}}}}}\right\rVert_{s,\infty}}\right\rVert_{i,1}\\        &=\left\lVert{v^{i,m}_{\boldsymbol{\pi}} - v^{i,m}_{\pi'^{i},\boldsymbol{\pi^{{\textrm{-}i'}}}}}\right\rVert_{s,\infty}\\
    &= \left\lVert{v^{i,m}_{\pi^{i'},\boldsymbol{\pi^{{\textrm{-}i'}}}} - v^{i,m}_{\pi'^{i},\boldsymbol{\pi^{{\textrm{-}i'}}}}}\right\rVert_{s,\infty}\\
    &\leq \frac{\gamma \iota_{m}}{\left(1 - \gamma\right)^2}\kappa(\pi'^{i}) + \frac{\kappa(\pi'^{i})}{1 - \gamma}\\
    &\leq \max_{\iota_m}\frac{\gamma \iota_{m}}{\left(1 - \gamma\right)^2}\kappa(\pi'^{i}) + \frac{\kappa(\pi'^{i})}{1 - \gamma}\\
    &\leq \epsilon,
\end{align*}
where the second equality holds by the definition of $i$ and $i'$ in \eqref{i_d}, the third equality holds since $\mathcal{D}_{m}(\boldsymbol{\pi})=0$, and the last inequality follows from \eqref{eq_in_epsilon}.

This indicates that if we choose $\epsilon$ satisfying \eqref{eps_con}, then the policy $\pi'^i$ that achieves Nash Equilibrium in the Markov Game $\Tilde{m}'$ is guaranteed to be contained in the policy set. 

Since all the above inequalities hold for any $\Tilde{m}'\in\mdp'$, the policy that achieves NE in all MGs in the evaluation MG set $\mdp'$ is guaranteed to be included in the policy set. This completes the proof of Lemma \ref{prop1}.
\end{proof}

\subsection{Proof of Lemma \ref{prop_in}}
\label{A3}

\begin{proof}
We first know that
$$
\mathbb{E}_{\substack{\boldsymbol{a}\sim \boldsymbol{\hat{\pi}},\boldsymbol{o}}}\left [\log \boldsymbol{\hat{\pi}}(\boldsymbol{a}\mid \boldsymbol{o}) - \log \boldsymbol{\pi}(\boldsymbol{a}\mid \boldsymbol{o}) \right ] \geq 0.
$$

Since $\lvert\boldsymbol{\Pi}\rvert \geq\lvert\boldsymbol{\hat{\Pi}}\rvert$, the optimal $\lvert\boldsymbol{\Pi}\rvert$ that maximizes $\mathcal{L}(\boldsymbol{\Pi})$ in \eqref{maxmin} must satisfy

$$\min_{\boldsymbol{\hat{\pi}}\sim\boldsymbol{\hat{\Pi}}}\mathbb{E}_{\substack{\boldsymbol{a}\sim \boldsymbol{\hat{\pi}},\boldsymbol{o}}}\left [\log \boldsymbol{\pi}(\boldsymbol{a}\mid \boldsymbol{o}) \right ] \leq \mathbb{E}_{\substack{\boldsymbol{a}\sim \boldsymbol{\hat{\pi}},\boldsymbol{o}}}\left [\log \boldsymbol{\hat{\pi}}(\boldsymbol{a}\mid \boldsymbol{o})\right ] \leq  0.$$

In other words, for every policy $\boldsymbol{\hat{\pi}}$ in the joint strategy set $\boldsymbol{\hat{\Pi}}$, \ie, $\boldsymbol{\hat{\pi}}\in\boldsymbol{\hat{\Pi}}$, there exists a learned policy $\boldsymbol{\pi}\in \boldsymbol{\Pi}$, such that $\boldsymbol{\pi} = \boldsymbol{\hat{\pi}}$. Note that the above statement is true only when $\lvert\boldsymbol{\Pi}\rvert \geq \lvert\boldsymbol{\hat{\Pi}}\rvert$.

For $\epsilon$ that satisfies \eqref{eps_con}, from Lemma \ref{prop1} we know that for every evaluation MG $m'\in \mdp'$, the strategy that achieves Nash Equilibrium are guaranteed to be contained in $\boldsymbol{\hat{\Pi}}$. Thus, the optimal policy set $\boldsymbol{\Pi}$ that results from optimizing \eqref{maxmin} also contains the strategies that are NE in every MG $m'\in \mdp'$.

So the optimal policy set $\boldsymbol{\Pi}$ satisfies \eqref{primal}, which completes the proof.

\end{proof}

\subsection{Proof of Theorem \ref{theorem1}}
\label{A4}
\begin{proof}
Since $\lvert\boldsymbol{\Pi}_{\Theta}\rvert \geq \lvert\boldsymbol{\hat{\Pi}}\rvert$, we can simplify the objective in \eqref{maxmin} by updating the joint strategy $\boldsymbol{\pi}$ in a fixed-size set $\boldsymbol{\Pi}$. That is, maximizing \eqref{maxmin} is equivalent to

\begin{equation}
\label{simp}
\max_{\boldsymbol{\Pi}}\min_{\boldsymbol{\hat{\pi}}\sim\boldsymbol{\hat{\Pi}}}\max_{\boldsymbol{\pi}\sim\boldsymbol{\Pi}}\mathbb{E}_{\substack{\boldsymbol{a}\sim \boldsymbol{\hat{\pi}},\boldsymbol{o}}}\left [\log \boldsymbol{\pi}(\boldsymbol{a}\mid \boldsymbol{o}) \right ] = \max_{\boldsymbol{\pi}\sim\boldsymbol{\Pi}}\min_{\boldsymbol{\hat{\pi}}\sim\boldsymbol{\hat{\Pi}}}\mathbb{E}_{\substack{\boldsymbol{a}\sim \boldsymbol{\hat{\pi}},\boldsymbol{o}}}\left[\log \boldsymbol{\pi}(\boldsymbol{a}\mid \boldsymbol{o}) \right].
\end{equation}

From the non-negativeness of KL divergence, we have:
$$
\mathcal{D}_{KL} \left( \boldsymbol{\hat{\pi}}, \boldsymbol{\pi} \right) = \mathbb{E}_{\boldsymbol{\hat{\pi}}}\left [\log \frac{\boldsymbol{\hat{\pi}}}{\boldsymbol{\pi}} \right ] \geq 0,
$$
where the equality holds when $\boldsymbol{\pi}=\boldsymbol{\hat{\pi}}$.

Thus, we have from the definition of entropy that
\begin{align*}
\max_{\boldsymbol{\pi}}\mathbb{E}_{\substack{\boldsymbol{\hat{\pi}}}}\left [\log \boldsymbol{\pi} \right ] \leq - \mathcal{H}(\boldsymbol{\hat{\pi}}).
\end{align*}

If $\boldsymbol{\pi} \in \boldsymbol{\hat{\Pi}}$ is constrained, then the equality holds and $ \max\limits_{\boldsymbol{\pi}}\mathbb{E}_{\substack{\boldsymbol{\hat{\pi}}}}\left [\log \boldsymbol{\pi} \right ] = - \mathcal{H}(\boldsymbol{\hat{\pi}})$. This leads to
\begin{equation}
\begin{aligned}
\label{e2}
\max_{\boldsymbol{\pi}\sim\boldsymbol{\Pi}}\min_{\boldsymbol{\hat{\pi}}\sim\boldsymbol{\hat{\Pi}}}\mathbb{E}_{\substack{ \boldsymbol{\hat{\pi}}}}\left [\log \boldsymbol{\pi} \right ]
= & \min_{\boldsymbol{\hat{\pi}}\sim\boldsymbol{\hat{\Pi}}} - \mathcal{H}(\boldsymbol{\hat{\pi}})\\
= & \max_{\boldsymbol{\hat{\pi}}\sim\boldsymbol{\hat{\Pi}}} \mathcal{H}(\boldsymbol{\hat{\pi}})\\
= & \max_{\boldsymbol{\pi}\sim\boldsymbol{\Pi}}\mathcal{H}(\boldsymbol{\pi}), \text{~s.t.~} \boldsymbol{\pi}\in\boldsymbol{\hat{\Pi}}.
\end{aligned}
\end{equation}

The above equation states that the strategy $\boldsymbol{\pi}$ is learned to fit $\boldsymbol{\hat{\pi}}$. And to cover all the $\boldsymbol{\hat{\pi}}\in \boldsymbol{\hat{\Pi}}$, the entropy of strategies in $\boldsymbol{\Pi}$ should also be maximized.

Then we get the following equivalence:
\begin{equation}
\begin{aligned}
\label{e3}
\max_{\boldsymbol{\pi}\sim\boldsymbol{\Pi}}\mathcal{H}(\boldsymbol{\pi})
= & \max_{\Omega} \mathcal{I}(m;a\mid o) + \mathcal{H}(a\mid m,o)\\
= & \max_{\Omega} \mathcal{I}(m;a\mid o) + \mathcal{I}(g;a\mid m,o) + \mathcal{H}(a\mid m,o,g)\\
= & \max_{\Omega} \mathcal{I}(m;a\mid o) + \mathcal{I}(g;a\mid o) + \mathcal{H}(a\mid m,o,g)\\
= & \max_{\Omega} \mathcal{I}(m;a\mid o) + \mathcal{I}(g;a\mid o),\\
\end{aligned}
\end{equation}
where the first equality holds following the definition of mutual information and by noticing that policy $\boldsymbol{\pi}$ is meta-represented by $\Omega$. 
The last equality holds since the size of the meta-represented policy set $\mid \boldsymbol{\Pi}_{\Theta}\mid $ is sufficiently large and satisfying $\mid \boldsymbol{\Pi}_{\Theta}\mid  \geq \mid \boldsymbol{\hat{\Pi}}\mid $.

Combining \eqref{e2} and \eqref{e3}, we have
\begin{equation}
\label{c1}
\min_{\boldsymbol{\hat{\pi}}\sim\boldsymbol{\hat{\Pi}}}\max_{\boldsymbol{\pi}\sim\boldsymbol{\Pi}}\mathbb{E}_{\substack{ \boldsymbol{\hat{\pi}}}}\left [\log \boldsymbol{\pi} \right ] = \max_{\Omega} \mathcal{I}(m;a\mid o) + \mathcal{I}(g;a\mid o), \text{~s.t.~} \boldsymbol{\pi}\in\boldsymbol{\hat{\Pi}}.
\end{equation}

Then by \eqref{simp} and \eqref{c1}, we have that the solution of the following objective is equivalent to the solution of \eqref{maxmin}, i.e.,
$$
    \psi^*, \phi^* = \argmax_{\psi,\phi} \mathcal{I}(g;a\mid o) + \mathcal{I}(m;g\mid o)
    \text{~~s.t.~~}  \boldsymbol{\pi}_{\theta^*} \in \boldsymbol{\hat{\Pi}}.
$$
Thus, for every evaluation Markov Game $m'\in \mdp'$, there exists a joint strategy $\boldsymbol{\pi}\in \boldsymbol{\Pi}_{\Theta}$ that reaches Nash Equilibrium (\ie, $\boldsymbol{\Pi}_{\Theta^*}=\boldsymbol{\Pi}^*$ satisfies \eqref{primal}).
\end{proof}

\section{Fast Adaptation with First-Order Gradient}
\label{fast_adap}
Reptile \citep{nichol2018first} is a meta-learning algorithm that uses first-order gradient information for fast adaptation.

For parameter $\theta$ that maximizes objective $\mathcal{L}^k_{m}(\theta)$ in the $k$-th mini-batch of game $m$, $\theta$ is updated by $\theta\shortleftarrow\theta +\alpha \Delta\theta$, where $\Delta\theta = U_m^K(\theta) - \theta$ and $U_m^K(\theta)$ denotes the updated $\theta$ after $K$ gradient steps with learning rate $\beta$, and $\alpha$ is a hyperparameter. 

Denote the $k$-th step parameter as $\theta_k$, then the update $\Delta \theta$ of $K$ gradient steps is as follows:
\begin{align*}
    \Delta \theta \!&= \theta_K - \theta_1\\
    &= \beta \sum_{k=1}^{K-1}\nabla{\mathcal{L}}_m^k(\theta_k)\\
    &= \beta \sum_{k=1}^{K-1}\left(\nabla{\mathcal{L}}_m^k(\theta_1) + \nabla^{2}{\mathcal{L}_m^k(\theta_1)}\left(\theta_k - \theta_1\right) + \mathcal{O}\left(\left\lVert{\theta_k - \theta_1}\right\rVert^2\right)\right)\\
    &= \beta \sum_{k=1}^{K-1}\left(\nabla{\mathcal{L}}_m^k(\theta_1) + \beta \nabla^{2}{\mathcal{L}_m^k(\theta_1)}\sum_{j=1}^{k-1} \nabla{\mathcal{L}}_m^k(\theta_j) + \mathcal{O}\left(\beta^2\right)\right)\\
    &=\beta\left[\sum\limits_{k=1}^{K-1}\left(\nabla{\mathcal{L}}_m^k(\theta_1) \!+\! \beta\sum\limits_{j=1}^{k-1}\Big(\nabla^{2}{\mathcal{L}_m^k(\theta_1)} \nabla{\mathcal{L}_m^j(\theta_1)}\Big)\right)\!+\! \mathcal{O}\left(\beta^2\right)\right],
\end{align*}
where the last equation holds since $\nabla{\mathcal{L}}_m^k(\theta_j) = \nabla{\mathcal{L}}_m^k(\theta_1) + \mathcal{O}\left(\beta\right)$.

For the initial parameter $\theta = \theta_1$, the term $\sum\limits_{k=1}^{K-1}\nabla{\mathcal{L}}_m^k(\theta_1)$ maximizes the overall performance at $\theta$ in all the $K$ mini-batches in an MG $m$. The key difference from the joint training
objective is the term $\nabla^{2}{\mathcal{L}_m^k(\theta_1)} \nabla{\mathcal{L}_m^j(\theta_1)}$. When the expectation are taken under mini-batch sampling in $m$, we denote by $\mathbb{E}_{k}$ the expectation over the mini-batch defined by $J^k$. Omitting the higher-order term $\mathcal{O}\left(\beta^2\right)$, we have
\begin{align*}
    \mathbb{E}\left[\Delta\theta\right] &= (K-1) \mathbb{E}_k \left[\nabla\mathcal{L}_m^k(\theta)\right] + (K-1)(K-2)\beta\cdot
    \mathbb{E}_{j,k}\left[\nabla^{2}{\mathcal{L}_m^k(\theta)} \nabla{\mathcal{L}_m^j(\theta)}\right] \\
    &= (K-1) \mathbb{E}_k \left[\nabla\mathcal{L}_m^k(\theta)\right]\\
    &+ \frac{(K-1)(K-2)\beta}{2} \mathbb{E}_{j,k}\left[\nabla^{2}{\mathcal{L}_m^k(\theta)} \nabla{\mathcal{L}_m^j(\theta)} + \nabla^{2}{\mathcal{L}_m^j(\theta)} \nabla{\mathcal{L}_m^k(\theta)}\right]\\
    &= (K-1) \mathbb{E}_k \left[\nabla\mathcal{L}_m^k(\theta)\right] + \frac{(K-1)(K-2)\beta}{2} \mathbb{E}_{j,k}\left[\nabla{\Big(\nabla{\mathcal{L}_m^k(\theta)} \nabla{\mathcal{L}_m^j(\theta)}}\Big)\right].
\end{align*}

Thus, updating $\theta$ by $\theta\shortleftarrow\theta +\alpha \Delta\theta$ not only maximizes the average performance in $K$ mini-batches of all Markov Games, but also maximizes the inner product between gradients of different mini-batches, \ie, $\nabla{\mathcal{L}_m^k(\theta)} \nabla{\mathcal{L}_m^j(\theta)}$. Thus, the generalization ability is improved and fast adaptation is achieved.

\section{Derivation of Mutual Information Calculation}
\label{mi_der}
The two mutual information terms in \eqref{objective}, \ie, $\mathcal{I}(g;a\mid o)$ and $\mathcal{I}(m;g\mid o)$ can be calculated as follows:
\begin{align*}
    \mathcal{I}(g;a\mid o) &= \int p(a,o,g)\log\frac{p(a\mid o,g)}{p(a\mid o)}da\,do\,dg\\
    &= \mathbb{E}_{a,o,g}[\log\frac{\pi_{\bar{\theta}}(a\mid o,g)}{p(a\mid o)}]+ \mathbb{E}_{a,o,g}[\mathcal{D}_{KL}(p(a\mid o,g)\mid \mid \pi_{\bar{\theta}}(a\mid o,g))]\\
    &\geq \mathbb{E}_{a,o,g}[\log\frac{\pi_{\bar{\theta}}(a\mid o,g)}{p(a\mid o)}]\\
    &\approx \mathbb{E}_{o\sim D, z\sim p(\cdot\mid m), a \!\sim\! \pi_{{\theta}}(\cdot\mid o,\phi(o,z))}\left[\log\frac{\pi_{\bar{\theta}}(a\mid o,\phi(o,z))}{\mathbb{E}_{z'\sim p(\cdot\mid m), g'=\phi(o,z')}\left[\pi_{\bar{\theta}}(a\mid o,g')\right]}\right].
\end{align*}

Similarly, we have for the second mutual information term that
\begin{align*}
\mathcal{I}(m;g\mid o) &= \int p(m,o,g)\log\frac{p(m\mid o,g)}{p(m\mid o)}dm\,do\,dg\\
&= \mathbb{E}_{\substack{m,o,g}}[\log\frac{p(m\mid o,g)}{p(m)}]\\
&= \mathbb{E}_{m,o\sim D}\left[\log p(m\mid o,g)\right] + \log \lvert\mdp\rvert.
\end{align*}

When maximizing the mutual information objectives described in Section \ref{o_s}, the Gumbel-softmax trick \citep{jang2016categorical} can be used for discrete $z$.

\section{Complete Pseudocode}
\label{pcode}
We begin by describing the \textit{overview} of the training and adaptation procedures of MRA in Algorithm \ref{alg:training} and \ref{alg:evaluation}, respectively. Notably, the main difference between Algorithm \ref{alg_high} in the main text and Algorithm \ref{alg:training} here lies in that the latter is instantiated from the former, using the policy gradient and mutual information objectives depicted in Section \ref{sec_opt}.

\begin{figure}[h]
\begin{minipage}{0.49\textwidth}
\begin{algorithm}[H]
\centering
\caption{MRA: Training in the MG set $\mdp$ (overview)}
\begin{algorithmic}
\WHILE {not converged}
    \FOR {$\text{MG }   m\in\mdp$}
        \STATE {Sample lower-level latent $z\sim p_\psi(\cdot|m)$}
            \STATE {Execute action $a\!\sim\!\pi_{\theta}(\cdot|o,g)$, where $g=\phi(o,z)$}
            \STATE {Push $(\boldsymbol{o},\boldsymbol{a},\boldsymbol{o'},\boldsymbol{g},\boldsymbol{r})$ to replay buffer}
            \FOR {$k = 1, \ldots, K$}
            \STATE {Update critic $\zeta$ by minimizing  \eqref{loss_q}} \STATE {update policy at the $k$-th step $\theta_k$ by  \eqref{grad_pi}}
            \ENDFOR
            \STATE {Update $\theta$ by $\theta\shortleftarrow\theta + \alpha(\theta_K - \theta)$}\\
            \STATE {Update $\phi$ to maximize the RHS of  \eqref{mi1};}\\
            \STATE {Update $\psi$ and $\xi$ by \eqref{mi2}}
            \STATE {Update delayed parameters $\bar{\theta}$ and $\bar{\zeta}$}
    \ENDFOR
\ENDWHILE
\end{algorithmic}
\label{alg:training}
\end{algorithm}
\end{minipage}
\hfill
\begin{minipage}{0.49\textwidth}
\begin{algorithm}[H]
\centering
\caption{MRA: Adaptation in an evaluation Markov Game $m'\in \mdp'$ (overview)}
\begin{algorithmic}
\WHILE {not converged}
            \STATE {Sample lower-level latent $z\sim p_\psi(\cdot|m')$}
            \STATE {Execute action $a\!\sim\!\pi_{\theta}(\cdot|o,g)$, where $g=\phi(o,z)$}
            \STATE {Push $(\boldsymbol{o},\boldsymbol{a},\boldsymbol{o'},\boldsymbol{r})$ to replay buffer}
                            \STATE {Update critic $\zeta$ by minimizing  \eqref{loss_q}} \STATE {Update $\theta$ and $\phi$ by  \eqref{grad_pi_eval}}
            \STATE {Update delayed parameters $\bar{\theta}$ and $\bar{\zeta}$}
\ENDWHILE
\end{algorithmic}
\label{alg:evaluation}
\end{algorithm}
\end{minipage}
\end{figure}

The complete pseudocode of the training and adaptation procedures of MRA  that contains the training details and agent indexes is provided in Algorithm \ref{alg:training_app} and Algorithm \ref{alg:adaptation}, respectively.

\begin{algorithm}
\caption{MRA: Training Procedure of MRA (complete)}
\label{alg:training_app}
\begin{algorithmic}
\STATE {\textbf{Input:} Training set $\mdp$ that contains Markov Games constructed by varying the population (i.e., the number of agents) from the same underlying environment.}
\STATE {Initialize $P$ threads of games}
\STATE {Initialize $T_{\text{update}} \shortleftarrow 0$}
\STATE {Initialize replay buffer $\mathcal{D}_m$ for each Markov Game $m$}
\WHILE {total episode number not reach}
    \FOR {$\text{Markov Game }   m = 1, \ldots, \left|\mdp\right|$}
        \STATE {Reset game, each agent $i$ samples lower-level latent code $z^i\sim p_\psi(z|m)$}
        \FOR {time steps in an episode}
            \STATE {Each agent $i$ executes action $a^i\sim\pi\!\left(\cdot|o^i,\phi_i(o^i,z^i);\theta^i\right)$ simultaneously and get reward $r^i$, next observation $o'^i$}
            \STATE {Push $\left(\boldsymbol{o},\boldsymbol{a},\boldsymbol{o'},\boldsymbol{g},\boldsymbol{r}\right)$ to replay buffer $\mathcal{D}_m$}
            \STATE {$\boldsymbol{o} \shortleftarrow \boldsymbol{o'}$}
            \STATE {$T_{\text{update}} \shortleftarrow T_{\text{update}} + P$}
            \IF {$T_{\text{update}} \%  (\text{min steps per update}) \leq P$}
            \STATE {A mini-batch of $B$ samples of $\left(\boldsymbol{o}_b,\boldsymbol{a}_b,\boldsymbol{o'}_b,\boldsymbol{g}_b,\boldsymbol{r}_b\right)$ is sampled from $\mathcal{D}_m$}
            \FOR {$k = 1 $ to $ K$}
            \STATE {Update all agents' critic parameter $\zeta^i$ by minimizing $\mathcal{L}\left(\zeta^i\right) = \frac{1}{B}\sum_{b}\left(\mathcal{B}^i_{\boldsymbol{\pi}} Q - Q\left(\boldsymbol{o}_b,\boldsymbol{a}_b,g^i_b;\zeta^i\right)\right)^2$,}
            \STATE {\hspace*{3em}$\text{where } \mathcal{B}^i_{\boldsymbol{\pi}} Q = r^i_b + \gamma  \mathbb{E}_{\boldsymbol{a'} \sim \bar{\boldsymbol{\pi}}}\left[ Q\left(\boldsymbol{o'_b},\boldsymbol{a'},g^i_b;\bar{\zeta}^i\right) \right]$}
            \STATE {The $k$-th step of policy parameter $\theta_k^i$ is updated by gradient ascent:} 
            \STATE {\hspace*{3em}$\nabla_{\theta^i}J = \frac{1}{B}\sum_{b}{\nabla_{\theta^i}\log\pi\!\left(a^i|o^i_b,g^i_b;\theta^i_k\right) Q\left(\boldsymbol{o}_b,\boldsymbol{a}_b,g^i_b;\zeta^i\right)}$}
            \ENDFOR
            \STATE{Update all agents' $\theta^i$ by $\theta^i\shortleftarrow\theta^i + \alpha(\theta^i_K - \theta^i)$}
            \STATE {Sample $n$ latent code $z'^i$ and approximate $p(a|o^i_b)$ by $\frac{1}{n}\sum\pi\!\left(a^i|o^i_b,\phi^i\left(o^i_b,z'^i\right);\bar{\theta}^i\right)$}
            \STATE {Update all agents' $\phi_i$ by maximizing:}
            \STATE {\hspace*{3em}$\mathcal{L}\left(\phi^i\right) = \frac{1}{B}\sum_{b}\mathbb{E}_{\small\substack{a^i \sim \pi(\cdot|o^i_b,\phi^i(o^i_b,z^i);{\theta}^i)}\normalsize}\left[\log\bigl(\pi(a^i|o^i_b,\phi^i(o^i_b,z^i);\bar{\theta}^i) / p(a|o)\bigr)\right]$}
            \STATE {Update $\psi$ and  auxiliary network $\xi$ simultaneously by minimizing:}
            \STATE {\hspace*{3em}$\mathcal{L}\left(\psi,\xi\right) = \mathbb{E}_{\substack{z'^i\!\sim\!p_{\psi}(\cdot|m)}}\left[-y \log\left(p\left(\hat{y}|o^i_b,\phi(o^i_b,z'^i);\xi\right)\right)\right]$}
            \STATE {Update all agents' delayed parameters $\bar{\theta}^i$ and $\bar{\zeta}^i$}
            \ENDIF
            \ENDFOR
    \ENDFOR
\ENDWHILE
\STATE {\textbf{Output:} Parameter $\psi$ and $\theta^i$, $\phi^i$, $\zeta^i$ for each agent $i$.}
\end{algorithmic}
\end{algorithm}

\newpage
\begin{algorithm}
\caption{MRA: Adaptation Procedure of MRA (complete)}
\label{alg:adaptation}
\begin{algorithmic}
\STATE {\textbf{Input:} Trained parameters from Algorithm \ref{alg:training_app}, an evaluation Markov Game $m'\in\mdp'$.}
\STATE {Initialize replay buffer $\mathcal{D}$;}
\STATE {Each agent $i$ samples lower-level latent code $z^i\sim p_\psi(z|m')$;}
\WHILE {total episode number not reach}
        \STATE {Reset game and receive initial observation $\boldsymbol{o}$;}
        \FOR {time steps in an episode}
            \STATE {Each agent $i$ executes action $a^i\sim\pi\!\left(\cdot|o^i,\phi_i(o^i,z^i);\theta^i\right)$ simultaneously and get reward $r^i$, next observation $o'^i$\hspace{-0.1cm}}
            \STATE {Push $\left(\boldsymbol{o},\boldsymbol{a},\boldsymbol{o'},\boldsymbol{r}\right)$ to replay buffer $\mathcal{D}_{m'}$}
            \STATE {$\boldsymbol{o} \shortleftarrow \boldsymbol{o'}$}
            \STATE {A mini-batch of $B$ samples of $\left(\boldsymbol{o}_b,\boldsymbol{a}_b,\boldsymbol{o'}_b,\boldsymbol{r}_b\right)$ is sampled from $\mathcal{D}_m$}
            \STATE {Calculate the detached relational graph $g^i = \phi_i(o^i_b,z^i_b)$}
            \STATE {Update all agents' critic parameter $\zeta^i$ by minimizing:}
            \STATE {\hspace*{3em}$\mathcal{L}\left(\zeta^i\right) = \frac{1}{B}\sum_{b}\left(\mathcal{B}^i_{\boldsymbol{\pi}} Q - Q\left(\boldsymbol{o}_b,\boldsymbol{a}_b,g^i;\zeta^i\right)\right)^2, \text{where } \mathcal{B}^i_{\boldsymbol{\pi}} Q = r^i_b + \gamma  \mathbb{E}_{\boldsymbol{a'} \sim \bar{\boldsymbol{\pi}}}\left[ Q\left(\boldsymbol{o'_b},\boldsymbol{a'},g^i;\bar{\zeta}^i\right) \right]$}
            \STATE {Update all agents' parameter $\omega^i =(\theta^i,\phi^i)$ by gradient ascent:}
            \STATE {\hspace*{3em}$\nabla_{\omega^i}J = \frac{1}{B}\sum_{b}{\nabla_{\omega^i} \log\pi\!\left(a^i|o^i_b,\phi^i(o^i_b,z^i_b);\theta^i_k\right) Q\left(\boldsymbol{o}_b,\boldsymbol{a}_b,g^i;\zeta^i\right)}$}
            \STATE {Update all agents' delayed parameters $\bar{\theta}^i$ and $\bar{\zeta}^i$}
            \ENDFOR
\ENDWHILE
\STATE {\textbf{Output:} Parameter $\psi$ and $\theta^i$, $\phi^i$, $\zeta^i$ for each agent $i$.}
\end{algorithmic}
\end{algorithm}

\section{Details of Experiments}
\label{add_exp}

\subsection{Experiment Settings and Task Descriptions}
\label{sec_app_env_set}
\textbf{Treasure Collection:} Each agent is with the goal to collect more treasures in an episode. Treasures disappear and re-generate at a random location when touched by agents. Agents must act according to the distances to other agents and treasures to gain high rewards, and paying attention to the right agents is key to success in this task.

\textbf{Resource occupation:} Agents receive rewards for occupying varisized resource landmarks: higher reward if one agent is occupying a larger resource with fewer other agents in it. Intuitively, training in various MGs could improve the performance in each single game benefiting from knowledge transfer. Specifically, agents can learn to move to large resources in small-population games, and learn to keep away from other agents in the competitive games with large populations. We will verify the benefit of such knowledge transfer with meta representations in Section \ref{benefits_single}.

\textbf{Pacman-like World:} There are competing agents in the Pacman-like World, Pac-Man agents and ghost agents. Pac-Man agents are with goals to collect food and elude ghost agents, and ghost agents are with goals to touch the Pac-Man agents. The environment is similar to the predatory-prey environment, but with additional food dots, which means that the Pacman game is not a zero-sum game.

The screenshots of the three environments in the experiments are shown in Figure \ref{demo}. The treasures in the treasure collection environment and the food dots in the Pacman-like world are randomly initialized in the position range of $[-1,+1]$, and regenerated when touched by collector agents and PacMan agents, respectively. For an $n$-resource occupation task, the sizes of the resource landmarks are pre-defined as $\{0.1,0.2,\ldots, 0.1\times n\}$ and fixed in each episode. 

\begin{figure*}[htbp]
\centering
\subfigure[Treasure collection.]{
\begin{minipage}[t]{0.3\linewidth}
\centering
\includegraphics[width=1.8in]{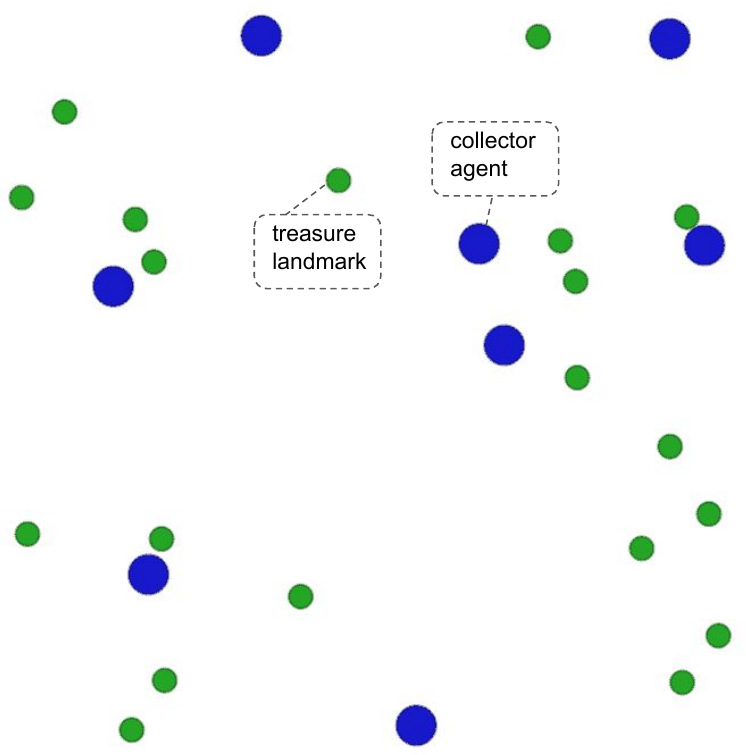}
\end{minipage}%
}%
\subfigure[Resource occupation.]{
\begin{minipage}[t]{0.3\linewidth}
\centering
\includegraphics[width=1.8in]{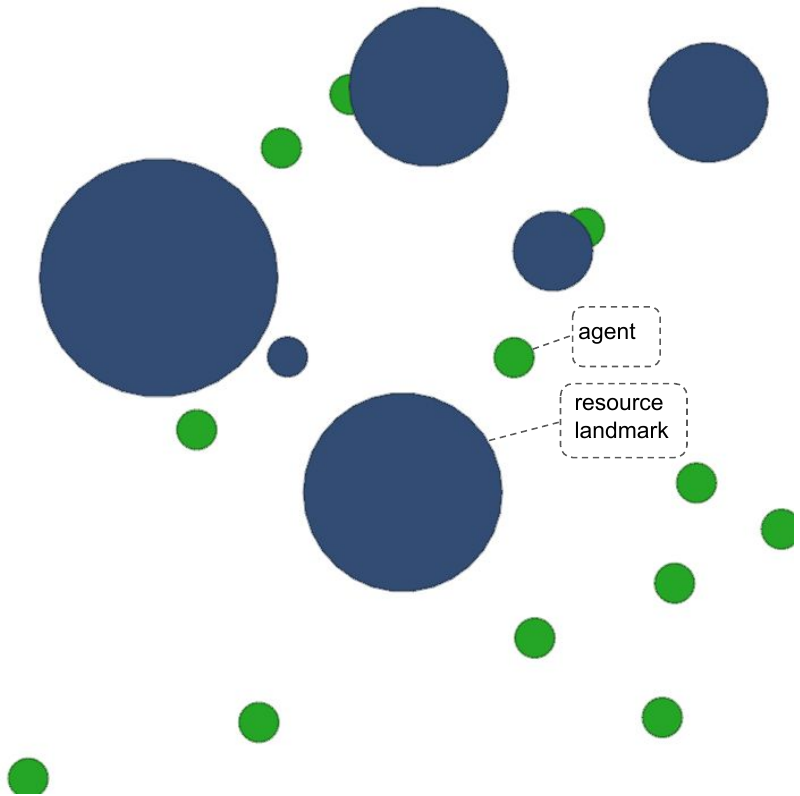}
\end{minipage}%
}%
\subfigure[Pacman-like world.]{
\begin{minipage}[t]{0.3\linewidth}
\centering
\includegraphics[width=1.8in]{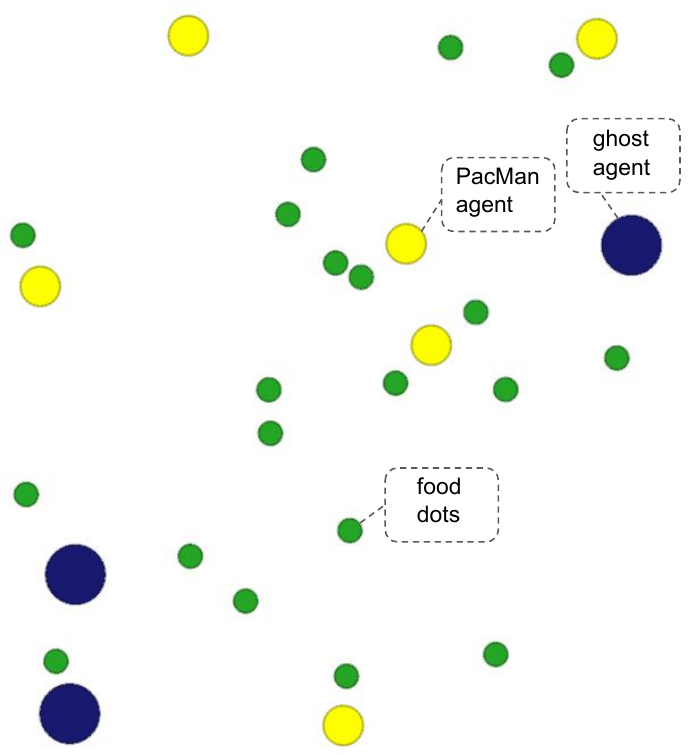}
\end{minipage}%
}%
\centering
\caption{The illustration of the three environments that are used in our experiments.}
\label{demo}
\end{figure*}

We adopt the same set of hyperparameters for experiments. Specifically, $12$ rollouts are executed in parallel when training. The maximum length of the replay buffer is $1e6$. The episode length is set to $20$ and the number of random seeds is set to $4$. The dimension of the latent code $z$ is 6. The critic also adopts a self-attention network in a similar way with MAAC \citep{iqbal2018actor}. And the number of gradient steps of policy and critic parameters in each update, \ie, $K$, is set as $10$. And $\alpha=1$ works well in experiments. Batch size is set to $1024$ and Adam is used as the optimizer. The initial learning rate is set to $0.0003$. In all experiments, we use one NVIDIA Tesla P40 GPU.



\subsection{Cross-Comparison Results}
\label{sec_app_cc}
We provide the cross-comparison results in the PacMan-like world. The comparisons are conducted between the MRA agents trained in multiple MGs and the agents trained in a single MG. The score is summed in each episode, averaged across homogeneous agents on $40$ runs, and normalized. 
\begin{figure*}[ht]
\begin{minipage}{0.50\textwidth}{
\captionof{table}{PacMan scores.}
\label{cross1}
\renewcommand\arraystretch{1.18}
\centering
\scalebox{0.98}{
\begin{tabular}{p{2cm}<{\centering} p{2.0cm}<{\centering} p{1.8cm}<{\centering}}
\hline
{\diagbox[innerwidth=2cm]{Ghosts}{PacMan}}  & Single  & MRA \\
\hline
Single &0.78 & \textbf{1.00} \\
MRA & 0.54  & 0.89  \\
\hline
\end{tabular}}}
\end{minipage}
\hspace{3mm}
\begin{minipage}{0.50\textwidth}{
\captionof{table}{Ghost scores.}
\label{cross2}
\renewcommand\arraystretch{1.18}
\centering
\scalebox{0.98}{
\begin{tabular}{p{2cm}<{\centering} p{2.0cm}<{\centering} p{1.8cm}<{\centering}}
\hline
{\diagbox[innerwidth=2cm]{Ghosts}{PacMan}}  & Single  & MRA \\
\hline
Single &0.82 & 0.59 \\
MRA & \textbf{1.00}  & 0.85  \\
\hline
\end{tabular}}}
\end{minipage}
\end{figure*}

The cross-comparison results are shown in Table \ref{cross1} and Table \ref{cross2}. We can see that the agents created by the proposed MRA outperform the single-MG counterparts for both PacMan agents and ghosts agents, validating the effectiveness of the proposed method.

\subsection{Ablation on the Implementation Variants}
We now depict two variants of implementing $\phi(o,z)$ and compare them with the default implementation, which we denote as the option \citep{sutton1999between} architecture.

Specifically, the variants we consider are the ones discussed in \citep{florensa2017stochastic}. The first variant is to concatenate $z$ to each entity of the observation decomposition $o^i = \left[o_s^i, o_1^i, \cdots, o_j^i, \cdots, o_N^i\right]$. The same relational representation is also adopted to generate the relational graph $g$. The second variant is to perform the outer product between each observation entity and $z$. We refer to the two variants as ``concat" and ``bilinear", respectively. 

The performance of the three implementations is evaluated in the 6-resource occupation environment. The number of training MGs is set to $3$, and the numbers of agents in the $3$ games are $\{6,9,12\}$. The results in Figure \ref{variants} show that all the three implementations can obtain agents that effectively act in all the $3$ scenarios. And the default option architecture achieves better performance than the other two variants. The possible reason is that the lower-level latent code $z$ in the option architecture can explicitly control the structural factors and can thus learn the common knowledge more quickly and better.

\begin{figure*}[htbp]
\centering
\subfigure{
\begin{minipage}[t]{0.33\linewidth}
\centering
\includegraphics[width=2.2in]{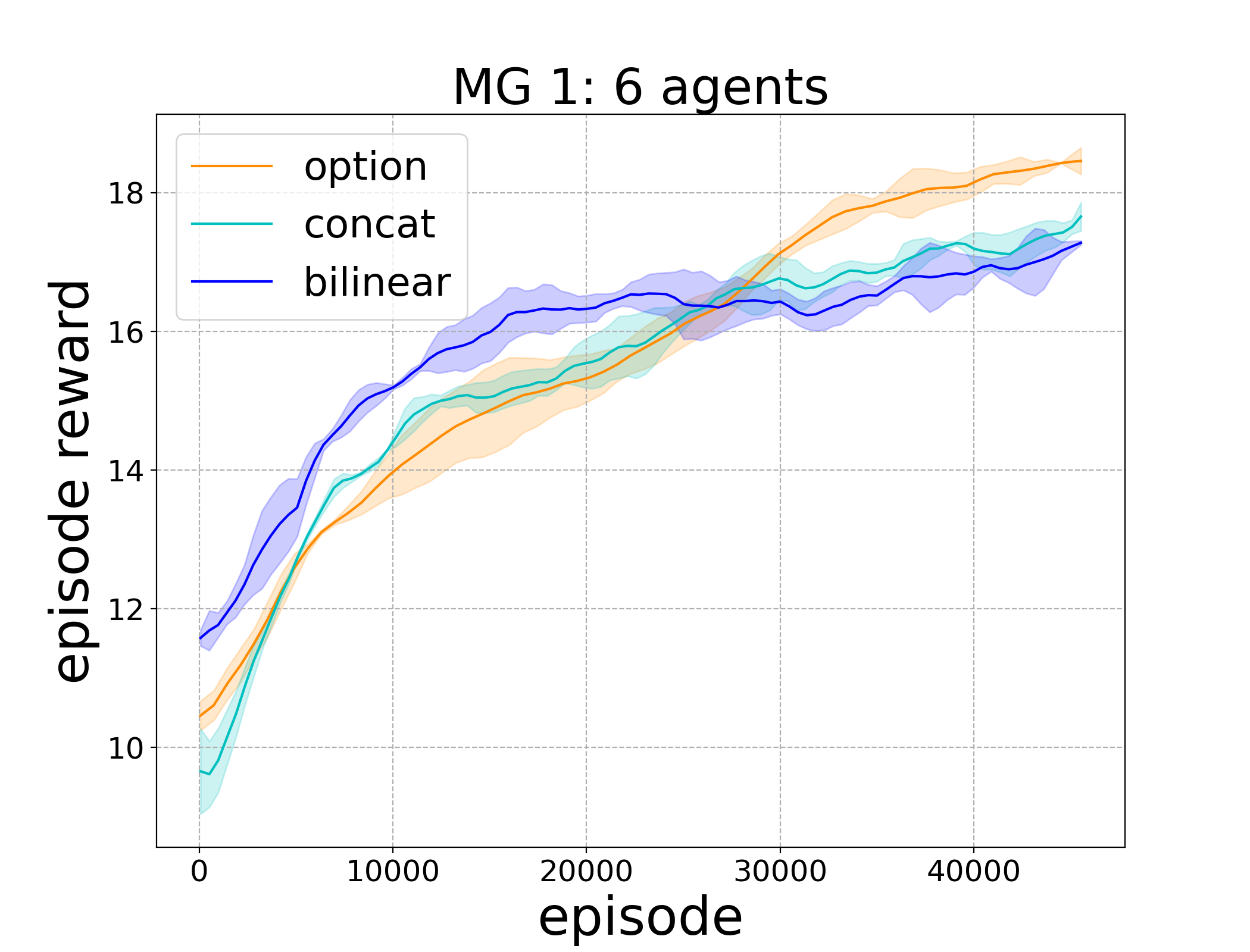}
\end{minipage}%
}%
\subfigure{
\begin{minipage}[t]{0.33\linewidth}
\centering
\includegraphics[width=2.2in]{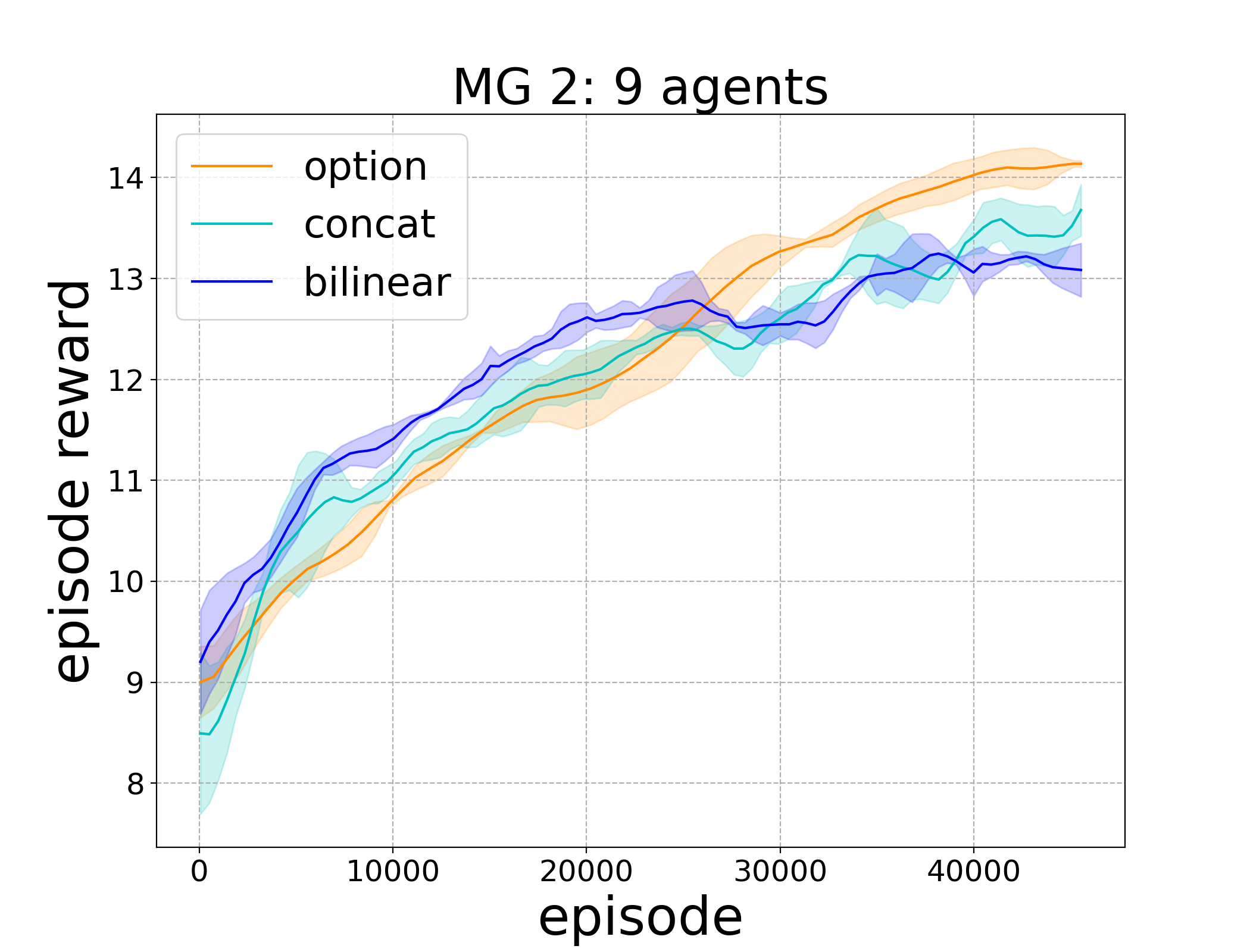}
\end{minipage}%
}%
\subfigure{
\begin{minipage}[t]{0.33\linewidth}
\centering
\includegraphics[width=2.2in]{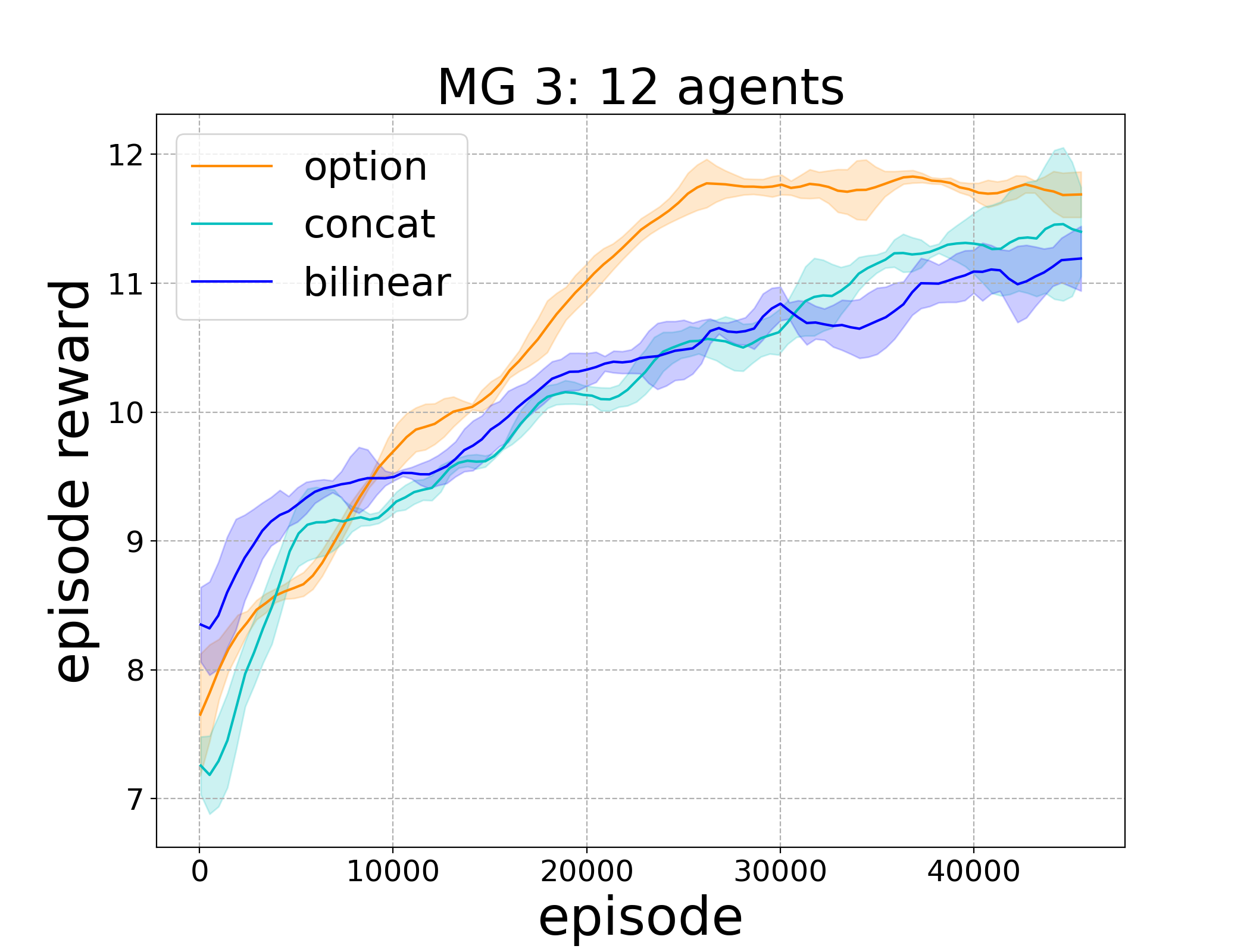}
\end{minipage}%
}%
\centering
\caption{Performances of different implementation variants in the resource occupation environment.}
\label{variants}
\end{figure*}

\subsection{Ablation Study on the Size of Training MG Set}
The information in all the training MGs determines the common knowledge that agents can learn. We provide ablation study on the number of training MGs. In the resource occupation environment, we train the agents in three settings, each of which is with different size of training MG set:  $2,3$ and $4$. Specifically, the population size of the three settings are: $\{6,12\}$, $\{6,9,12\}$ and $\{6,9,12,15\}$. The curves in MGs with $\{6,9,12\}$ populations are shown in Figure \ref{sc_num}.

\begin{figure*}[htbp]
\centering
\subfigure{
\begin{minipage}[t]{0.33\linewidth}
\centering
\includegraphics[width=2.2in]{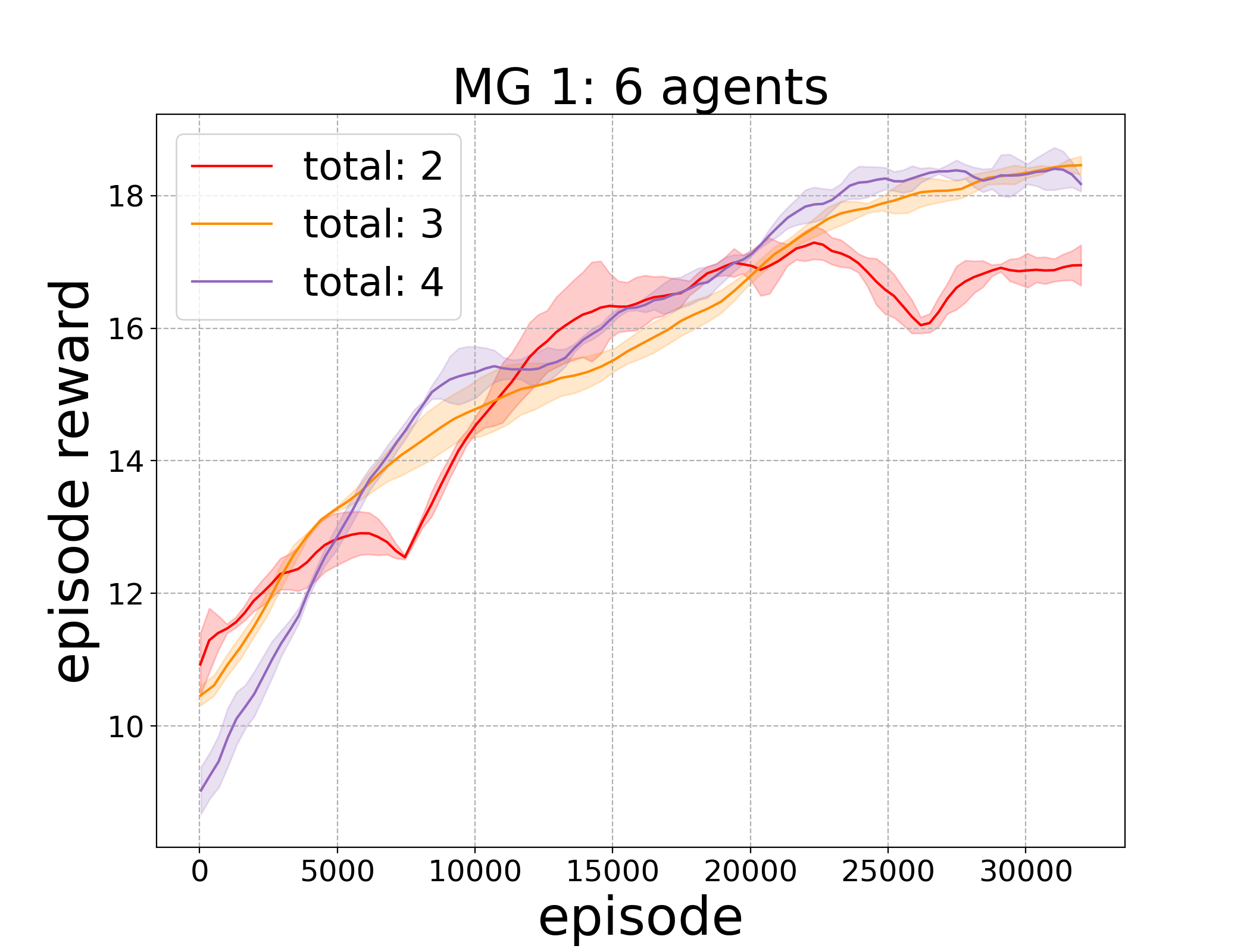}
\end{minipage}%
}%
\subfigure{
\begin{minipage}[t]{0.33\linewidth}
\centering
\includegraphics[width=2.2in]{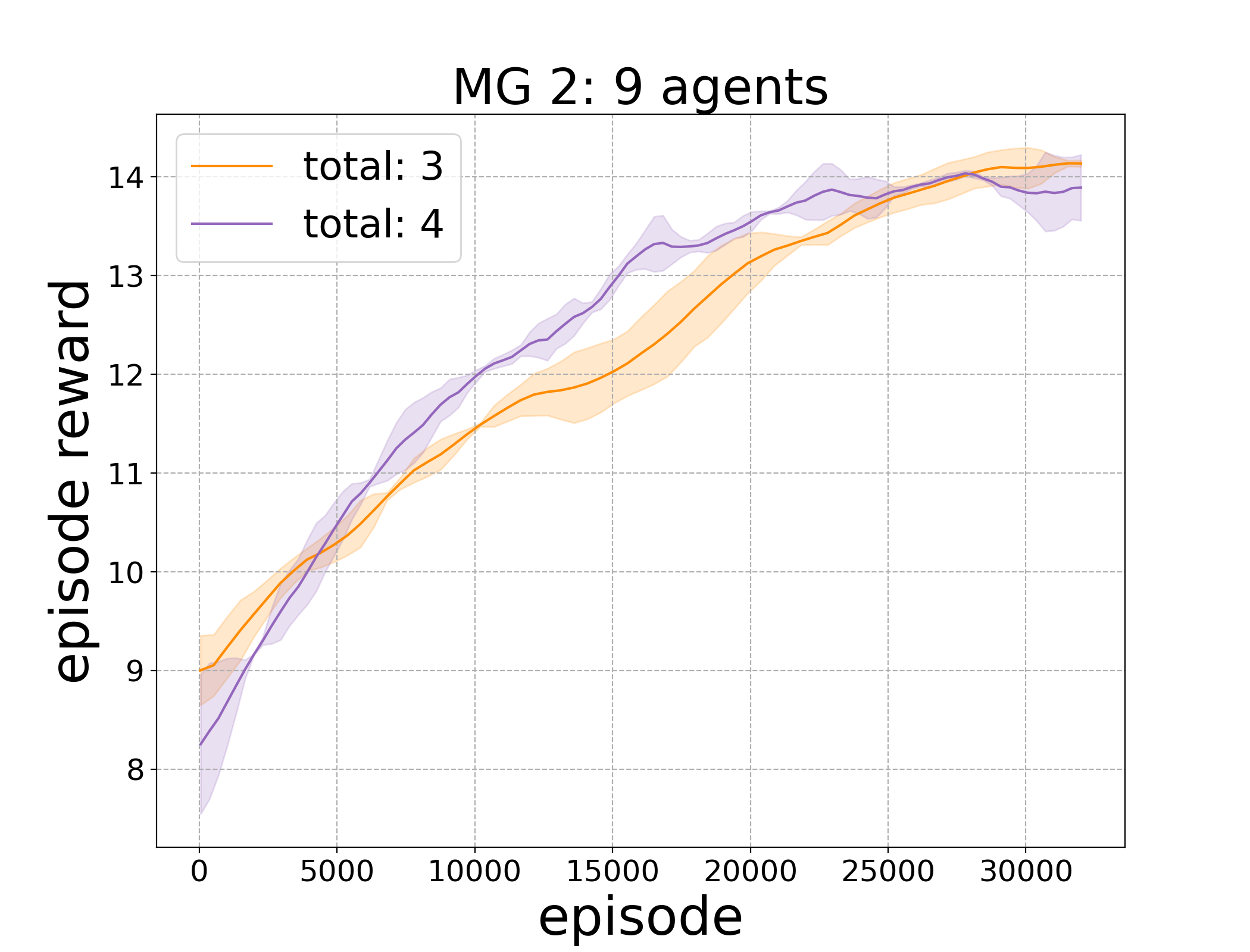}
\end{minipage}%
}%
\subfigure{
\begin{minipage}[t]{0.33\linewidth}
\centering
\includegraphics[width=2.2in]{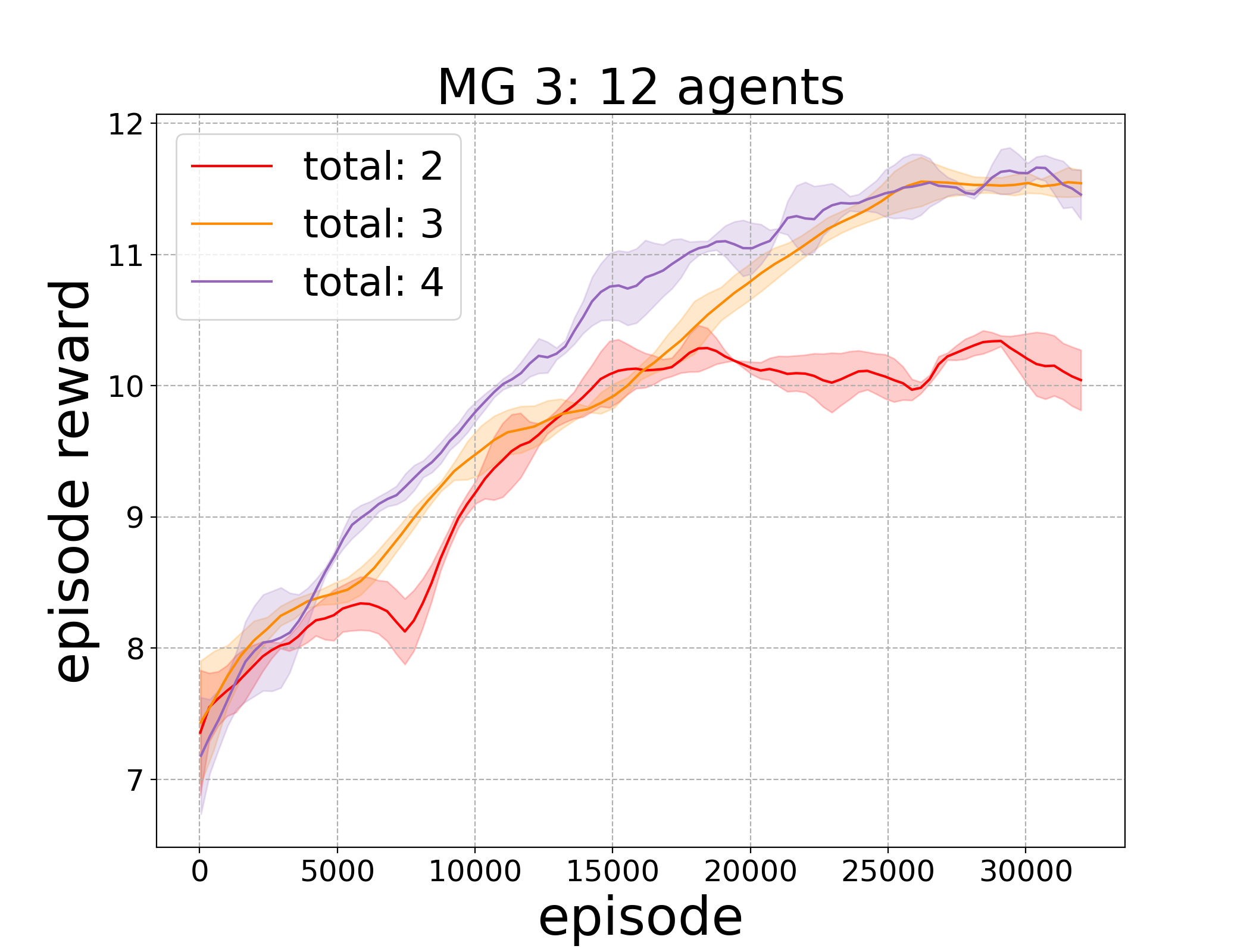}
\end{minipage}%
}%
\centering
\caption{Ablation study on the size of the training MG set in the resource occupation environment.}
\label{sc_num}
\end{figure*}

We observe that when the size of the training MG set is greater than $2$, the benefits of the meta-representation are obvious. The knowledge that agents learn from few training MGs, \eg, $1$ or $2$ training MGs, is limited, and the random exploration bottleneck still exists. However, the performance can be significantly improved by leveraging the information from more training MGs, \eg, $3$ or $4$ training MGs, where the common knowledge is more likely to be distilled and thus guiding the exploration.


\subsection{Training Curves}
We provide the training phase curves of the approximated mutual information $I(g;a\mid o)$ and the inference loss of MG index output by auxiliary network $\xi$. The curves are shown in Figure \ref{curves}.
\begin{figure*}[htbp]
\centering
\subfigure[Curve of mutual information $I(g;a\mid o)$ during training.]{
\begin{minipage}[t]{0.45\linewidth}
\centering
\includegraphics[width=2.3in]{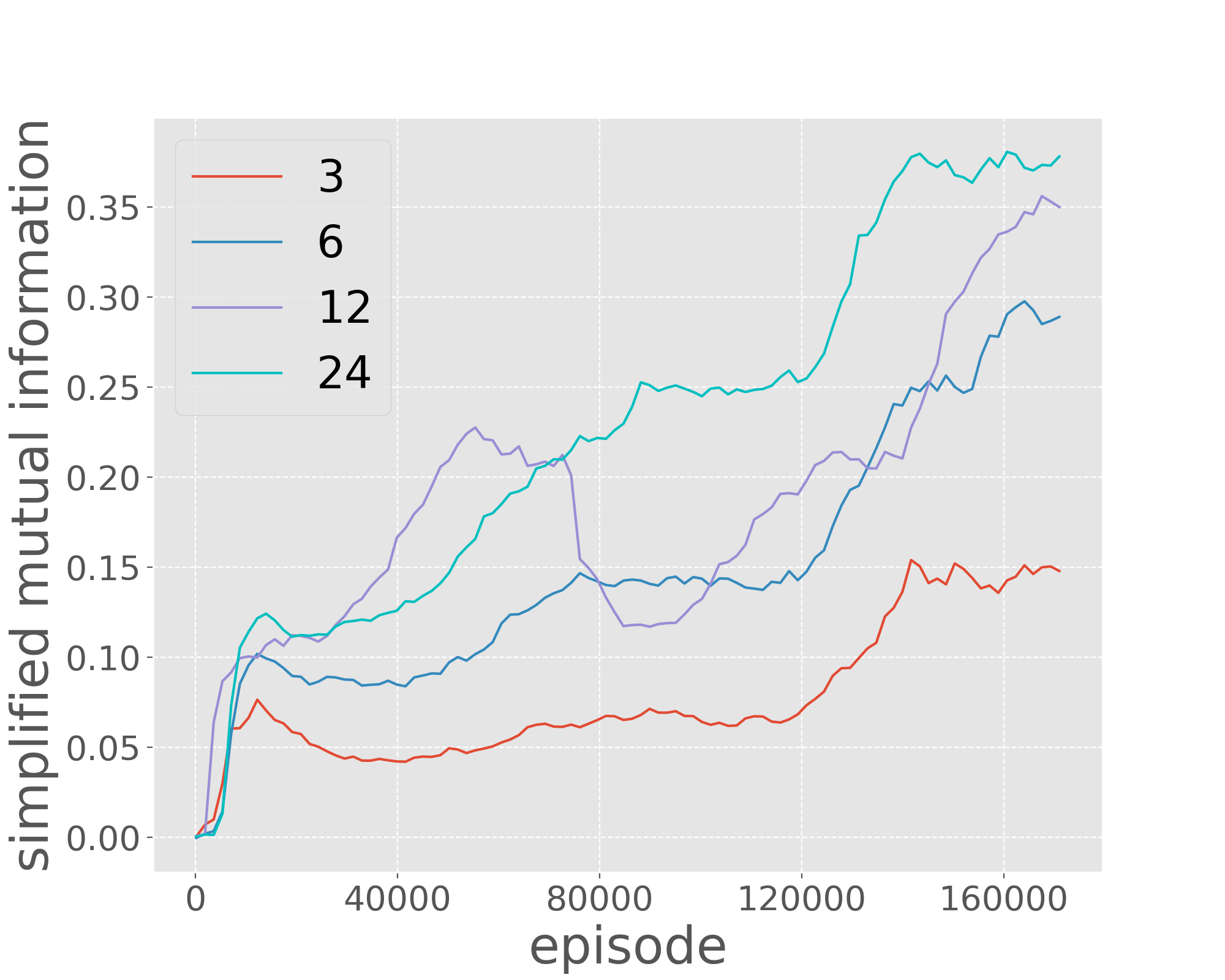}
\end{minipage}%
}%
\hspace{0.1in}
\subfigure[The loss curve of the MG index inference.]{
\begin{minipage}[t]{0.45\linewidth}
\centering
\includegraphics[width=2.3in]{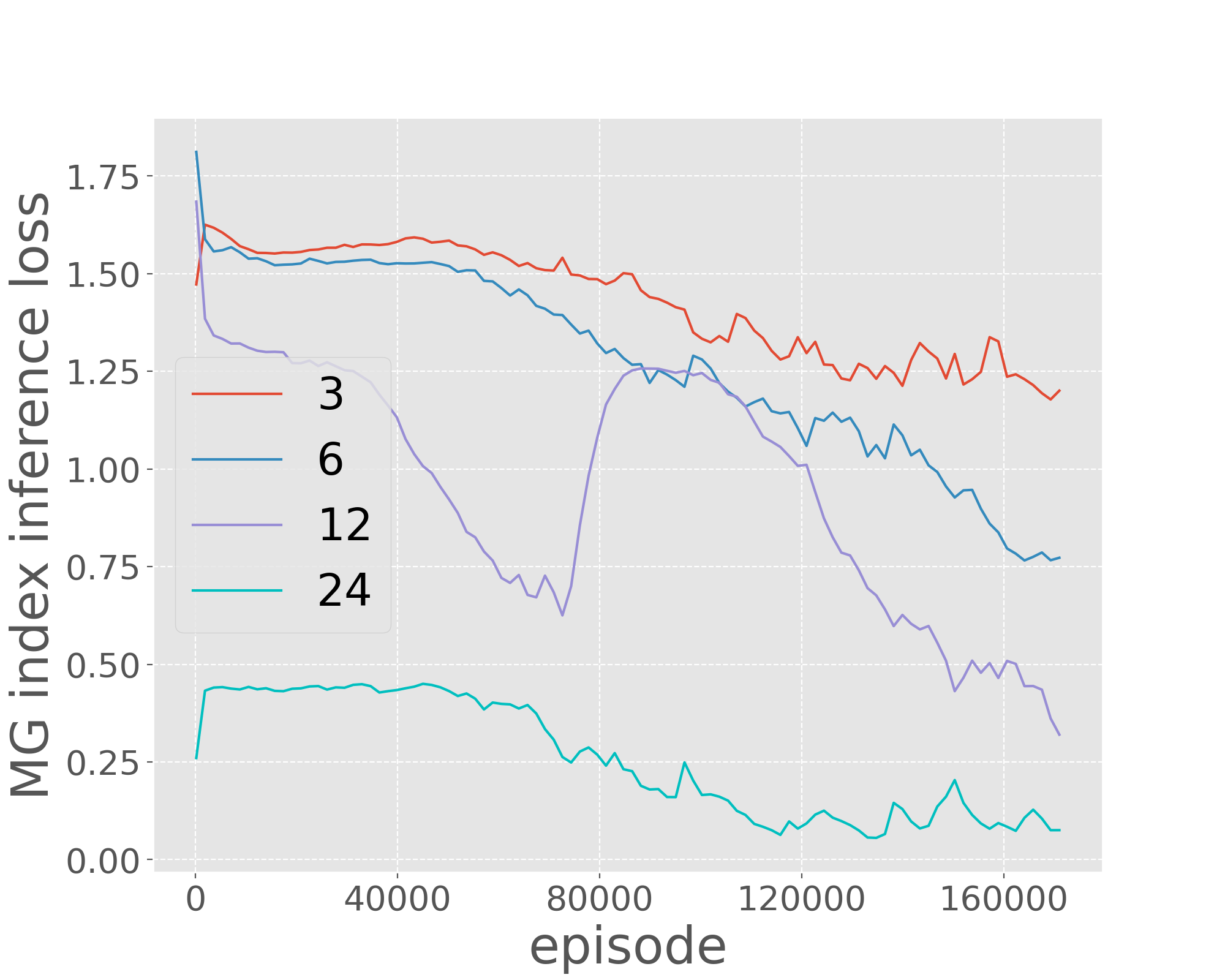}
\end{minipage}%
}%
\centering
\caption{Curves during training. The total number of training MGs is $4$, with $\{3,6,12,24\}$ agents in the resource occupation environment.}
\label{curves}
\end{figure*}
     
\subsection{Visualizations}
We visualize the trajectories of one agent in a resource occupation game in Figure \ref{traje}. Green dots and blue dots are agents and resources, respectively. When agents are only trained in this MG with different random seeds, different behaviors are obtained. This indicates that agents trained in single MGs are confined to environmental settings. Agents only learn the best responses and fit an NE. However, if the agents are only aware of some of the successful behaviors, the generalization will be constrained. On the contrary, the proposed MRA algorithm has a large capacity to represent multiple strategies by incorporating different relational graph with the distilled common knowledge, which leads to various modes of behaviors that are reasonable and diverse.

\begin{figure*}[htbp]
\centering
\subfigure{
\begin{minipage}[t]{0.45\linewidth}
\centering
\includegraphics[width=1.92in]{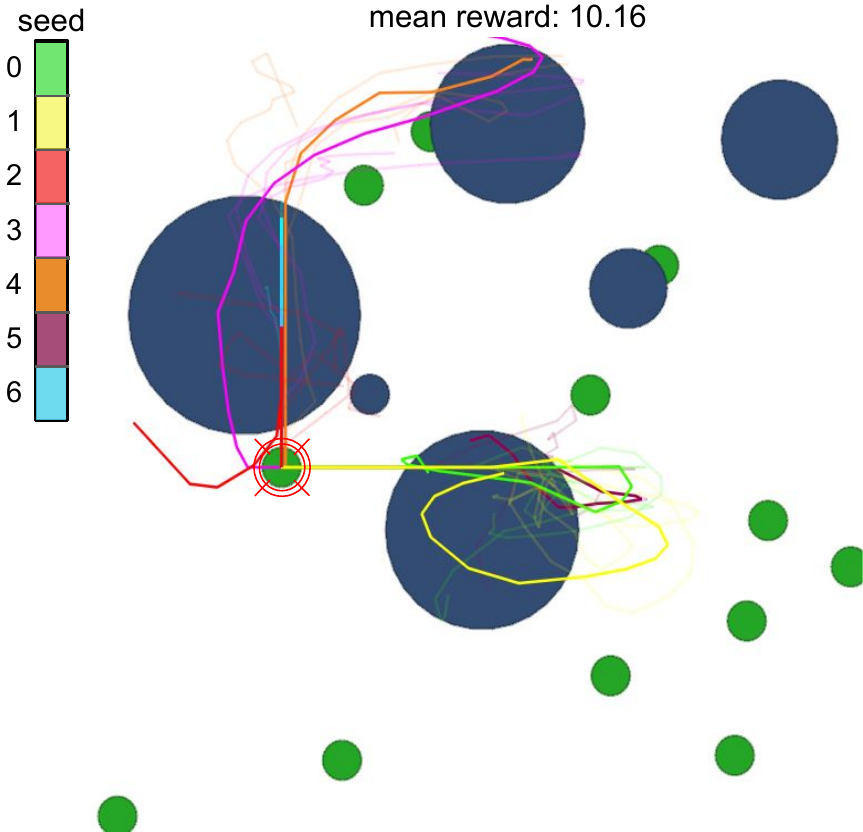}
\end{minipage}%
}%
\hspace{0.1in}
\subfigure{
\begin{minipage}[t]{0.45\linewidth}
\centering
\includegraphics[width=1.7in]{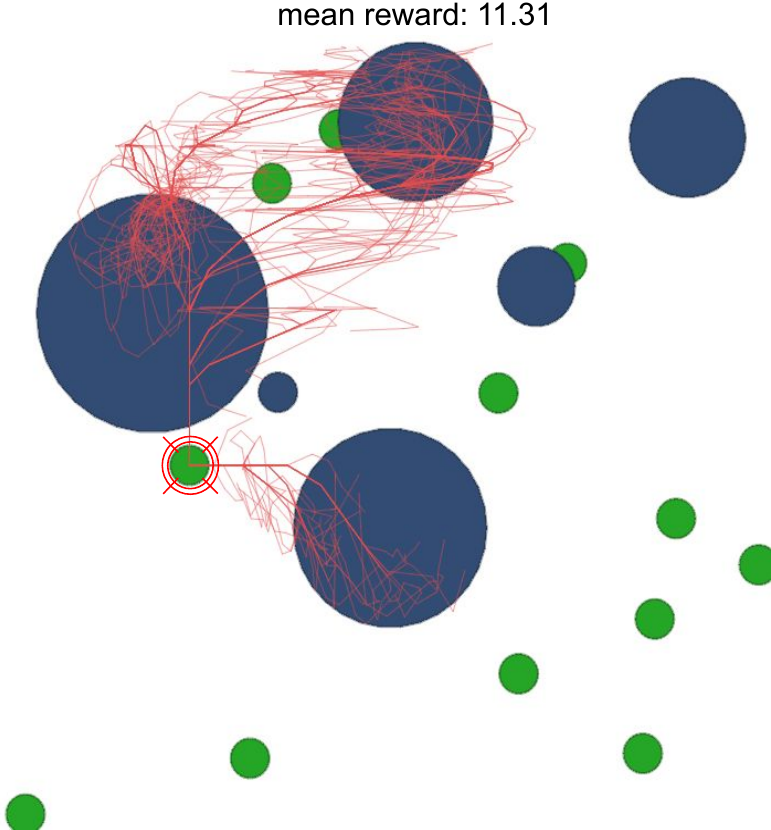}
\end{minipage}%
}%
\centering
\caption{Trajectory visualization in the resource occupation environment. \textbf{Left:} Trajectories of agents that are trained in a single Markov Game. \textbf{Right:} Trajectories of agents that are trained in multiple Markov Games.}
\label{traje}
\end{figure*}

We also visualize some instances of the learned relation variations, \ie, different relational graph $g$ under observation $o$, as well as how agents make the smartest decisions under different variations in Fig. \ref{bah}. The common knowledge learned by the agent can be interpreted as "moving to less-agent resources". Specifically, in Fig. \ref{bah}(a) the black agent makes decisions to move left by focusing on the topmost red agents which are occupying a resource. By focusing on the leftmost red agents in Fig. \ref{bah}(b), the black agent makes decisions to move up. Although such behavior might not be optimal, since the topmost resource is smaller than the leftmost resource, this variation helps agents learn common knowledge and optimally behave in an unseen MG by incorporating the optimal relation mapping in that game.
\begin{figure*}[h!tbp]
\centering
\subfigure{
\begin{minipage}[t]{0.33\linewidth}
\label{vis_a}
\centering
\includegraphics[width=1.8in]{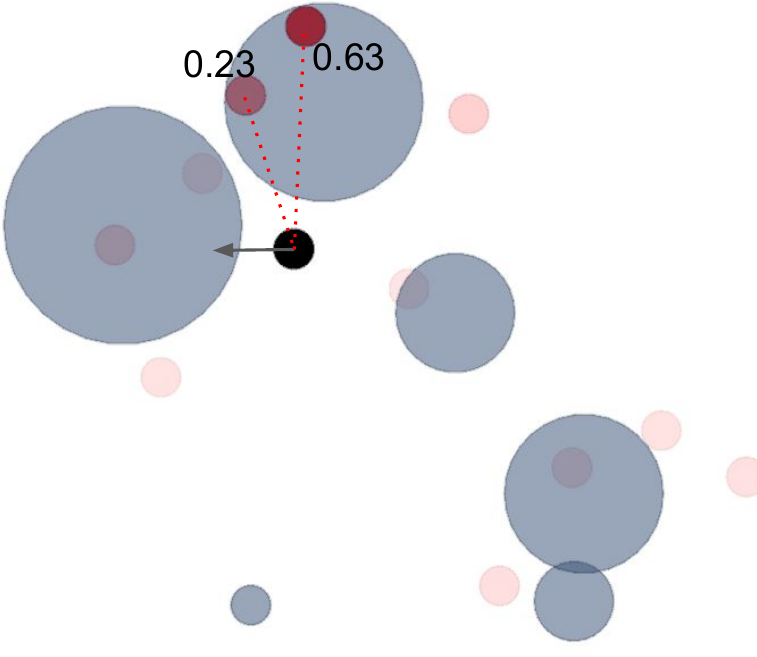}
\end{minipage}%
}%
\subfigure{
\begin{minipage}[t]{0.33\linewidth}
\centering
\label{vis_b}
\includegraphics[width=1.8in]{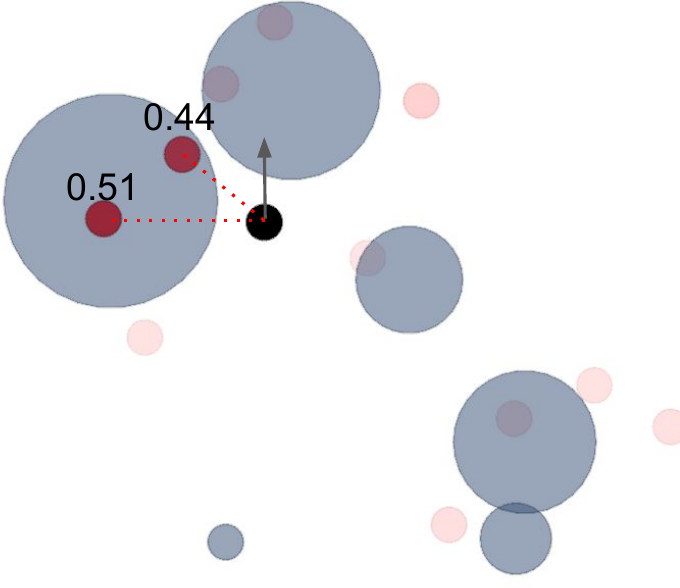}
\end{minipage}%
}%
\subfigure{
\begin{minipage}[t]{0.33\linewidth}
\centering
\label{vis_c}
\includegraphics[width=1.8in]{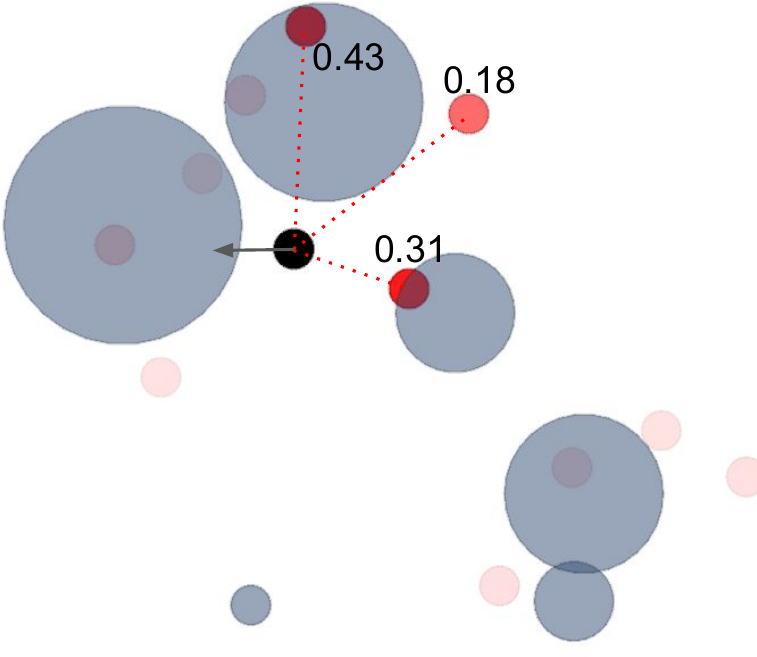}
\end{minipage}%
}%
\centering
\caption{Instances of different learned relational graph and the corresponding actions that agents take. We visualize how the black agent makes different reasonable decisions by incorporating different relational graphs. The relation scores $g^{i,j}$ that are smaller than $0.1$ are not shown.}
\label{bah}
\end{figure*}

\end{document}